%% file: main.tex
\documentclass{article}

\usepackage{microtype}
\usepackage{graphicx}
\usepackage{subfigure}
\usepackage{booktabs} 

\usepackage{hyperref}

\usepackage[accepted]{icml2024}

\usepackage{wrapfig}

\usepackage{wrapfig}
\usepackage[breakable, skins]{tcolorbox}
\definecolor{greyC}{RGB}{180,180,180}
\definecolor{greyL}{RGB}{235,235,235}
\usepackage{arydshln}
\usepackage{amsmath}
\usepackage{amsmath,amssymb,amsfonts}
\usepackage{amsthm}
\usepackage{mathrsfs}

\usepackage{amssymb}
\usepackage{enumitem}
\usepackage{subcaption}	
\usepackage{booktabs}
\usepackage{multirow}
\usepackage{xspace}
\usepackage{color, colortbl}
\definecolor{Gray}{gray}{0.9}

\usepackage[capitalize,noabbrev]{cleveref}

\theoremstyle{plain}
\newtheorem{theorem*}{Theorem}
\newtheorem{theorem}{Theorem}
\newtheorem{Definition}{Definition}
\newtheorem{lemma}{Lemma}
\newtheorem{lemma*}{Lemma}

\theoremstyle{definition}

\newtheorem{assumption}{Assumption}
\theoremstyle{remark}
\newtheorem{remark}{\textbf{Remark}}
\newtheorem{remark*}{\textbf{Remark}}

\def\*#1{\mathbf{#1}}
\usepackage{dsfont}

\usepackage[textsize=tiny]{todonotes}

\definecolor{mydarkred}{rgb}{0.6,0,0}
\definecolor{myblue}{HTML}{268BD2}
\definecolor{mygreen}{HTML}{658354}
\definecolor{orangeinplot}{HTML}{e29c7a}
\definecolor{purpleinplot1}{HTML}{7373a2}
\definecolor{purpleinplot2}{HTML}{ababd9}
\definecolor{greeninplot}{HTML}{288308}

\definecolor{mydarkred}{rgb}{0.6,0,0}
\definecolor{myblue}{HTML}{268BD2}
\definecolor{mygreen}{HTML}{658354}

\usepackage{xspace}
\usepackage{array}
\usepackage{ragged2e}
\newcolumntype{P}[1]{>{\RaggedRight\hspace{0pt}}p{#1}}
\newcolumntype{X}[1]{>{\RaggedRight\hspace*{0pt}}p{#1}}
\usepackage{tcolorbox}
\usepackage{tikz}
\usetikzlibrary{arrows,shapes,positioning,shadows,trees,mindmap}
\usepackage[edges]{forest}
\usetikzlibrary{arrows.meta}
\colorlet{linecol}{black!75}
\usepackage{xkcdcolors} 
\usepackage{wrapfig}

\usepackage{tikz}
\usetikzlibrary{backgrounds}
\usetikzlibrary{arrows,shapes}
\usetikzlibrary{tikzmark}
\usetikzlibrary{calc}

\usepackage{algorithm}
\usepackage[ruled,vlined,algo2e]{algorithm2e}
\usepackage{dsfont}

\icmltitlerunning{When and How Does In-Distribution Label Help Out-of-Distribution Detection?}

\begin{document}

\twocolumn[
\icmltitle{When and How Does In-Distribution Label Help Out-of-Distribution Detection?}



\icmlsetsymbol{equal}{*}

\begin{icmlauthorlist}
\icmlauthor{Xuefeng Du}{yyy}
\icmlauthor{Yiyou Sun}{yyy}
\icmlauthor{Yixuan Li}{yyy}
\end{icmlauthorlist}

\icmlaffiliation{yyy}{Department of Computer Sciences, UW-Madison}

\icmlcorrespondingauthor{Yixuan Li}{sharonli@cs.wisc.edu}

\icmlkeywords{Machine Learning, ICML}

\vskip 0.3in
]



\printAffiliationsAndNotice{} 

\begin{abstract}

Detecting data points deviating from the training distribution is pivotal for ensuring reliable machine learning. Extensive research has been dedicated to the challenge, spanning classical anomaly detection techniques to contemporary out-of-distribution (OOD) detection approaches. While OOD detection commonly relies on supervised learning from a labeled in-distribution (ID) dataset, anomaly detection may treat the entire ID data as a single class and disregard ID labels.
This fundamental distinction raises a significant question that has yet to be rigorously explored: \textit{when and how does ID label help OOD detection?} This paper bridges this gap by offering a formal understanding to theoretically delineate the impact of ID labels on OOD detection. We employ a graph-theoretic approach, rigorously analyzing the separability of ID data from OOD data in a closed-form manner. Key to our approach is the characterization of data representations through spectral decomposition on the graph. Leveraging these representations, we establish a provable error bound that compares the OOD detection performance with and without ID labels, unveiling conditions for achieving enhanced OOD detection. Lastly, we present empirical results on both simulated and real datasets, validating theoretical guarantees and reinforcing our insights.   Code is publicly available at
\url{https://github.com/deeplearning-wisc/id_label}.

\end{abstract}

\section{Introduction}
\input{intro}

\section{Problem Setup}

\input{setup}

\section{Analysis Framework}
\input{framework}

\section{Theoretical Results}
\input{theory}

\section{Experiments on Real Datasets}
\input{exp}

\section{Related Work}
\input{related_work}

\section{Conclusion}
\input{conclusion}
\bibliography{citation}
\bibliographystyle{icml2024}

\newpage
\appendix
\onecolumn
\input{notations_app}

\subsection{Assumptions}
\input{assumption_app}

\section{Main Theorems}
\input{theorem_app}

\section{Proofs of Main Theorems}
\label{Proof_all_app}

\input{proof_app}

\section{Necessary Lemmas}

\input{lemmas_app}

\input{empi_verifi}

\end{document}

%% file: intro.tex
When deployed in the real world, machine learning models often encounter unfamiliar data points that fall outside the distribution of the observed data. This problem has been studied extensively, dating from the classical anomaly detection methods~\cite{chandola2009anomaly,ahmed2020detecting,han2022adbench} to contemporary out-of-distribution (OOD) approaches~\cite{liu2020energy,yang2021generalized, fang2022learnable}.

While both anomaly detection and OOD detection share the goal of identifying test-time input that deviates from the training distribution, a crucial distinction lies in the \textit{usage of in-distribution (ID) labels in training time}. Specifically, classical anomaly detection may disregard ID labels~\cite{yang2021generalized}, treating the entire ID dataset as a single class. In contrast, OOD detection commonly relies on supervised learning from a labeled ID dataset. It is reasonable to hypothesize that incorporating ID labels during training might influence the resulting feature representations, potentially leading to distinct capabilities in separating ID from OOD samples during test time.  This raises a significant question that has yet to be rigorously explored in the field:
\vspace{-0.2cm}
\begin{center}
    \textbf{\emph{RQ: When and how does ID label help OOD detection?}}
\end{center}
\vspace{-0.2cm}

Answering this question offers the fundamental key to understanding and bridging two highly related fields of anomaly detection and OOD detection. In pursuit of this objective, we provide a formal understanding to theoretically delineate the influence of ID labels on OOD detection. We base our analysis on a graph-theoretic approach by modeling the ID data via a graph, where the vertices are all the data points and edges encode the similarity among data. This analytical framework is well-suited for our investigation, as data points' representation similarity can differ between the self-supervised and supervised learning setting, contingent upon the availability of ID labels. For instance, when ID labels are present, the supervisory signal facilitates connecting points belonging to the same class, resulting in each class manifesting as a distinct connected sub-graph. In both cases (with or without ID labels), the sub-structures can be revealed by performing spectral decomposition on the graph and can be expressed equivalently as a contrastive learning objective on neural net representations (expounded further in Section~\ref{sec:3}). Importantly, these learned feature representations allow us to rigorously analyze the separability of ID data from OOD data in a closed-form manner.

 \begin{figure*}[t]
  \begin{center}
   {\includegraphics[width=1\linewidth]{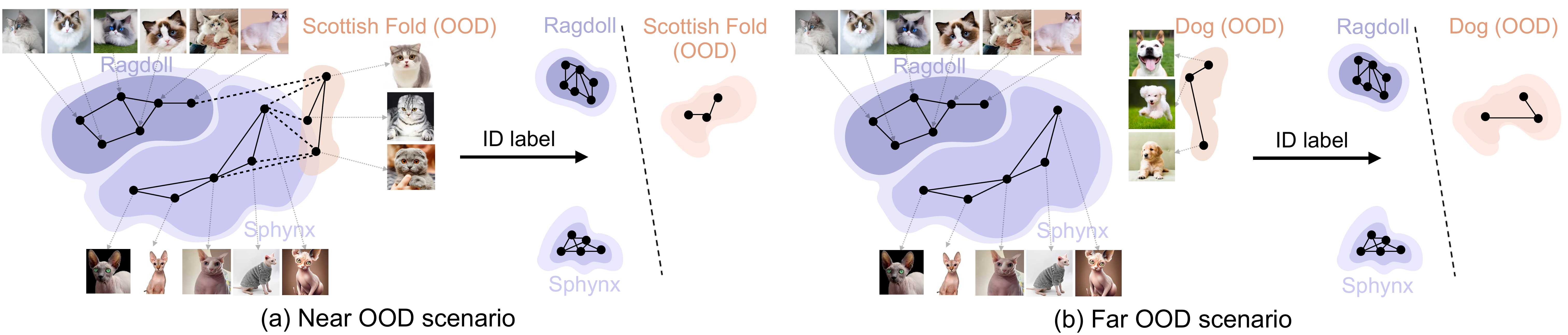}}
  \end{center}
  \caption{\small Intuitive example on the ID labels' impact on OOD detection. (a) In the near OOD scenario where the OOD data connects densely with the ID data, without ID labels, the neural network produces indistinguishable embeddings for the ID (\textcolor{purpleinplot1}{{Ragdoll}} and \textcolor{purpleinplot2}{{Sphynx}} class) and OOD data (\textcolor{orangeinplot}{{Scottish Fold}} class). By harnessing the power of the ID labeling information, the model learns more distinguishable embeddings that help ID vs. OOD separation. (b) In the far OOD scenario (\textcolor{orangeinplot}{{Dog}} class), ID labels can be less beneficial because the representations learned in an unsupervised manner can already be separable between ID vs OOD. }
     \vspace{1em}
  \label{fig:teaser}
  \end{figure*}

Based on the analytical framework, we provide a formal error bound in Theorem~\ref{MainT-2}, comparing the OOD detection performance with and without the inclusion of ID labels. This theorem reveals sufficient conditions for achieving improved OOD detection performance by leveraging ID labels. To establish the error bound, we first calculate the closed-form solution of the
ID and OOD representations based on the graph factorization and then quantify the OOD detection performance by linear probing error. As a result, Theorem~\ref{MainT-2} demonstrates that the difference in the OOD detection performance with and without ID labels can be lower bounded by a function of the adjacency matrix of ID data as well as the OOD-ID connectivity. Furthermore, we offer intuitive interpretations in Theorem~\ref{MainT-3}, and show that the ID labels are the most beneficial when either: (i) the OOD data is relatively close to the ID data, which is also known as the near OOD scenario; (ii) the ID data are connected sparsely without ID labels; and (iii) the semantic connection between each ID data to the
labeled ID data from different ID classes is sufficiently
large. To help readers understand the key insights of our theory, we provide a simple intuitive example in Figure~\ref{fig:teaser}, which demonstrates how adding ID labels in the near OOD scenario can lead to a greater benefit compared to the far OOD scenario.

Lastly, we provide empirical verifications to support our theory. In particular, we compare the OOD detection performance with and without ID labels using both simulated data and real-world datasets (Section~\ref{exp_main}). The result aligns with our
theoretical guarantee, showcasing the benefits of the ID label information under proper conditions. For example, the OOD detection result can be improved by 12.3\% (AUROC) in the near OOD scenario compared to 6.06\% in the far OOD scenario on \textsc{Cifar100}, validating our theory. 

Our main contributions are summarized as follows:

\begin{itemize}
\vspace{-0.2cm}
    \item We study an important but underexplored problem: when and how in-distribution labels can help OOD detection. Our exposition has fundamental value in understanding and bridging the two highly related fields of anomaly
detection and OOD detection.

    \item We provide an analytical framework based on graph formulation to characterize the ID and OOD representations. Based on that, we analyze the error bound for ID vs. OOD separation with and without ID labels and investigate the necessary conditions for which the labeling information can bring the most benefits.
   
    \item We present empirical analysis on both simulated and real-world datasets to verify and support our theory. The observation in practice echoes and reinforces our theoretical insights.

\end{itemize}

%% file: setup.tex
 \textcolor{black}{Let $\mathcal{X}$ be the input space, and $\mathcal{Y}=\{1,...,c\}$ be the label space for ID data. Given an unknown ID joint distribution $\mathbb{P}_{\mathcal{X}\mathcal{Y}}$ defined over $\mathcal{X}\times \mathcal{Y}$, the \emph{labeled ID data} $\mathcal{S}_{\rm id}^{(l)} = \{(\*x_1, y_1),...,(\*x_N, y_N)\}$ are drawn independently and identically from $\mathbb{P}_{\mathcal{X}\mathcal{Y}}$. Alternatively, the \emph{unlabeled ID data} $\mathcal{S}_{\rm id}^{(u)} =  \{\*x_1,...,\*x_N\}$ is  drawn from $\mathbb{P}_{\rm id}$, which is the marginal distribution of $\mathbb{P}_{\mathcal{X}\mathcal{Y}}$ on $\mathcal{X}$. Furthermore, we denote $\mathbb{P}_{l_i}$ the distribution of labeled data with label $i \in \mathcal{Y}$. }

\begin{Definition}[{Out-of-Distribution Detection w/ ID Labels}]\label{P1}
   Given labeled ID data
   $\mathcal{S}_{\rm id}^{(l)}$, 
    the aim is to learn a predictor $\*g: \mathcal{X}\rightarrow \mathcal{Y} \cup \{ {\rm ood} \}$ such that for any test data $\mathbf{x}$:
   1) if $\mathbf{x}$ is drawn from $\mathbb{P}_{\rm id}$, then the model classifies $\mathbf{x}$ into one of ID classes $\mathcal{Y}$, and 2)
        if $\mathbf{x}$ is drawn from another distribution $\mathbb{P}_{\rm ood}$ with unknown OOD class, then $\*g$ can detect $\mathbf{x}$ as OOD data~\citep{fang2022learnable}. 
\end{Definition}

\vspace{0.2cm}
\begin{Definition}[{Out-of-Distribution Detection w/o ID Labels}]\label{P2}
The definition is similar to above, except that we are using unlabeled ID data $\mathcal{S}_{\rm id}^{(u)} =  \{\*x_1,...,\*x_N\}$ to learn the binary predictor $\*g$. This is in accordance with the classical anomaly detection problem~\cite{chandola2009anomaly}.
\end{Definition}

%% file: framework.tex
\label{sec:3}
\paragraph{Overview of rationale.} In this section, we introduce our analytical framework, which allows us to formalize and understand the OOD detection performance in two cases: \textbf{(1)} learning without ID labels, and \textbf{(2)} learning with ID labels, respectively.   
Our analytical framework models the ID data via a graph, where the vertices are all the data points, and edges encode the similarity among data (Section~\ref{sec:3.1}). The similarity can be defined in either a self-supervised or supervised manner, contingent on the availability of the ID labels. For example, when ID labels are present, the supervision signal can help connect points belonging to the same class, so that each class emerges clearly as a connected sub-graph. In both cases, the sub-structures can be revealed by performing spectral decomposition on the graph and can be expressed equivalently as a contrastive learning objective on neural net representations (Section~\ref{sec:3.2}). Importantly, these learned feature representations allow us to rigorously analyze the separability of ID data from OOD data in a closed form. Since the features learned can be directly impacted by the presence or absence of ID labels, the OOD detection performance can vary accordingly.

\subsection{Graph Formulation}
\label{sec:3.1}
We start by formally defining the graph and
adjacency matrix. For notation clarity, we use $\Bar{\*x}$ to indicate the natural sample (raw inputs without augmentation). Given an $\Bar{\*x}$, we use $\mathcal{T}(\*x|\Bar{\*x})$ to denote the probability of $\*x$ being augmented from $\Bar{\*x}$. For instance, when $\Bar{\*x}$ represents an image, $\mathcal{T}(\cdot|\Bar{\*x})$ can be the distribution of common augmentations \cite{chen2020simple} such as Gaussian blur, color distortion, and random cropping. {The augmentation allows us to define a general population space $\mathcal{X}_\text{id}$, which contains all the original ID data points along with their augmentations.} We denote the cardinality of the population space with $|\mathcal{X}_\text{id}|=N$.

 We define the graph $G(\mathcal{X}_{\rm id}, \zeta)$ over the finite
vertex set $\mathcal{X}_\text{id}$ with edge weights $\zeta$. To define edge weights $\zeta$, we consider two cases: (1) self-supervised connectivity $\zeta^{(u)}$ by treating all points in $\mathcal{X}_{\rm id}$ as entirely unlabeled, and (2) supervised connectivity $\zeta^{(l)}$ by utilizing 
labeling information of ID data. 

\begin{Definition}[\textbf{Unlabeled case} (u)]\label{Def4}
When all ID points are unlabeled, two samples ($\*x$, $\*x^+$) are considered a {\textbf{positive pair}} if $\*x$ and $\*x^+$ are augmented from the same image $\Bar{\*x} \sim \mathbb{P}_{\rm id}$.
For any two augmented data $\*x, \*x' \in \mathcal{X}_{\rm id}$, the edge weight $\zeta^{(u)}_{\*x \*x'}$ is defined as the marginal probability of generating the pair~\cite{haochen2021provable}:
\begin{align}
\zeta^{(u)}_{\*x \*x^{\prime}} \triangleq \mathbb{E}_{\Bar{\*x} \sim {\mathbb{P}}_{\rm id}}  \mathcal{T}(\*x| \Bar{\*x}) \mathcal{T}\left(\*x'| \Bar{\*x}\right).
    \label{eq:def_wxx}
\end{align}
\end{Definition}
The magnitude of $\zeta^{(u)}_{\*x\*x'}$ indicates the ``positiveness'' or similarity between  $\*x$ and $\*x'$.

Alternatively, when having access to the labeling information for ID data, we can define the edge weight by {adding additional supervised connectivity to the graph}. 

\begin{Definition}[\textbf{Labeled case} (l)]\label{Def5}
When all ID points are labeled, two samples ($\*x$, $\*x^+$) are considered a {\textbf{positive pair}} if $\*x$ and $\*x^+$ are augmented from two labeled samples $\Bar{\*x}_{l}$ and $\Bar{\*x}'_{l}$ \emph{with the same ID class $i$}. The overall edge weight for any pair of data $(\*x,\*x')$ is given by: 
\begin{align*}
\small
\begin{split}
&\zeta^{(l)}_{\*x \*x^{\prime}} = \phi_{u} \cdot \zeta^{(u)}_{\*x \*x^{\prime}} + \phi_{l}  \cdot \sum_{i \in \mathcal{Y}}\mathbb{E}_{\Bar{\*x}_{l} \sim {\mathbb{P}_{l_i}}} \mathbb{E}_{\Bar{\*x}'_{l} \sim {\mathbb{P}_{l_i}}} \mathcal{T}(\*x | \Bar{\*x}_{l}) \mathcal{T}\left(\*x' | \Bar{\*x}'_{l}\right),
    \label{eq:def_wxx_b}
\end{split}
\end{align*}
where $\phi_{u},\phi_{l}$ are the weight coefficients. Compared to the unlabeled case, the second term strengthens the connectivity for points belonging to the same class.
\end{Definition}

\begin{Definition}[{Adjacency matrix for unlabeled ID data}]
    We define the adjacency matrix $\*A^{(u)}$ with entry value $\zeta^{(u)}_{\*x \*x^{\prime}}$ for each $(\*x, \*x')$ pair . Further, $\zeta^{(u)}_\*x = \sum_{\*x' \in \mathcal{X}}\zeta^{(u)}_{\*x\*x'}$
denotes the total edge weights connected to a vertex $\*x$. 
\end{Definition}
\begin{Definition}[{Adjacency matrix for labeled ID data}] 
    Similarly, we define the adjacency matrix for labeled ID data $\*A^{(l)}$ with entry value $\zeta^{(l)}_{\*x \*x^{\prime}}$ for each $(\*x, \*x')$ pair and $\zeta^{(l)}_\*x = \sum_{\*x' \in \mathcal{X}}\zeta^{(l)}_{\*x\*x'}$.
\end{Definition}

 As a standard technique in graph theory~\cite{chung1997spectral}, we use the \textit{normalized adjacency matrix} of $G(\mathcal{X}_{\rm id}, \zeta)$:
\begin{equation}
    \tilde{\*A}\triangleq \*D^{-\frac{1}{2}} \*A \*D^{-\frac{1}{2}},
    \label{eq:def}
\end{equation}
where $\*A$ can be instantiated by either $\*A^{(u)}$ or $\*A^{(l)}$ defined above. ${\*D} \in \mathbb{R}^{N \times N}$ is the corresponding diagonal matrix with ${\*D}^{}_{\*x \*x}=\zeta^{(u)}_\*x$ for unlabeled case and ${\*D}^{}_{\*x \*x}=\zeta^{(l)}_\*x$ for labeled case. The normalization balances the degree of each node,  reducing the influence of vertices with very large degrees. The normalized adjacency matrix allows us to perform spectral decomposition as we show next.

\subsection{Learning Representations Based on Graph Spectral}
\label{sec:3.2}

In this section, we perform spectral decomposition or spectral clustering~\cite{ng2001spectral}---a classical approach to graph
partitioning---to the adjacency matrices defined above. This process forms a matrix where the top-$k$ eigenvectors are the columns and \emph{each row of the matrix can be viewed as a $k$-dimensional representation of an example}. The resulting
feature representations enable us to rigorously analyze the separability of ID data from OOD data in a closed form, and formally compare the OOD detection error under two scenarios either with and without ID labels (in Section~\ref{sec:theory}). 

Specifically, taking the labeled case as an example, we consider the following optimization, which performs low-rank matrix approximation on the adjacency matrix $\tilde{\*A}^{(l)}$:
\begin{equation}
     \min _{\*F^{(l)} \in \mathbb{R}^{N \times k}} \mathcal{L}(\*F^{(l)}, {\*A}^{(l)})\triangleq\left\|\tilde{\*A}^{(l)}-\*F^{(l)} \*F^{(l)\top}\right\|_F^2,
    \label{eq:lmf}
\end{equation}
where $\|\cdot\|_F$ denotes the matrix Frobenious norm. According to the Eckart–Young–Mirsky theorem~\cite{eckart1936approximation}, the minimizer of this loss function is $\*F_k^{(l)}\in \mathbb{R}^{N \times k}$ such that $\*F_k^{(l)} \*F_k^{(l)\top}$ contains the top-$k$ components of $\tilde{\*A}^{(l)}$'s eigen decomposition.

\paragraph{A surrogate objective.} In practice, directly solving objective~\ref{eq:lmf} can be computationally expensive for an extremely large matrix. To circumvent this, the feature representations
can be equivalently recovered by minimizing the following contrastive learning objective~\cite{sun2023graph,sun2023when} as shown in Lemma~\ref{lemma1_main}, which can be efficiently trained end-to-end using a neural net parameterized by $\*w$:
\begin{align}
\begin{split}
    \mathcal{L}_\text{labeled}(\*h_\*w) &\triangleq - 2\phi_{l} \mathcal{L}_1(\*h_\*w) 
- 2\phi_{u}  \mathcal{L}_2(\*h_\*w)  +\\& \phi_{l}^2 \mathcal{L}_3(\*h_\*w) + 2\phi_{l} \phi_{u} \mathcal{L}_4(\*h_\*w) +  
\phi_{u}^2 \mathcal{L}_5(\*h_\*w),
\label{eq:def_SORL}
\end{split}
\end{align} where $\*h_\*w:\mathcal{X}_{\rm id}\mapsto \mathbb{R}^k$ denotes the feature representation,
\begin{equation*}
\begin{aligned}
    & \mathcal{L}_1(\*h_\*w) = \sum_{i \in \mathcal{Y}}\underset{\substack{\Bar{\*x}_{l} \sim \mathbb{P}_{{l_i}}, \Bar{\*x}'_{l} \sim \mathbb{P}_{{l_i}},\\\*x \sim \mathcal{T}(\cdot|\Bar{\*x}_{l}), \*x^{+} \sim \mathcal{T}(\cdot|\Bar{\*x}'_l)}}{\mathbb{E}}\left[\*h_\*w(\*x)^{\top} {\*h_\*w}\left(\*x^{+}\right)\right] , \\&
    \mathcal{L}_2(\*h_\*w) = \underset{\substack{\Bar{\*x}_{u} \sim \mathbb{P}_{\rm id},\\\*x \sim \mathcal{T}(\cdot|\Bar{\*x}_{u}), \*x^{+} \sim \mathcal{T}(\cdot|\Bar{\*x}_u)}}{\mathbb{E}}
\left[\*h_\*w(\*x)^{\top} {\*h_\*w}\left(\*x^{+}\right)\right], \\&
    \mathcal{L}_3(\*h_\*w) = \sum_{i, j\in \mathcal{Y}}
    \underset{\substack{\Bar{\*x}_l \sim \mathbb{P}_{{l_i}}, \Bar{\*x}'_l \sim \mathbb{P}_{{l_{j}}},\\\*x \sim \mathcal{T}(\cdot|\Bar{\*x}_l), \*x^{-} \sim \mathcal{T}(\cdot|\Bar{\*x}'_l)}}{\mathbb{E}}
\left[\left(\*h_\*w(\*x)^{\top} {\*h_\*w}\left(\*x^{-}\right)\right)^2\right], \\&
    \mathcal{L}_4(\*h_\*w) = \sum_{i \in \mathcal{Y}}\underset{\substack{\Bar{\*x}_l \sim \mathbb{P}_{{l_i}}, \Bar{\*x}_u \sim \mathbb{P}_{\rm id},\\\*x \sim \mathcal{T}(\cdot|\Bar{\*x}_l), \*x^{-} \sim \mathcal{T}(\cdot|\Bar{\*x}_u)}}{\mathbb{E}}
\left[\left(\*h_\*w(\*x)^{\top} {\*h_\*w}\left(\*x^{-}\right)\right)^2\right], \\&
    \mathcal{L}_5(\*h_\*w) = \underset{\substack{\Bar{\*x}_u \sim \mathbb{P}_{\rm id}, \Bar{\*x}'_u \sim \mathbb{P}_{\rm id},\\\*x \sim \mathcal{T}(\cdot|\Bar{\*x}_u), \*x^{-} \sim \mathcal{T}(\cdot|\Bar{\*x}'_u)}}{\mathbb{E}}
\left[\left(\*h_\*w(\*x)^{\top} {\*h_\*w}\left(\*x^{-}\right)\right)^2\right].
\label{eq:def_SORL_detail}
\end{aligned}
\end{equation*}

\paragraph{Interpretation.}
At a high level, $\mathcal{L}_1,$ $\mathcal{L}_2$ push embeddings of \textit{positive pairs} to be closer while $\mathcal{L}_3$, $\mathcal{L}_4$ 
 and $\mathcal{L}_5$ pull away embeddings of \textit{negative pairs}. Particularly, $\mathcal{L}_1$ samples two random augmentation views of two images from labeled data with the \textit{same} label. $\mathcal{L}_2$ samples two views from the same image in $\mathcal{X}_{\rm id}$. For negative pairs, $\mathcal{L}_3$ uses two augmentation views from two labeled samples in $\mathcal{X}_{l}$ with \textit{any} label. $\mathcal{L}_4$ uses two views of one sample in $\mathcal{X}_{l}$ and another one in $\mathcal{X}_{\rm id}$. $\mathcal{L}_5$ uses two views from two random samples in $\mathcal{X}_{\rm id}$. 
 
Importantly, the contrastive loss allows drawing a theoretical equivalence between learned representations and the top-$k$ singular vectors of $\tilde{\*A}^{(l)}$, and facilitates theoretical understanding of the OOD detection on the data represented by $\tilde{\*A}^{(l)}$. We formalize the equivalence below. 

\begin{lemma}[Theoretical equivalence between two objectives]
\label{lemma1_main} 
We define each row $\*f_\*x^{\top}$ of $\*F^{(l)}$ as a scaled version of learned feature representation $\*h_\*w$, with $\*f_\*x = \sqrt{\zeta_\*x}\*h_\*w(\*x)$. Then minimizing the loss function $\mathcal{L}(\*F^{(l)}, \*A^{(l)})$ in Equation~\ref{eq:lmf} is equivalent to minimizing the surrogate loss in Equation~\ref{eq:def_SORL}. Full proof is in Appendix Section~\ref{sec:loss_deri_app}.

\end{lemma}

\begin{remark} We can extend the contrastive learning objective in Equation~\ref{eq:def_SORL} to the \emph{unlabeled case} by setting the coefficient $\phi_l$ to 0 and keeping the remaining parts:
\begin{align}
\begin{split}
    \mathcal{L}_\text{unlabeled}(\*h_\*w) &\triangleq - 2\phi_{u} \mathcal{L}_2(\*h_\*w) 
+  \phi_{u}^2 \mathcal{L}_5(\*h_\*w).
\label{eq:def_SORL_unlabeled}
\end{split}
\end{align}
\end{remark}

The loss has been employed in prior works on spectral contrastive learning~\cite{sun2023graph,sun2023when}, which analyzed problems such as novel category discovery and open-world semi-supervised learning. However, our paper
focuses on the problem of OOD detection, which has fundamentally different learning goals. Accordingly, we derive novel theoretical analyses uniquely
tailored to our problem focus (i.e., the impact of the ID label information), which we present next.

%% file: theory.tex
\label{sec:theory}
 Based on the analytical framework, we now provide theoretical insights to the core question: \emph{\textbf{when and how does ID label information help OOD detection?}} To answer this question, we start by deriving the closed-form solution of the representations for both ID and OOD data (Section~\ref{sec:4.1}), and then quantify the OOD detection
performance by measuring the linear probing error (Section~\ref{sec:4.2}). Finally, we provide a formal bound contrasting the OOD detection performance with and without ID labels (Section~\ref{sec:4.3}).

\subsection{Representation for ID and OOD Data}
\label{sec:4.1}

\paragraph{ID representations.} We first derive the ID representations for the labeled case, which can be similarly derived for the unlabeled case. Specifically, one can train the neural network $\*h_{\*w}: \mathcal{X}_\text{id} \rightarrow \mathbb{R}^k$ using the surrogate objective in Equation~\ref{eq:def_SORL}.  Minimizing the loss yields representation $\*Z^{(l)}\in \mathbb{R}^{N\times k}$, where each row vector $\*z_i = \*h_{\*w}(\*x_i)$. According to Lemma~\ref{lemma1_main}, the closed-form solution for the representations is equivalent to performing spectral decomposition of the adjacency matrix. Thus, we have $\*F_k^{(l)} = [{\*D^{(l)}}]^{\frac{1}{2}}\*Z^{(l)}$, where $\*F_k^{(l)} \*F_k^{(l)\top}$ contains the top-$k$ components of $\tilde{\*A}^{(l)}$'s SVD decomposition. 
We further denote the top-$k$ singular vectors of $\tilde{\*A}^{(l)}$  as $\*V_k^{(l)} \in \mathbb{R}^{N\times k}$, so we have $\*F_k^{(l)} = \*V_k^{(l)} [\*\Sigma_k^{(l)}]^{\frac{1}{2}}$, where $\*\Sigma_k^{(l)}$ is a diagonal matrix of the top-$k$ singular values of $\tilde{\*A}^{(l)}$.
{By equalizing the two forms of $\*F_k^{(l)}$}, the closed-formed solution of the learned feature space is given by 
\begin{equation}
    \*Z^{(l)} = [\*D^{(l)}]^{-\frac{1}{2}} \*V_k^{(l)} [\*\Sigma_k^{(l)}]^{\frac{1}{2}}.
    \label{labeled_id_feat}
\end{equation}

\begin{figure*}[t]
  \begin{center}
   {\includegraphics[width=1\linewidth]{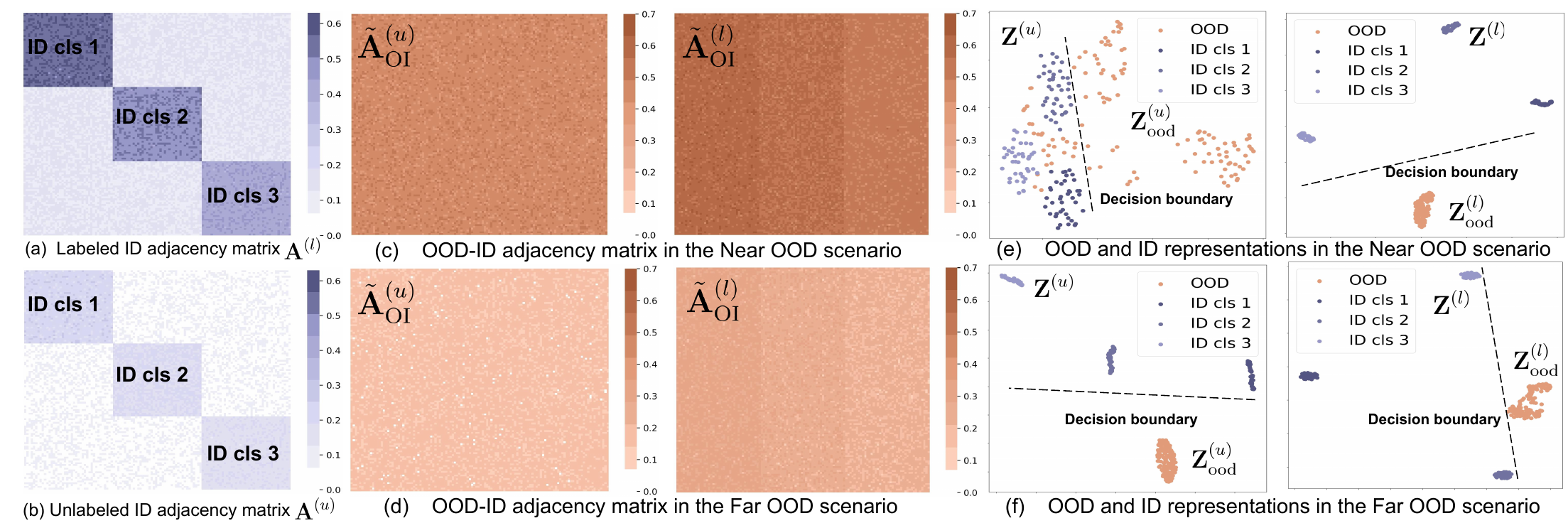}}
  \end{center}
  \caption{\small  Example showcasing the contrast between adjacency matrices and representations w/ (l) and w/o (u) ID labels. (a) The ID adjacency matrix in the labeled case $\*A^{(l)}$. (b) The ID adjacency matrix in the unlabeled case $\*A^{(u)}$. Here darker color indicates denser connectivity. The contrast of the OOD-ID adjacency matrix $\tilde{\*A}_{\rm OI}$ w/ and w/o ID labels in the near OOD and far OOD scenario is shown in (c) and (d), where the adjacency matrices have a larger Frobenius norm, i.e., $\|\tilde{\*A}_{\rm OI}^{(u)}\|_F=60$ in the near OOD scenario and smaller norm in the far OOD scenario, i.e., $\|\tilde{\*A}_{\rm OI}^{(u)}\|_F=24$. (e) Learned representations in the near OOD scenario, where the OOD representations are overlapped in the unlabeled case but become linearly separable from the ID representations in the labeled case. (d) Representations in the far OOD scenario. The ID and OOD representations can already be separable in the unlabeled case. The benefit of ID labels is marginal.}
  \label{fig:toy_example}
  \end{figure*}

 \paragraph{OOD representations.} In post hoc OOD detection, the learning algorithm can only observe ID data in $\mathcal{X}_{\rm id}$ and the corresponding adjacency matrix. Hence, a key challenge in our framework is how to derive the OOD representations based on ID and OOD data connectivity in the input space. Unlike previous literature~\cite{lee2018simple}, we refrain from making simplified assumptions in the feature representation space (although that makes analysis much easier). More realistically, we \emph{characterize OOD data directly in the input space} by the adjacency matrix $\tilde{\*A}_{\rm OI}^{(l)} \in \mathbb{R}^{M \times N}$, where $M$ is the number of OOD data points.  

 Each row in the matrix indicates the similarity between an OOD data w.r.t. all the ID samples. Depending on the characteristics of the OOD data, this matrix may be sparse if OOD data is far away from all the ID samples (e.g., far OOD), or can have dense entries if it is close to some ID classes (e.g., near OOD). Our characterization is thus general enough to enable analysis under different scenarios.

 Now a question remains: how do we go from this matrix $\tilde{\*A}_{\rm OI}^{(l)}$ to a $k$-dimensional embedding for each OOD data? While a naive solution is to perform spectral decomposition on the stack of two matrices $\tilde{\*A}_{\rm OI}^{(l)}$ and $\tilde{\*A}^{(l)}$, this violates the principle of post hoc OOD detection as it incurs retraining. Instead, we derive the embeddings of OOD vertices using
existing ID embeddings $\*F^{(l)}$ and the OOD-ID similarity. This can be achieved by solving the following optimization: 

\begin{equation}
     \min _{\*F_{\rm ood}^{(l)} \in \mathbb{R}^{M \times k}} \left\|\tilde{\*A}_{\rm OI}^{(l)} - \*F_{\rm ood}^{(l)} \*F^{(l)\top}\right\|_F^2,
     \label{aoi_optimization}
\end{equation}
where $\*F_{\rm ood}^{(l)}$ denotes the OOD embeddings. Intuitively, the objective distills the similarity in the input space into the representation space. 
For instance, it searches for an OOD representation, aligning it closely with ID representations when there is a dense connectivity between OOD and ID data in the adjacency matrix, and vice versa. Similar to the ID case, we have $\*Z_{\rm ood}^{(l)} =[{\*D_{\rm ood}^{(l)}}]^{-\frac{1}{2}}\*F_{\rm ood}^{(l)} $, where $\*D_{\rm ood}^{(l)}$ can be calculated in the same way as $\*D^{(l)}$ based on $\tilde{\*A}_{\rm OI}^{(l)}$. Therefore, the analytic form of the OOD representations from the neural network can be derived as \begin{equation}
\*Z_{\rm ood}^{(l)} = [\*D_{\rm ood}^{(l)}]^{-\frac{1}{2}} \tilde{\*A}_{\rm OI}^{(l)}  \mathbf{V}_k^{(l)}  [\mathbf{\Sigma}_k^{(l)}]^{-\frac{1}{2}}.
    \label{labeled_ood_feat}
\end{equation}
Detailed proof and the design rationale are in Appendix~\ref{derivation_represe}.

\paragraph{Representation in the unlabeled case.} For the unlabeled case, we can get the representations for ID and OOD data by replacing the matrices $\*V_k^{(l)}, \*\Sigma_k^{(l)}, \*D^{(l)}, \*D_{\rm ood}^{(l)}$ and $\tilde{\*A}_{\rm OI}^{(l)}$  in Equations~\ref{labeled_id_feat}  and~\ref{labeled_ood_feat} with the unlabeled version. We will show how the labeled and unlabeled representations are rigorously related in Appendix~\ref{Proof_theorem2}.

\paragraph{An illustrative example.} To contrast the adjacency matrices and the corresponding representations for the \textbf{u}nlabeled and \textbf{l}abeled cases, we simulate an example in Figure~\ref{fig:toy_example}. The simulation is constructed with simplicity in mind, to
facilitate understanding. Evaluations on complex high-dimensional data will be provided in Section~\ref{exp_main}. In particular, we base our analysis on the ID adjacency matrix as depicted in  Figure~\ref{fig:toy_example} (a) and (b), which consists of three ID classes and 40 data points for each class. In the labeled case, the ID adjacency matrix has a denser connectivity pattern, especially for data that belongs to the same ID label. In Figure~\ref{fig:toy_example} (c) and (d), we compare the OOD-ID adjacency matrices with and without ID labels in two scenarios, i.e., near OOD where there are dense connections in $\tilde{\*A}_{\rm OI}$ and far OOD where the connectivity in $\tilde{\*A}_{\rm OI}$ is sparse.

Based on the graph, we further visualize in Figures~\ref{fig:toy_example} (e) and (f) the 2D data representations ($k=2$, calculated by Equations~\ref{labeled_id_feat}  and~\ref{labeled_ood_feat}) for the near OOD and far OOD scenarios.  We observe that having different adjacency matrices can lead to significantly different data representations. We will provide theory to rigorously analyze the OOD detection performance and contrast between the labeled and unlabeled cases (Section~\ref{sec:4.3}). Details of the illustrative example are included in Appendix~\ref{detail_toy_example_app}.

\subsection{Evaluation Target}
\label{sec:4.2}
With the closed-form representations for both ID and OOD, we evaluate OOD detection by linear probing error. The strategy is
commonly used in representation learning~\cite{chen2020simple}. 
Specifically, the weight of a linear classifier is denoted as $\boldsymbol{\theta}\in \mathbb{R}^{k\times 2}$. The class prediction (ID vs. OOD) is given by $\*g_{\boldsymbol{\theta}}(\*z) = {\rm argmax}_{i \in \{+,-\}} (\*h_\*w(\*x) \boldsymbol{\theta})_i$. Denote the set of ID and OOD features as ${\*Z}_{\rm all} \in \mathbb{R}^{(N+M) \times  k} = [\*Z^\top, \*Z_{\rm ood}^\top]^\top$ (either labeled (u) or unlabeled (l)), the linear probing
error $R({\*Z}_{\rm all}) $ is given by the least error of all possible linear classifiers:
\begin{equation}
   R({\*Z}_{\rm all}) \triangleq \min _{\boldsymbol{\theta} \in \mathbb{R}^{k \times 2 }} \mathbb{E}_{\*z \in {\*Z}_{\rm all}} \mathds{1}[y(\*z) \neq \*g_{\boldsymbol{\theta}}(\*z)],
\end{equation}
where $y(\*z)$ denotes indicates the ground-truth class of feature $\*z$ (ID or OOD). With the definition, we can bound the
linear probing error $R({\*Z}_{\rm all}) $ by the residual of the regression error as shown in Lemma~\ref{lemma2_main} with proof in Appendix~\ref{proof_upper_bound_lp_loss}.

\begin{lemma}\label{lemma2_main}
    Denote $\*y\in \mathbb{R}^{(N+M)\times 2}$ as a matrix where each row contains the one-hot label for features in ${\*Z}_{\rm all}$. We have:
    \begin{equation}
    \label{eq:lp_loss}
        R({\*Z}_{\rm all}) \leq  \frac{2}{N+M}\operatorname{Tr}\left(\left(\mathbf{I}- {\*Z}_{\rm all} {\*Z}_{\rm all}^{\dagger}\right) \mathbf{y} \mathbf{y}^\top\right).
    \end{equation}
\end{lemma}
Here $\operatorname{Tr}\left(\cdot \right)$ denotes the trace operator. $ {\*Z}_{\rm all}^{\dagger}$ is the Moore-Penrose inverse of matrix ${\*Z}_{\rm all}$. We denote this upper bound as $\overline{R}({\*Z}_{\rm all}) $, which is more tractable to analyze and behaves similarly to $R({\*Z}_{\rm all})$ as shown in Appendix~\ref{proof_upper_bound_lp_loss}. Therefore our subsequent analysis revolves around it.

\subsection{Error Bound on OOD Detection Performance}
\label{sec:4.3}
With the evaluation target defined above, we now present the formal error bound on OOD detection performance
by contrasting the labeled and unlabeled cases. As an overview, 
 Theorem~\ref{MainT-2}  will present the lower bound of linear probing 
 error difference between the unlabeled and label case, along with an intuitive version in Theorem~\ref{MainT-3}. We specify several mild assumptions and necessary notations for our theorems in Appendix~\ref{notation,definition,Ass,Const}. {Due to space limitation, we omit unimportant constants and simplify the statements of our theorems. We defer the \textbf{full formal} statements in  Appendix \ref{main_theorems}. All proofs can be found in Appendix \ref{Proof_all_app}.}

\paragraph{Error difference between unlabeled and labeled cases.} Formally, we investigate the following linear probing error
difference between the unlabeled and labeled case:
\begin{equation}
\mathcal{G} = \overline{R}(  {\*Z}_{\rm all}^{(u)}) - \overline{R}({\*Z}_{\rm all}^{(l)}),
\label{eq:error_diff}
\end{equation}
where a larger error difference indicates that labeled ID data benefits OOD detection more substantially, and vice versa. The lower bound on $\mathcal{G}$ is given by the following theorem.

\vspace{0.2cm}
\begin{theorem}[Lower bound of the linear probing error difference w/ and w/o ID labels]\label{MainT-2} (Informal.) Suppose we have  adjacency matrices $\tilde{\*A}^{(u)}, \tilde{\*A}^{(l)}\in \mathbb{R}^{N \times N}$ and $\tilde{\*A}_{\rm OI}^{(u)}, \tilde{\*A}_{\rm OI}^{(l)} \in \mathbb{R}^{M \times N}$  for both the labeled and unlabeled cases. Under mild conditions, given positive constants $\phi_l, C$, the error difference $ \mathcal{G}$  in Equation~\ref{eq:error_diff} is lower bounded by
\begin{equation}
\begin{aligned}
\mathcal{G} &\geq      \frac{C   \phi_l}{N+M} \epsilon(\mathfrak{p}, \mathfrak{q},\tilde{\*A}^{(u)},\tilde{\*A}_{\rm OI}^{(u)} ), 
\label{main_error_bound}
\end{aligned}
\end{equation}
where $\mathfrak{q} \in \mathbb{R}^{N\times c}$ with each column defined as $(\mathfrak{q}_i)_\*x = \mathbb{E}_{\Bar{\*x}_{l} \sim {\mathbb{P}_{l_i}}} \mathcal{T}(\*x | \Bar{\*x}_{l}), \*x \in \mathcal{X}_{\rm id}$. Similarly, $\mathfrak{p}\in \mathbb{R}^{M\times c}$ is defined as $(\mathfrak{p}_i)_\*x = \mathbb{E}_{\Bar{\*x}_{l} \sim {\mathbb{P}_{l_i}}} \mathcal{T}(\*x | \Bar{\*x}_{l}), \*x \in \mathcal{X}_{\rm ood}$. Semantically, each entry in $\mathfrak{q}$ means the connection magnitude from $\*x$ to all ID data while each entry in $\mathfrak{p}$ is the connection from $\*x$ to OOD data. Furthermore,
\begin{equation*}
\begin{aligned}
  &  \epsilon(\mathfrak{p}, \mathfrak{q},\tilde{\*A}^{(u)},\tilde{\*A}_{\rm OI}^{(u)} )  =  2 \sum_{i=1}^{c} \operatorname{Tr}\left(\mathfrak{p}_i\mathfrak{q}_i^\top \cdot   \tilde{\*A}_{\rm OI}^{(u)\top}\right)  +  \\ & \left(1- \|\tilde{\*A}_{\rm OI}^{(u)} \|_F^2 \|\tilde{\*A}^{(u)}\|_F^2  \right)  \sum_{i=1}^c \|\mathfrak{q}_i\|_F^2 + \\ & 
           r^2\left( \|\tilde{\*A}_{\rm OI}^{(u)}\|_F^2 \| \tilde{\*A}^{(u)}\|_F^2  \frac{ 2(\tau- k)}{\tau-1} -2\right)   \cdot \sum_{i=1}^c \|\mathfrak{q}_{i}\|_1,
\end{aligned}
           \label{epsilon_def_main}
\end{equation*}
where $\tau$ is a constant that measures the $k$-th spectral gap of matrix $\tilde{\*A}^{(u)}$, i.e., $\boldsymbol{\lambda}_{k}^{(u)} \geq \tau \boldsymbol{\lambda}_{k+1}^{(u)}$ and $\boldsymbol{\lambda}_{k}^{(u)} $ is the $k$-th largest singular value of $\tilde{\*A}^{(u)}$. $r$ is the maximum $l_2$ norm of the ID representations, i.e.,  $r = \max_{\*z \in \*Z^{(u)}} \|\*z\|_2$.
\end{theorem}
Theorem~\ref{MainT-2} is a general characterization of the error difference in the labeled and unlabeled cases. To gain a better insight, we introduce
Theorem~\ref{MainT-3} which provides intuition interpretations. 

\begin{table*}[t]
    \centering
    \small
    \begin{tabular}{cc|ccccccc}
    \hline

   OOD category&     OOD dataset & ID labels & FPR95$\downarrow$ & AUROC$\uparrow$ & LP error$\downarrow$ &   FPR95$\downarrow$ & AUROC$\uparrow$ & LP error$\downarrow$\\
        \hline
       &   &  &\multicolumn{3}{c}{ $\mathbb{P}_{\rm ood}^{\rm test} =\mathbb{P}_{\rm ood}^{\rm LP} $}      &  \multicolumn{3}{c}{  $\mathbb{P}_{\rm ood}^{\rm test} \neq\mathbb{P}_{\rm ood}^{\rm LP} $}     \\
        \cline{4-9}
\multirow{10}{*}{\textsc{Far OOD}}      &  \multirow{2}{*}{\textsc{Svhn} }     & - & 0.09$^{\pm 0.02}$&99.96$^{\pm 0.01}$ &\textbf{0.02}$^{\pm 0.00}$   & 75.62$^{\pm 4.74}$ & 77.09$^{\pm 2.17}$ & 0.52$^{\pm 0.21}$ \\
  &   &+& \cellcolor[HTML]{EFEFEF}\textbf{0.07}$^{\pm 0.00}$&\cellcolor[HTML]{EFEFEF}\textbf{99.97}$^{\pm 0.03}$ & \cellcolor[HTML]{EFEFEF}\textbf{0.02}$^{\pm 0.00}$&   \cellcolor[HTML]{EFEFEF}\textbf{68.73}$^{\pm 5.03}$ & \cellcolor[HTML]{EFEFEF}\textbf{82.97}$^{\pm 1.87}$  & \cellcolor[HTML]{EFEFEF}\textbf{0.48}$^{\pm 0.28}$
    \\
   &    \multirow{2}{*}{\textsc{Textures} }     &-&  0.37$^{\pm 0.20}$&99.80$^{\pm 0.15}$& 0.02$^{\pm 0.00}$ &\textbf{86.35}$^{\pm 2.47}$ &   70.94$^{\pm 3.77}$ & 1.05$^{\pm 0.17}$ \\
&     & +& \cellcolor[HTML]{EFEFEF}\textbf{0.24}$^{\pm 0.19}$&\cellcolor[HTML]{EFEFEF}\textbf{99.86}$^{\pm 0.11}$ & \cellcolor[HTML]{EFEFEF}\textbf{0.01}$^{\pm 0.00}$ & \cellcolor[HTML]{EFEFEF}86.44$^{\pm 0.58}$ & \cellcolor[HTML]{EFEFEF}\textbf{75.15}$^{\pm 3.36}$ &  \cellcolor[HTML]{EFEFEF}\textbf{0.95}$^{\pm 0.10}$\\
  &     \multirow{2}{*}{\textsc{Places365} }     &- & 0.68$^{\pm 0.10}$&99.98$^{\pm 0.00}$ & \textbf{0.02}$^{\pm 0.01}$& 75.35$^{\pm 2.78}$   & 79.85$^{\pm 1.85}$ & 0.83$^{\pm 0.27}$\\
  &   &+&\cellcolor[HTML]{EFEFEF}\textbf{0.44}$^{\pm 0.06}$&\cellcolor[HTML]{EFEFEF}\textbf{99.99}$^{\pm 0.01}$ & \cellcolor[HTML]{EFEFEF}\textbf{0.02}$^{\pm 0.01}$&  \cellcolor[HTML]{EFEFEF}\textbf{66.60}$^{\pm 1.61}$ & \cellcolor[HTML]{EFEFEF}\textbf{85.03}$^{\pm 3.19}$ & \cellcolor[HTML]{EFEFEF}\textbf{0.74}$^{\pm 0.09}$\\ 
 &    \multirow{2}{*}{\textsc{Lsun-Resize} }     & - & \textbf{0.24}$^{\pm 0.03}$&\textbf{99.95}$^{\pm 0.03}$ &  \textbf{0.02}$^{\pm 0.00}$   &83.57$^{\pm 2.89}$ & 77.57$^{\pm 5.21}$& 0.84$^{\pm 0.21}$ \\
  &   &+& \cellcolor[HTML]{EFEFEF}\textbf{0.24}$^{\pm 0.01}$&\cellcolor[HTML]{EFEFEF}{99.91}$^{\pm 0.07}$&  \cellcolor[HTML]{EFEFEF}\textbf{0.02}$^{\pm 0.01}$&\cellcolor[HTML]{EFEFEF}\textbf{74.13}$^{\pm 4.95}$ & \cellcolor[HTML]{EFEFEF}\textbf{82.71}$^{\pm 1.64}$ & \cellcolor[HTML]{EFEFEF}\textbf{0.75}$^{\pm 0.14}$\\ 
   &    \multirow{2}{*}{\textsc{Lsun-C} }     & - & 1.68$^{\pm 0.36}$&99.20$^{\pm 0.17}$ &0.05$^{\pm 0.01}$ & 63.42$^{\pm 6.16}$ & 83.43$^{\pm 3.48}$ & 0.82$^{\pm 0.38}$ \\
  &   &+&  \cellcolor[HTML]{EFEFEF}\textbf{1.04}$^{\pm 0.41}$&\cellcolor[HTML]{EFEFEF}\textbf{99.35}$^{\pm 0.08}$ &  \cellcolor[HTML]{EFEFEF}\textbf{0.04}$^{\pm 0.02}$&\cellcolor[HTML]{EFEFEF}\textbf{51.36}$^{\pm 2.26}$ & \cellcolor[HTML]{EFEFEF}\textbf{89.49}$^{\pm 1.91}$  & \cellcolor[HTML]{EFEFEF}\textbf{0.72}$^{\pm 0.41}$\\
  \hline
\multirow{2}{*}{\textsc{Near OOD}}   &   \multirow{2}{*}{\textsc{Cifar10} }     & -& 62.20$^{\pm 3.49}$&85.93$^{\pm 1.72}$ &0.27$^{\pm 0.11}$ &94.54$^{\pm 0.79}$ & 55.42$^{\pm 1.64}$ & 0.97$^{\pm 0.18}$ \\
 &    &+& \cellcolor[HTML]{EFEFEF}\textbf{58.28}$^{\pm 2.90}$&\cellcolor[HTML]{EFEFEF}\textbf{89.01}$^{\pm 0.98}$ & \cellcolor[HTML]{EFEFEF}\textbf{0.19}$^{\pm 0.05}$&\cellcolor[HTML]{EFEFEF}\textbf{91.07}$^{\pm 3.28}$ & \cellcolor[HTML]{EFEFEF}\textbf{67.72}$^{\pm 0.97}$ & \cellcolor[HTML]{EFEFEF}\textbf{0.85}$^{\pm 0.32}$\\
    \hline
    \end{tabular}
    \caption{\small OOD detection results w/ and w/o ID labels (\textsc{Cifar100} as ID). Mean and std are estimated on three different runs. Better results are highlighted in bold. ``+,-" denotes the labeled and unlabeled case. ``LP error" denotes the error of linear probing. $\uparrow$ indicates larger values are better,
and $\downarrow$ indicates smaller values are better. The table shows that (1) the ID labeling information helps OOD detection, especially in the near OOD scenario and when  $\mathbb{P}_{\rm ood}^{\rm test} =\mathbb{P}_{\rm ood}^{\rm LP} $; Moreover, our observations can generalize to the case where the OOD distribution in linear probing is not the same as that in actual testing ($\mathbb{P}_{\rm ood}^{\rm test} \neq \mathbb{P}_{\rm ood}^{\rm LP} $), showcasing the generality of our theory.  }
    \label{exp_result}
\end{table*}

\begin{tcolorbox}
\begin{theorem}[Intuitive version of Theorem~\ref{MainT-2}]\label{MainT-3}
    Under the same conditions in Theorem~\ref{MainT-2}, assume the $k$-th spectral gap of $\tilde{\*A}^{(u)}$ is sufficiently large, i.e., $\tau>k$, then the main error component $ \epsilon(\mathfrak{p}, \mathfrak{q},\tilde{\*A}^{(u)},\tilde{\*A}_{\rm OI}^{(u)} )  $ in Equation~\ref{main_error_bound} satisfies
\begin{equation*}
\small
\begin{aligned}
     \epsilon(\mathfrak{p}, \mathfrak{q},&\tilde{\*A}^{(u)},\tilde{\*A}_{\rm OI}^{(u)} )  \geq \\& \left[1+ \|\tilde{\*A}_{\rm OI}^{(u)} \|_F^2\left( 2N^2-   \|\tilde{\*A}^{(u)}\|_F^2    \right)  \right] \|\mathfrak{q}\|_F^2.
\end{aligned}
\end{equation*}
\end{theorem}
\end{tcolorbox}
\paragraph{Interpretation and key insights.} Theorem~\ref{MainT-3} shows that the optimal scenarios for achieving the greatest reduction in linear probing error—signifying the most significant benefit by incorporating the ID labels, are when 
\begin{enumerate}
    \item {The ID data are connected sparsely in the unlabeled case (i.e., $ \|\tilde{\*A}^{(u)}\|_F^2 < 2N^2$), which always holds because $\|\tilde{\*A}^{(u)}\|_{\infty}< N$};
    \item {The OOD data is closely connected to the ID data (near OOD, i.e., $\|\tilde{\*A}_{\rm OI}^{(u)} \|_F^2$ is relatively large)}; 
    \item {The semantic connection between each ID data to the labeled ID data from different ID classes is sufficiently large (i.e., $\|\mathfrak{q}\|_F^2$ is large).}
\end{enumerate}

Moreover, the simplified bound also enables us to interpret the relationship of each key component with the error reduction $\mathcal{G}$. For example, 1) \textit{the bound will monotonically increase when the connection within ID data in the unlabeled case becomes sparser (i.e., $\|\tilde{\*A}^{(u)} \|_F^2$ $\downarrow$)}; 2) \textit{Since $\|\tilde{\*A}^{(u)} \|_F^2$ is smaller than $2N^2$, strengthening the semantic connection from each ID data to the labeled data from different ID classes (i.e., $\|\mathfrak{q}\|_F^2$ $\uparrow$) is always helpful. Intuitively, a larger $\|\mathfrak{q}\|_F^2$ means one ID data point is more likely to be augmented from another ID data;}
3) \textit{If the ID and OOD data are closer ($\|\tilde{\*A}_{\rm OI}^{(u)} \|_F^2$ $\uparrow$), the value of the bound will increase given the same $\mathfrak{q}$ and $\tilde{\*A}^{(u)}$.}  

\paragraph{Verification of bound on the illustrative example.} Our theoretical guarantee aligns well with the empirical results. For example, in Figure~\ref{fig:toy_example} (e) where the OOD data is densely connected to the ID data (near OOD case),  the ID labels can better shape the ID and OOD representations compared to the unlabeled case, rendering them linearly separable. As a result, the linear probing error is reduced from 0.09 in the unlabeled case to 0 in the labeled case.  In contrast, when the OOD data is far from the ID data (Figure~\ref{fig:toy_example} (f)), the representations for ID and OOD are already well separated in the unlabeled case, and thus the benefit of ID labels is relatively marginal. As a verification, the linear probing error remains 0 both with and without ID labels. Therefore, these observations align with our key insight on the effect of OOD-ID connection $\tilde{\*A}_{\rm OI}^{(u)}$ on the error difference for linear probing. Moreover, we provide additional visualization results on changing the Frobenius norm of the ID adjacency matrix $\tilde{\*A}^{(u)}$ and the semantic connection $\mathfrak{q}$ in Appendix~\ref{result_toy_example_app}. 

%% file: exp.tex
\label{exp_main} 
In this section, we verify our theoretical results using real-world OOD detection benchmarks.

\paragraph{Experimental setup.} For ID datasets, we use \textsc{Cifar10} and \textsc{Cifar100}~\cite{krizhevsky2009learning}. We first train the neural network on the ID data for 200 epochs with a ResNet-18~\cite{he2016deep}, using objective $\mathcal{L}_\text{unlabeled}$ and $\mathcal{L}_\text{labeled}$ for the unlabeled and labeled case, respectively. The penultimate layer embedding dimension $k=512$. We then extract the embeddings for the ID and OOD data and perform linear probing (50 epochs). We explore two different scenarios depending on whether the OOD in linear probing ($\mathbb{P}_{\rm ood}^{\rm LP}$) is the same as the test OOD ($\mathbb{P}_{\rm ood}^{\rm test}$). In the first scenario ($\mathbb{P}_{\rm ood}^{\rm test} =\mathbb{P}_{\rm ood}^{\rm LP} $), we use 75\% of the OOD dataset for linear probing and the remaining for testing. Specifically, for far OOD test datasets, we use a suite of natural image datasets including \textsc{Textures}~\citep{cimpoi2014describing}, \textsc{Svhn}~\citep{netzer2011reading}, \textsc{Places365}~\citep{zhou2017places}, and \textsc{Lsun}~\citep{DBLP:journals/corr/YuZSSX15}. For near OOD, we evaluate on \textsc{Cifar10} when \textsc{Cifar100} is ID and vice versa. In the second scenario ($\mathbb{P}_{\rm ood}^{\rm test} \neq \mathbb{P}_{\rm ood}^{\rm LP} $), we use \textsc{300K Random Images}~\cite{hendrycks2018deep} for linear probing and the other OOD datasets for evaluation. More experiment details are in Appendix~\ref{add_experiment_detail}.

\paragraph{Evaluation metrics.} We report the following metrics: (1) the false positive rate (FPR95$\downarrow$) of OOD samples when the true positive rate of ID samples is 95\%, (2) the area under the receiver operating characteristic curve (AUROC$\uparrow$), and (3) the linear probing error (LP error $\downarrow$).

\paragraph{Experiment results.} The results are shown in Table~\ref{exp_result}, which demonstrate that: (1) the ID labeling information helps OOD detection, especially in the near OOD scenario and when  $\mathbb{P}_{\rm ood}^{\rm test} =\mathbb{P}_{\rm ood}^{\rm LP} $. For example, the AUROC is improved by 3.08\% on \textsc{Cifar10} compared to 0.02\% on \textsc{Svhn}, echoing our theoretical insights; (2) Our observations can generalize to the case where the OOD distribution in linear probing is not the same as that in actual testing ($\mathbb{P}_{\rm ood}^{\rm test} \neq \mathbb{P}_{\rm ood}^{\rm LP} $), where the AUROC increases by 12.3\% compared to the unsupervised counterpart on \textsc{Cifar10}, showcasing the flexibility and generality of our theory. Additional results on \textsc{Cifar10} as the ID dataset and the evaluation using post-hoc OOD detection score are shown in Appendix~\ref{add_result_c10_app}.

\paragraph{Verification of bound.} We verify the error difference $\mathcal{G}$ and its relationship to the Frobenius norm of the adjacency matrices $\tilde{\*A}^{(u)}$ and $\tilde{\*A}_{\rm OI}^{(u)}$. Firstly, to verify how the value of $\mathcal{G}$ will change given a larger Frobenius norm of $\tilde{\*A}_{\rm OI}^{(u)}$, we compare the linear probing error on \textsc{Svhn} and \textsc{Cifar10} in Table~\ref{tab:veri_a_ood_id} with and without ID labels, where the error difference on \textsc{Cifar10} (near OOD, with larger $\|\tilde{\*A}_{\rm OI}^{(u)}\|_F$) is consistently larger than that on \textsc{Svhn} (far OOD). 

\begin{wraptable}{r}{0.55\linewidth}
    \centering
     \tabcolsep 0.04in\renewcommand\arraystretch{0.745}{\small {}}%
    \small
    \begin{tabular}{c|cc}
    \hline 
    OOD dataset & \multicolumn{1}{c}{\textsc{Svhn}}  & \multicolumn{1}{c}{\textsc{C10}}\\
   \cline{2-3}
   & {\scriptsize \textsc{Far OOD}}& {\scriptsize\textsc{Near OOD}}\\
   \hline
 $\| \tilde{\*A}_{\rm OI}^{(u)} \|_F$  $\uparrow$  & $2334$& $2739$ \\
$\mathcal{G}$ $\uparrow$ & 0.00 & \textbf{0.08} \\
       \hline
    \end{tabular}
    \caption{\small Verification with different $\| \tilde{\*A}_{\rm OI}^{(u)} \|_F$ (\textsc{Cifar100} as ID).}
    \label{tab:veri_a_ood_id}
\end{wraptable} In addition, we verify the relationship of $\|\tilde{\*A}^{(u)}\|_F$ and the error difference (\textsc{Cifar10} as OOD) in Table~\ref{tab:veri_a_id_id}. Specifically, we calculate the norm of the ID adjacency matrix from different training epochs and observe that the difference in linear probing error tends to increase with decreasing $\|\tilde{\*A}^{(u)}\|_F$, which aligns with Theorem~\ref{MainT-3}. Additional results and details are included in Appendix~\ref{sec:verification_discrepancy}. We further analyze the tightness of our bound in Appendix~\ref{sec:tightness_app}.

\begin{table}[t]
    \centering
    \tabcolsep 0.04in\renewcommand\arraystretch{0.8}{\small {}}%
    \small
    \begin{tabular}{c|ccccccc}
    \hline 
    Epochs  & 40 & 80 & 120  & 160& 200 & 240 \\
    \hline
      $\| \tilde{\*A}^{(u)} \|_F$ $\downarrow$   &20191 & 19549& 18939 & 16073& 15509 & 14810\\
      $\mathcal{G}$ $\uparrow$ & 0.01& 0.03 & 0.06 & 0.06& 0.08 &0.09\\
       \hline
    \end{tabular}
    \caption{\small Verification with different $\| \tilde{\*A}^{(u)} \|_F$ (\textsc{Cifar100} as ID).}
    \label{tab:veri_a_id_id}
\end{table}

%% file: related_work.tex
 \paragraph{OOD detection} has attracted a surge of interest in recent years~\citep{fort2021exploring,yang2021generalized,fang2022learnable,zhu2022boosting,ming2022delving,ming2022spurious,yang2022openood,wang2022outofdistribution,galil2023a,djurisic2023extremely,zheng2023out,wang2022watermarking,wang2023outofdistribution,narasimhan2023learning,yang2023auto,uppaal2023fine,zhu2023diversified,zhu2023unleashing,ming2023finetune,zhang2023openood,ghosal2024how}. One line of work performs OOD detection by devising scoring functions, including confidence-based methods~\citep{bendale2016towards,hendrycks2016baseline,liang2018enhancing}, energy-based score~\citep{liu2020energy,wang2021canmulti,wu2023energybased}, distance-based approaches~\citep{lee2018simple,tack2020csi,DBLP:journals/corr/abs-2106-09022,2021ssd,sun2022out, du2022siren, ming2023cider,ren2023outofdistribution}, gradient-based score~\citep{huang2021importance}, and {Bayesian approaches~\citep{gal2016dropout,lakshminarayanan2017simple,maddox2019simple,dpn19nips,Wen2020BatchEnsemble,kristiadi2020being}.} Another  line of work addressed OOD detection by training-time regularization~\citep{bevandic2018discriminative,malinin2018predictive,geifman2019selectivenet,hein2019relu,meinke2019towards,DBLP:conf/nips/JeongK20,liu2020simple,DBLP:conf/icml/AmersfoortSTG20,DBLP:conf/iccv/YangWFYZZ021,DBLP:conf/icml/WeiXCF0L22,du2022unknown,du2023dream,wang2023learning}.  For example, the model is
regularized to produce lower confidence~\citep{lee2018training} or higher energy~\citep{liu2020energy,du2022towards} on a set of clean OOD data~\cite{hendrycks2018deep,DBLP:conf/icml/MingFL22}, wild data~\cite{zhou2021step,katzsamuels2022training,he2023topological,bai2023feed, du2024sal} and synthetic outliers~\cite{du2023dream,tao2023nonparametric,park2023powerfulness}. Additionally, a similar topic in a different domain, i.e., anomaly detection, often restricts the ID
(normality) to be with a single class and with a different definition of the outliers~\cite{chandola2009anomaly,han2022adbench}. 

Most OOD detection methods rely on the supervision of ID labels, and there have been prior works, such as~\cite{tack2020csi,sehwag2021ssd,sun2022out} that empirically verified that training with the ID labels can achieve a much better OOD detection performance compared to the unsupervised version, but a provable analysis of their impact is critical yet missing in the field.

\paragraph{OOD detection theory.} Recent studies have begun to focus on the theoretical understanding of OOD detection. ~\citet{fang2022learnable} studied the generalization of
OOD detection by PAC learning and they found a necessary condition for the learnability of OOD
detection.~\citet{morteza2022provable} derived a novel OOD score and provided a provable understanding of the OOD detection result using that score.~\citet{du2024sal} theoretically studied the impact of unlabeled data for OOD detection. In contrast, we formally analyze the impact of ID labels on OOD detection, which has not been studied in the past.

\paragraph{Spectral graph theory} is a classical research problem~\cite{von2007tutorial,chung1997spectral,cheeger2015lower,kannan2004clusterings,lee2014multiway,mcsherry2001spectral}, which aims to partition the graph by studying the eigenspace of the adjacency matrix. Recently, it has been applied in different applications in machine learning~\cite{ng2001spectral,shi2000normalized,blum2001learning,zhu2003semi,argyriou2005combining,shaham2018spectralnet}.~\citet{haochen2021provable} derived the spectral contrastive learning from the factorization of the graph's adjacency matrix, and provably understand unsupervised domain adaptation~\cite{shen2022connect,haochen2022beyond}. \citet{sun2023graph,sun2023when} expanded the spectral contrastive learning approach to novel class discovery and open-world semi-supervised learning. Our focus is on the OOD detection problem, which differs from prior literature.

%% file: conclusion.tex
In this paper, we propose a novel analytical framework that studies the impact of ID labels on OOD detection. Our framework takes a graph-theoretic approach by modeling
the ID data via a graph, which allows us to characterize the feature representations by performing
spectral decomposition on the graph that can be expressed
equivalently as a contrastive learning objective on neural
net representation. Leveraging these representations, we establish a provable 
error bound that compares the OOD detection
performance with and without
ID labels, which reveals sufficient conditions for
achieving improved OOD detection performance. Empirical observations further support our theoretical conclusions, showcasing the benefits of ID labeling information under proper
conditions. We hope our work will inspire future research on the theoretical understanding of OOD detection. One promising direction is to analyze the setting where there is access to OOD samples, which belongs to an important branch of work in OOD detection literature.

\section*{Acknowledgement}
We thank Hyeong Kyu (Froilan) Choi and Shawn Im for their valuable suggestions on the draft. The authors would also like to thank
ICML anonymous reviewers for their helpful feedback. Du is supported by the Jane Street Graduate
Research Fellowship. Li gratefully acknowledges the support from the AFOSR
Young Investigator Program under award number FA9550-23-1-0184, National Science Foundation
(NSF) Award No. IIS-2237037 \& IIS-2331669, Office of Naval Research under grant number
N00014-23-1-2643, Philanthropic Fund from SFF, and faculty research awards/gifts from Google
and Meta.

\section*{Impact Statement}
  \label{sec:broader}
 Our paper aims to improve the reliability and safety of modern machine learning models.  From the theoretical perspective, our analysis can facilitate and deepen the understanding of the effect of in-distribution labels for OOD detection. In Appendix~\ref{sec:verification_discrepancy} and Section~\ref{exp_main} of the main paper, we properly verify the necessary conditions and the value of our bound using real-world datasets. Hence, we believe our theoretical framework has a broad utility and significance.

 From the practical side,  our study can lead to direct benefits and societal impacts by deploying OOD detection techniques, particularly when practitioners need to determine how to better leverage the in-distribution labels, such as in safety-critical applications
i.e., autonomous driving and healthcare data analysis.  Our study does not involve any human subjects or violation of legal compliance. We do not anticipate any potentially harmful consequences to our work. Through our study and releasing our code, we hope to raise stronger research and societal awareness on the safe handling of out-of-distribution data in real-world settings. 

%% file: notations_app.tex
\begin{center}
    \Large{\textbf{When and How Does In-Distribution Label Help Out-of-Distribution Detection?\\(Appendix)}}
\end{center}

\section{Notations, Assumptions and Important Constants}\label{notation,definition,Ass,Const}
Here we summarize the important notations and constants in Table~\ref{tab: notation}, and restate necessary definitions and assumptions in Section \ref{sec:assumption_app}. 
\subsection{Notations}
Please see Table \ref{tab: notation} for detailed notations.
\begin{table}[!h]
    \centering
    \vspace{-1em}
    \caption{Main notations and their descriptions.}
    \begin{tabular}{cl}
    \toprule[1.5pt]
         \multicolumn{1}{c}{Notation} & \multicolumn{1}{c}{Description} \\
    \midrule[1pt]
    \multicolumn{2}{c}{\cellcolor{greyC} Spaces} \\
    $\mathcal{X}$,  $\mathcal{Z}$, $\mathcal{Y}$     & the input, representation space and the label space. \\

    \multicolumn{2}{c}{\cellcolor{greyC} Distributions} \\
 $\mathbb{P}_{{\rm id}}$, $\mathbb{P}_{{\rm ood}}$& data distribution for ID data and OOD data
    \\
     
    $\mathbb{P}_{\mathcal{X}\mathcal{Y}}$ & the joint data distribution for ID data.\\
    
    \multicolumn{2}{c}{\cellcolor{greyC} Data and Models} \\
    $\*x$ & input to the neural network\\
      $\tilde{\*A}_{\rm OI}^{(u)}$ and   $\tilde{\*A}_{\rm OI}^{(l)}$ &  adjacency matrix between OOD and ID data in the labeled and unlabeled case\\
    $\*A^{(u)}$ and  $\*A^{(l)}$ & adjacency matrix for ID data in the labeled and unlabeled case \\
 $\tilde{\*A}^{(u)}$ and  $\tilde{\*A}^{(l)}$ & normalized adjacency matrix for ID data in the labeled and unlabeled case \\
    $\*D^{(u)}$ and $\*D^{(l)}$ &ID diagonal matrix with the diagonal elements being the row sum of $\*A^{(u)}$ and  $\*A^{(l)}$. \\
    $\*D^{(u)}_{\rm ood}$ and $\*D^{(l)}_{\rm ood}$ &OOD diagonal matrix with the diagonal elements being the row sum of $\tilde{\*A}_{\rm OI}^{(u)}$ and   $\tilde{\*A}_{\rm OI}^{(l)}$. \\
     $\mathcal{T}$ &augmentation graph\\
    $\mathbf{w}$, $\boldsymbol{\theta}$ & weight of the ID feature extractor and the linear probing layer\\
     $\mathbf{h}_{\mathbf{w}}$ and $\mathbf{g}_{\boldsymbol{\theta}}$ & feature extractor on ID data and linear probing layer for OOD detection \\
    $y$, $\*y$ &binary label for linear probing, vectorized one-hot label for $y$ \\
    $\*z$  & feature for single input $\*x$ \\
     $\*{Z}^{(u)}$ and  $\*{Z}^{(l)}$ & representation matrix for ID data in the labeled and unlabeled case \\
     $\*{Z}_{\rm ood}^{(u)}$ and  $\*{Z}_{\rm ood}^{(l)}$ & representation matrix for OOD data in the labeled and unlabeled case \\
      $\*{Z}_{\rm all}^{(u)}$ and  $\*{Z}_{\rm all}^{(l)}$ & representation matrix for both ID and OOD data in the labeled and unlabeled case \\
    $\mathfrak{q}$ and $\mathfrak{p}$ & semantic connection from each ID/OOD data to the labeled ID data.\\
      $N$, $M$ & size of $\*{Z}$, size of $\*{Z}_{\rm ood}$\\
          $\boldsymbol{\lambda}$ & eigenvalue vector of $\tilde{\*A}$\\
          $\*v_j$&the $j$-th eigenvector of $\tilde{\*A}$ \\
          $\*V_k, \*V_{\varnothing}$ & the first $k$ eigenvectors of $\tilde{\*A}$/the latter $N-k$ eigenvectors of $\tilde{\*A}$ (null space of $\*V_k$)\\
      
    \multicolumn{2}{c}{\cellcolor{greyC} Distances} \\

    $\| \cdot \|_1, \| \cdot \|_2, \| \cdot \|_F$ & $\ell_1, \ell_2$ norm and Frobenius norm\\
    
    \multicolumn{2}{c}{\cellcolor{greyC} Loss and Risk} \\
    $\mathcal{L}_{\rm unlabeled}(\cdot), \mathcal{L}_{\rm labeled}(\cdot)$ & ID loss function in the unlabeled and labeled case\\
    $R({\*Z_{\rm all}})$ & the empirical risk w.r.t. linear probing module over feature set $\*Z_{\rm all}$  \\
      $\overline{R}({\*Z_{\rm all}})$ & the upper bound of $R(\*Z_{\rm all})$ \\
      $\mathcal{G}$ & linear probing error difference with and without ID labels\\
        \multicolumn{2}{c}{\cellcolor{greyC} Constants} \\ 
        $\phi_u,\phi_l$  & weight coefficients for the unlabeled and labeled case.\\
  $k$ & dimension of the feature representation\\
         $\tau$ & the parameter that measures the eigengap of $\tilde{\*A}^{(u)}$\\
           $r$ & the maximum $\ell_2$ norm of the ID features, i.e.,  $r = \max_{\*z \in \*Z^{(u)}} \|\*z\|_2$  \\
          $C$& constants in Theorem~\ref{MainT-2-app}\\

    \bottomrule[1.5pt]
    \end{tabular}
    
    \label{tab: notation}
\end{table}

%% file: assumption_app.tex
\label{sec:assumption_app}

\begin{assumption}[Property of ID adjacency matrix]\label{Ass1}
    We assume the adjacency matrix $\tilde{\*A}^{(u)}$ has the following property: There exists a positive constant $k \leq N$ such that the $k$-th eigengap of  $\tilde{\*A}^{(u)}$ satisfies that $\boldsymbol{\lambda}_{k}^{(u)} \geq \tau \boldsymbol{\lambda}_{k+1}^{(u)}, \tau > k$, where $\boldsymbol{\lambda}_{k}^{(u)}$ is the $k$-th largest eigenvalue of matrix $\tilde{\*A}^{(u)}$.
\end{assumption}
\begin{remark*} We have empirically verified our assumption using both simulated and real-world datasets in Section~\ref{sec:verification_discrepancy}.
\end{remark*}

\begin{assumption}[Property of vectors that depict the semantic connection between each ID/OOD data to the labeled ID data]\label{Ass2}

Similar to~\cite{sun2023graph}, we assume each row of the matrix $\mathfrak{q}$ lies in the linear span of $\*V_{k}^{(u)}$ and $\*V_{k}^{(u)}[\*\Sigma_{k}^{(u)}]^{-\frac{1}{2}}$, i.e., $\*V_{\varnothing}^{(u)\top} \mathfrak{q}_i = 0, \*V_{k}^{(u)}\*V_{k}^{(u)\top} \mathfrak{q}_i = \mathfrak{q}_i $ and $\*V_{k}^{(u)}[\*\Sigma_{k}^{(u)}]^{-1}\*V_{k}^{(u)\top} \mathfrak{q}_i = \mathfrak{q}_i$. 
\end{assumption}
\begin{remark*} Following the assumptions made in the existing literature~\cite{sun2023graph}, the assumption is used to simplify $\*V_{k}^{(u)} \*V_{k}^{(u) \top}   \mathfrak{q}_i \mathfrak{q}_i^{\top} \*V_{\varnothing}^{(u)} \*V_{\varnothing}^{(u)\top} + \*V_{k}^{(u)}\*V_{k}^{(u)\top} \mathfrak{q}_i \mathfrak{q}_i^\top $ to $ \mathfrak{q}_i \mathfrak{q}_i^\top $ and simplify $\*V_{k}^{(u)}[\*\Sigma_{k}^{(u)}]^{-1}\*V_{k}^{(u)\top} \mathfrak{q}_i$ to $\mathfrak{q}_i$.
\end{remark*}

%% file: theorem_app.tex
\label{main_theorems}
In this section, we provide a detailed and formal version of our main theorems with a complete description of the constant terms and other additional details that are omitted in the main paper.

\begin{theorem*}[Lower bound of the error difference between unlabeled and labeled cases, recap of Theorem~\ref{MainT-2} in the main paper]\label{MainT-2-app} 
Suppose we have  adjacency matrices $\tilde{\*A}^{(u)}, \tilde{\*A}^{(l)}\in \mathbb{R}^{N \times N}$ and $\tilde{\*A}_{\rm OI}^{(u)}, \tilde{\*A}_{\rm OI}^{(l)} \in \mathbb{R}^{M \times N}$  for both the labeled and unlabeled cases. If Assumption~\ref{Ass2} holds, given positive constants $\phi_l, C$, the error difference $ \mathcal{G}$  in Equation~\ref{eq:error_diff} is bounded by
\begin{equation}
\begin{aligned}
    &  \mathcal{G} \geq      \frac{C \phi_l}{N+M} \epsilon(\mathfrak{p}, \mathfrak{q},\tilde{\*A}^{(u)},\tilde{\*A}_{\rm OI}^{(u)} ) ,
\end{aligned}
\end{equation}
where $\mathfrak{q} \in \mathbb{R}^{c\times N}$ with each row being defined as $(\mathfrak{q}_i)_\*x = \mathbb{E}_{\Bar{\*x}_{l} \sim {\mathbb{P}_{l_i}}} \mathcal{T}(\*x | \Bar{\*x}_{l}), \*x \in \mathcal{X}_{\rm id}$. Similarly, we have $\mathfrak{p}\in \mathbb{R}^{c\times M}$ and $(\mathfrak{p}_i)_\*x = \mathbb{E}_{\Bar{\*x}_{l} \sim {\mathbb{P}_{l_i}}} \mathcal{T}(\*x | \Bar{\*x}_{l}), \*x \in \mathcal{X}_{\rm ood}$. Furthermore,
\begin{equation}
\begin{aligned}
    \epsilon(\mathfrak{p}, \mathfrak{q},\tilde{\*A}^{(u)},\tilde{\*A}_{\rm OI}^{(u)} )  &=  2 \sum_{i=1}^{c} \operatorname{Tr}\left(\mathfrak{p}_i\mathfrak{q}_i^\top \cdot   \tilde{\*A}_{\rm OI}^{(u)\top}\right)  +  \left(1- \|\tilde{\*A}_{\rm OI}^{(u)} \|_F^2 \|\tilde{\*A}^{(u)}\|_F^2  \right)  \sum_{i=1}^c \|\mathfrak{q}_i\|_F^2 +  \\&
           r^2\left( \|\tilde{\*A}_{\rm OI}^{(u)}\|_F^2 \| \tilde{\*A}^{(u)}\|_F^2  \frac{ 2(\tau- k)}{\tau-1} -2\right)   \cdot \sum_{i=1}^c \|\mathfrak{q}_{i}\|_1,
\end{aligned}
           \label{eq14_app}
\end{equation}
where $\tau$ is a constant that measures the $k$-th spectral gap of matrix $\tilde{\*A}^{(u)}$, i.e., $\boldsymbol{\lambda}_{k}^{(u)} \geq \tau \boldsymbol{\lambda}_{k+1}^{(u)}$ and $\boldsymbol{\lambda}_{k}^{(u)} $ is the $k$-th largest singular value of $\tilde{\*A}^{(u)}$. $r$ is the maximum $l_2$ norm of the ID representations, i.e.,  $r = \max_{\*z \in \*Z^{(u)}} \|\*z\|_2$.
\end{theorem*}
$~~$

\begin{theorem*}[Simplified version of Theorem~\ref{MainT-2-app}, recap of Theorem~\ref{MainT-3} in the main paper]\label{MainT-3-app}
    Under the same conditions in Theorem~\ref{MainT-2}, if the number of ID and OOD samples $N$ and $M$ is sufficiently large, assume the $k$-th spectral gap of $\tilde{\*A}^{(u)}$ satisfies $\tau>k$ (Assumption~\ref{Ass1} holds), then the main error component $ \epsilon(\mathfrak{p}, \mathfrak{q},\tilde{\*A}^{(u)},\tilde{\*A}_{\rm OI}^{(u)} )  $ in Equation~\ref{eq14_app} satisfies
\begin{equation}
\small
\begin{aligned}
     \epsilon(\mathfrak{p}, \mathfrak{q},&\tilde{\*A}^{(u)},\tilde{\*A}_{\rm OI}^{(u)} )  \geq  \left[1+ \|\tilde{\*A}_{\rm OI}^{(u)} \|_F^2\left( 2N^2-   \|\tilde{\*A}^{(u)}\|_F^2    \right)  \right]  \sum_{i=1}^c\|\mathfrak{q}_i\|_F^2.
\end{aligned}
\end{equation}
\end{theorem*}

\newpage

%% file: proof_app.tex
\subsection{Proof of Theorem~\ref{MainT-2-app}}
\label{Proof_theorem2}
Before proving Theorem~\ref{MainT-2-app}, we first explain the framework introduced in Section~\ref{sec:theory} of the main paper to analyze the difference in the adjacency matrices and the corresponding feature representations between the unlabeled and labeled cases. Specifically, we proposed to analyze the adjacency matrix in the labeled case by perturbation analysis.

\textbf{Matrix perturbation in the labeled case.} Recall that we define in Definition~\ref{Def5} that the adjacency matrix in the labeled case (l) is the unlabeled one (u) plus connectivity incurred by the ID labels, which can be regarded as the perturbation of the labeling information. Therefore, for the ID adjacency matrix with ID labels, we have the perturbation on $\*A^{(u)}$ as follows:
\begin{equation}
    \*A^{(l)} =  \phi_{u} \*A^{(u)} + \phi_l \Delta \*A,
\end{equation}
where $ \Delta \*A \in \mathbb{R}^{N \times N}$ can be calculated based on the augmentation graph $\mathcal{T}$ according to Definition~\ref{Def5}. Following~\cite{sun2023graph}, we study the perturbation from two aspects: (1) The direction of the perturbation which is given by $ \Delta \*A$, (2) The perturbation magnitude $\phi_l$. We first consider the perturbation direction $\Delta \*A$ and recall that we defined the concrete form in Definition~\ref{Def5}: 

\begin{equation}
    [\Delta \*A]_{\*x\*x^{\prime}}\triangleq \sum_{i \in \mathcal{Y}} \mathbb{E}_{\bar{\*x}_l \sim \mathbb{P}_{l_i}} \mathbb{E}_{\bar{\*x}_l^{\prime} \sim \mathbb{P}_{l_i}} \mathcal{T}\left(\*x \mid \bar{\*x}_l\right) \mathcal{T}\left(\*x^{\prime} \mid \bar{\*x}_l^{\prime}\right), \quad \*x, \*x^\prime \in \mathcal{X}_{\rm id}
    \label{eq17_app}
\end{equation}
where $\mathcal{T}$ is the augmentation graph. In our theory, we consider $\|\mathcal{Y}\| = c$, and then we observe that $ \Delta \*A$ is the sum of $c$ rank-1 matrices, which can be written as:
\begin{equation}
    \Delta \*A = \sum_{i=1}^c\mathfrak{q}_i  \mathfrak{q}_i^\top,
\end{equation}
where $\mathfrak{q} \in \mathbb{R}^{N \times c}$ with each column defined as $(\mathfrak{q}_i)_\*x=\mathbb{E}_{\bar{\*x}_l \sim \mathbb{P}_{l_i}} \mathcal{T}\left(\*x \mid \bar{\*x}_l\right), \*x  \in \mathcal{X}_{\rm id}$ and $\bar{\*x}_l$ is the labeled ID data with label $i$. And following~\cite{sun2023graph}, we define the diagonal matrix $\*D^{(l)}$ as follows:
\begin{equation}
    \*D^{(l)} = \phi_u \cdot \*D^{(u)} + \sum_{i=1}^c \phi_l \cdot  {\rm diag}(\mathfrak{q}_i).
\end{equation}

Without losing the generality, we let $\phi_u \cdot {\rm diag}(\*1_{N}^\top  \*A^{(u)})=\*I_{N\times N}$, which means the ID nodes have equal degree in the unlabeled case. We then have:
\begin{equation}
    \*D^{(l)} \triangleq \operatorname{diag}\left(\mathbf{1}_N^{\top} \*A^{(l)}\right)=\*I_{N\times N}+\phi_l \cdot\sum_{i=1}^c   {\rm diag}(\mathfrak{q}_i).
\end{equation}

\textbf{The perturbation function of representation. }We then consider a more generalized form for the ID adjacency matrix $\*A$:
\begin{equation}
    \*A(\phi_l) =  \phi_{u} \*A^{(u)} + \phi_l \cdot \sum_{i=1}^c\mathfrak{q}_i  \mathfrak{q}_i^\top.
     \label{eq21_app}
\end{equation}

For the perturbation on the OOD-ID adjacency matrix $\tilde{\*A}_{\rm OI}$, we define $\mathfrak{p} \in \mathbb{R}^{M\times c}$ with each column defined as $(\mathfrak{p}_i)_\*x=\mathbb{E}_{\bar{\*x}_l \sim \mathbb{P}_{l_i}} \mathcal{T}\left(\*x \mid \bar{\*x}_l\right),  \*x  \sim \mathcal{X}_{\rm ood}$ and $\bar{\*x}_l$ is the labeled ID data with label $l_i$. Similar to Equation~\ref{eq17_app}, we can define the concrete perturbation formula for $\tilde{\*A}_{\rm OI}$ as follows:
\begin{equation}
    [\Delta \tilde{\*A}_{\rm OI}]_{\*x\*x^{\prime}}\triangleq \sum_{i \in \mathcal{Y}} \mathbb{E}_{\bar{\*x}_l \sim \mathbb{P}_{l_i}} \mathbb{E}_{\bar{\*x}_l^{\prime} \sim \mathbb{P}_{l_i}} \mathcal{T}\left(\*x \mid \bar{\*x}_l\right) \mathcal{T}\left(\*x^{\prime} \mid \bar{\*x}_l^{\prime}\right),  \quad\*x^\prime  \sim \mathcal{X}_{\rm id},  \*x \sim \mathcal{X}_{\rm ood}.
\end{equation}
And we have:
\begin{equation}
    \tilde{\*A}_{\rm OI}(\phi_l) =  \phi_{u} \tilde{\*A}_{\rm OI}^{(u)} + \phi_l \cdot \sum_{i=1}^c \mathfrak{p}_i  \mathfrak{q}_i^\top.
     \label{eq23_app}
\end{equation}
Similarly, for the ID diagonal matrix $\*D$, we have that:
\begin{equation}
     \*D(\phi_l) =\*I_{N\times N}+\phi_l \cdot\sum_{i=1}^c   {\rm diag}(\mathfrak{q}_i).
     \label{eq24_app}
\end{equation}
In Equations~\ref{eq21_app} and~\ref{eq23_app}, we treat the adjacency matrix as a function of the “labeling perturbation” magnitude $\phi_l $. It is clear
that $\*A(0) = \phi_u \*A^{(u)}$ and $\tilde{\*A}_{\rm OI}(0) = \phi_u \tilde{\*A}_{\rm OI}^{(u)}$ which are the scaled adjacency matrices in the unlabeled case.  When we let the adjacency matrix be a function of $\phi_l$, the normalized form, and the derived feature representation should also be the function of $\phi_l$. We proceed by defining these terms.

The normalized ID adjacency matrix is given by:
\begin{equation}
    \tilde{\*A}(\phi_l) =  \*D(\phi_l)^{-\frac{1}{2}} {\*A}(\phi_l) \*D(\phi_l)^{-\frac{1}{2}}, 
\end{equation}
In addition, for in-distribution feature representation $\*Z(\phi_l)$,  it is derived from the top-$k$ eigen components of $\tilde{\*A}(\phi_l)$. Specifically, we have:
\begin{equation}
    \*Z(\phi_l) \*Z(\phi_l)^\top = \*D(\phi_l)^{-\frac{1}{2}} \tilde{\*A}_k(\phi_l) \*D(\phi_l)^{-\frac{1}{2}} = \*D(\phi_l)^{-\frac{1}{2}} \sum_{j=1}^k \boldsymbol{\lambda}_j(\phi_l) \*\Phi_{j}(\phi_l)  \*D(\phi_l)^{-\frac{1}{2}} ,
    \label{eq26_app}
\end{equation}
where $\tilde{\*A}_k(\phi_l)$ is the top-$k$ SVD components of $\tilde{\*A}(\phi_l)$ and can be written as $\tilde{\*A}_k(\phi_l) = \sum_{j=1}^k \boldsymbol{\lambda}_j(\phi_l) \*\Phi_j(\phi_l)$. Here $\boldsymbol{\lambda}_j(\phi_l)$ and $ \*\Phi_{j}(\phi_l) = \*v_j(\phi_l)\*v_{j}^{\top}(\phi_l) $
are the $j$-th eigenvalue and eigen projector of matrix $\tilde{\*A}(\phi_l)$. For simplicity, when $\phi_l=0$, we remove the suffix $(0)$ and give the following definitions:
\begin{equation}
    \tilde{\*A}(0) =  \tilde{\*A}^{(u)}, \*Z(0) =  {\*Z}^{(u)}, \*Z_{\rm ood}(0) =  {\*Z}_{\rm ood}^{(u)}, \*Z_{\rm all}(0) =  {\*Z}_{\rm all}^{(u)}, \boldsymbol{\lambda}_j(0) =  \boldsymbol{\lambda}_j^{(u)},  \*v_j(0) =  \*v_j^{(u)}, \*\Phi_{j}(0) =  \*\Phi_j^{(u)}.
\end{equation}
We proceed to provide the five concrete steps to prove Theorem~\ref{MainT-2-app}.

 $~~$

\textbf{Step 1}. Recall that our analysis is based on the upper bound of the linear probing loss, and focuses on the error difference between the labeled and the unlabeled case, which is formulated in Equation~\ref{eq:error_diff}  of the main paper. In the context of perturbation analysis, we can reformulate the error difference $\mathcal{G}$ as a function of $\phi_l$ as follows:
\begin{equation}
    \mathcal{G}(\phi_l) = \overline{R}(  {\*Z}_{\rm all}(0)) - \overline{R}(  {\*Z}_{\rm all}(\phi_l)).
    \label{eq28_app}
\end{equation}

By the definition of the derivative, we can rewrite Equation~\ref{eq28_app}:
\begin{equation}
     \mathcal{G}(\phi_l)   = -\left.\frac{\mathrm{d} \overline{R}\left(\*Z_{\rm all}(\phi_l)\right)}{\mathrm{d} \phi_l}\right|_{\phi_l=0} \cdot  \phi_l 
\end{equation}
With Lemma~\ref{lemma3_app}, we can get the lower bound of $\mathcal{G}(\phi_l)$ as follows:
\begin{equation}
   \mathcal{G}(\phi_l) \geq \frac{2}{3(N+M)} \phi_l \cdot \max \lambda\left(\*Z_{\rm all}(\phi_l)^{\top} \*Z_{\rm all}(\phi_l)\right) \cdot \operatorname{Tr}\left(\left(\*Z_{\rm all}(\phi_l) \*Z_{\rm all}(\phi_l)^{\top}\right)^{\prime}\*y\*y^{\top}\right)\bigg|_{\phi_l=0}
\end{equation}
According to Lemma~\ref{lemma4_app}, $\max \lambda\left(\*Z_{\rm all}(\phi_l)^{\top} \*Z_{\rm all}(\phi_l)\right)\big|_{\phi_l=0} > 0$ is a positive constant. Here $\max \lambda\left(\*Z_{\rm all}(\phi_l)^{\top} \*Z_{\rm all}(\phi_l)\right)$ means the maximum eigenvalue of the matrix $\*Z_{\rm all}(\phi_l)^{\top} \*Z_{\rm all}(\phi_l)$. We then proceed to analyze the key component $\operatorname{Tr}\left(\left(\*Z_{\rm all}(\phi_l) \*Z_{\rm all}(\phi_l)^{\top}\right)^{\prime}\*y\*y^{\top}\right)\big|_{\phi_l=0}$.

Next, according to Lemma 3.1 in~\cite{haochen2021provable}, which implies that multiplying any invertible matrix by the features of the linear probing module does not change the linear probing performance. Therefore, we simplify the mathematical representation of the ID and OOD features in the Equations~\ref{labeled_id_feat} and~\ref{labeled_ood_feat} in the main paper as follows by removing the terms $\*D(\phi_l)$ and $\*D_{\rm ood}(\phi_l)$:
\begin{equation}
  \*Z^{(l)} =  \*V_k^{(l)} [\*\Sigma_k^{(l)}]^{\frac{1}{2}},\quad \*Z_{\rm ood}^{(l)} = \tilde{\*A}_{\rm OI}^{(l)}  \mathbf{V}_k^{(l)}  [\mathbf{\Sigma}_k^{(l)}]^{-\frac{1}{2}},
\end{equation}
which simplifies the notation and will be used in our later analysis.

\textbf{Step 2}. Recall that the input feature for linear probing $\*Z(\phi_l)$ is defined as:
\begin{equation}
    \*Z_{\rm all}(\phi_l) = \begin{pmatrix}
        \*Z(\phi_l) \\ \*Z_{\rm ood}(\phi_l)
    \end{pmatrix} =\begin{pmatrix}
\mathbf{V}_k(\phi_l) \sqrt{\mathbf{\Sigma}_k(\phi_l)}\\
\tilde{\*A}_{\rm OI}(\phi_l) \mathbf{V}_k(\phi_l) \sqrt{\mathbf{\Sigma}_k(\phi_l)}^{-1} \end{pmatrix}
\label{eq32_app}
\end{equation}
Therefore, it is natural to have that:
\begin{equation}
  \tiny
\color{black}
\begin{aligned}
    \operatorname{Tr}( (\*Z_{\rm all}(\phi_l) \*Z_{\rm all}(\phi_l)^\top)^\prime \*y\*y^{\top})\bigg|_{\phi_l=0}    &=\operatorname{Tr}\left( \left(\begin{pmatrix}
\mathbf{V}_k(\phi_l) \sqrt{\mathbf{\Sigma}_k(\phi_l)}\\
\tilde{\*A}_{\rm OI}(\phi_l) \mathbf{V}_k(\phi_l) \sqrt{\mathbf{\Sigma}_k(\phi_l)}^{-1} \end{pmatrix} \cdot  \begin{pmatrix}
 \sqrt{\mathbf{\Sigma}_k(\phi_l)} \mathbf{V}_k(\phi_l)^\top &\sqrt{\mathbf{\Sigma}_k(\phi_l)}^{-1} \mathbf{V}_k(\phi_l)^\top \tilde{\*A}^k_{\rm OI}(\phi_l)^\top   
\end{pmatrix}\right)^\prime\*y\*y^{\top} \right)\bigg|_{\phi_l=0} \\& \geq { \scriptsize \operatorname{Tr}\left( \begin{pmatrix} \mathbf{V}_k(\phi_l) \mathbf{\Sigma}_k(\phi_l) \mathbf{V}_k(\phi_l)^\top &\mathbf{V}_k(\phi_l)\mathbf{V}_k(\phi_l)^\top \tilde{\*A}_{\rm OI}(\phi_l) \\   \tilde{\*A}_{\rm OI}(\phi_l) \mathbf{V}_k(\phi_l)\mathbf{V}_k(\phi_l)^\top  & \tilde{\*A}_{\rm OI}(\phi_l) \mathbf{V}_k(\phi_l)  \mathbf{\Sigma}_k(\phi_l)^{-1} \mathbf{V}_k(\phi_l)^\top\tilde{\*A}_{\rm OI}(\phi_l)^\top  \end{pmatrix}^\prime  \right)\bigg|_{\phi_l=0}} \cdot \operatorname{Tr}(\*y\*y^{\top})\\& = \left(\operatorname{Tr}\left(\left(\mathbf{V}_k(\phi_l) \boldsymbol{\Sigma}_k(\phi_l) \mathbf{V}_k(\phi_l)^{\top}\right)^{\prime}\right)+\operatorname{Tr}\left(\left(\tilde{\*A}_{\rm OI}(\phi_l) \mathbf{V}_k(\phi_l) \boldsymbol{\Sigma}_k(\phi_l)^{-1} \mathbf{V}_k(\phi_l)^{\top} {\tilde{\*A}}_{\rm OI}(\phi_l)^{\top}\right)^{\prime}\right)\bigg|_{\phi_l=0} \right)\cdot \|\*y\|_F^2
\end{aligned}
\end{equation}
In the following steps, we focus on the lower bounds of the two separate terms, i.e., $\operatorname{Tr}\left(\left(\tilde{\*A}_{\rm OI}(\phi_l) \mathbf{V}_k(\phi_l) \boldsymbol{\Sigma}_k(\phi_l)^{-1} \mathbf{V}_k(\phi_l)^{\top} \tilde{\*A}_{\rm OI}(\phi_l)^{\top}\right)^{\prime}\right)\big|_{\phi_l=0}$ and $\operatorname{Tr}\left(\left(\mathbf{V}_k(\phi_l) \boldsymbol{\Sigma}_k(\phi_l) \mathbf{V}_k(\phi_l)^{\top}\right)^{\prime}\right)\big|_{\phi_l=0}$, respectively. 

\textbf{Step 3}. As in Lemma~\ref{lemma5_app}, we provide the lower bound for the first term $\operatorname{Tr}\left(\left(\mathbf{V}_k(\phi_l) \boldsymbol{\Sigma}_k(\phi_l) \mathbf{V}_k(\phi_l)^{\top}\right)^{\prime}\right)\big|_{\phi_l=0}$ as follows:
\begin{equation}
    \operatorname{Tr}\left(\left(\mathbf{V}_k(\phi_l) \boldsymbol{\Sigma}_k(\phi_l) \mathbf{V}_k(\phi_l)^{\top}\right)^{\prime}\right)\bigg|_{\phi_l=0}\geq \sum_{i=1}^c \|\mathfrak{q}_i\|_F^2-
2r^2  \sum_{i=1}^c\|\mathfrak{q}_{i} \|_1,
      \label{eq34_app}
\end{equation}
where $r$ denotes the maximum $\ell_2$ norm of the ID features, i.e.,  $r = \max_{\*z \in \*Z^{(u)}} \|\*z\|_2$, and $\mathfrak{q} \in \mathbb{R}^{N\times c}$ with each column defined as $(\mathfrak{q}_i)_\*x = \mathbb{E}_{\Bar{\*x}_{l} \sim {\mathbb{P}_{l_i}}} \mathcal{T}(\*x | \Bar{\*x}_{l}), \*x \in \mathcal{X}_{\rm id}$.

\textbf{Step 4}. We provide the lower bound for second term $\operatorname{Tr}\left(\left({\tilde{\*A}}_{\rm OI}(\phi_l) \mathbf{V}_k(\phi_l) \boldsymbol{\Sigma}_k(\phi_l)^{-1} \mathbf{V}_k(\phi_l)^{\top} {\tilde{\*A}}_{\rm OI}(\phi_l)^{\top}\right)^{\prime}\right)\big|_{\phi_l=0}$  as in  Lemma~\ref{lemma6_app} as follows:
\begin{equation}
 \begin{aligned}
         &\operatorname{Tr}\left(\left(\tilde{\*A}_{\rm OI}(\phi_l) \mathbf{V}_k(\phi_l) \boldsymbol{\Sigma}_k(\phi_l)^{-1} \mathbf{V}_k(\phi_l)^{\top} \tilde{\*A}_{\rm OI}(\phi_l)^{\top}\right)^{\prime}\right)\big|_{\phi_l=0}\\ &  \geq 2\sum_{i=1}^{c} \operatorname{Tr}\left(\mathfrak{p}_i\mathfrak{q}_i^\top \cdot   \tilde{\*A}_{\rm OI}^{(u)\top}\right)  -\|\tilde{\*A}_{\rm OI}^{(u)} \|_F^2  \|\tilde{\*A}^{(u)}\|_F^2  \cdot(  {\sum_{i=1}^c \|\mathfrak{q}_i\|_F^2 - \frac{ 2r^2(\tau- k)}{\tau-1})}\sum_{i=1}^c\| \mathfrak{q}_{i}\|_1),
         \label{eq35_app}
\end{aligned}
\end{equation}
where $\mathfrak{p}\in \mathbb{R}^{M\times c}$ is defined as $(\mathfrak{p}_i)_\*x = \mathbb{E}_{\Bar{\*x}_{l} \sim {\mathbb{P}_{l_i}}} \mathcal{T}(\*x | \Bar{\*x}_{l}), \*x \in \mathcal{X}_{\rm ood}$.

\textbf{Step 5.} Putting the result in Equations~\ref{eq34_app} and~\ref{eq35_app} together, if we denote $C = \frac{2}{3}\left\|\mathbf{y}\right\|_F^2  \cdot \max \lambda\left(\*Z_{\rm all}(0)^{\top} \*Z_{\rm all}(0)\right) $, we can get:
\begin{equation}
\scriptsize
\begin{aligned}
          &  \mathcal{G}(\phi_l)\geq \frac{C   \phi_l}{N+M}  \cdot \operatorname{Tr}\left(\left(\*Z_{\rm all}(\phi_l) \*Z_{\rm all}(\phi_l)^{\top}\right)^{\prime}\right)\bigg|_{\phi_l=0}
          \\& \geq \frac{C   \phi_l}{N+M} \bigg[2 \sum_{i=1}^{c} \operatorname{Tr}\left(\mathfrak{p}_i\mathfrak{q}_i^\top \cdot   \tilde{\*A}_{\rm OI}^{(u)\top}\right)  +   \left(1- \|\tilde{\*A}_{\rm OI}^{(u)} \|_F^2 \cdot \|\tilde{\*A}^{(u)}\|_F^2  \right) \sum_{i=1}^c \|\mathfrak{q}_i\|_F^2- 
          r^2 \left( 2-\|\tilde{\*A}_{\rm OI}^{(u)} \|_F^2 \cdot \|\tilde{\*A}^{(u)}\|_F^2  \frac{ 2(\tau-k)}{\tau-1} \right)   \cdot \sum_{i=1}^c\|\mathfrak{q}_{i}\|_1\bigg].
\end{aligned}
\end{equation}
We have completed the proof of Theorem~\ref{MainT-2-app}.

\newpage

\subsection{Proof of Theorem~\ref{MainT-3-app}}
As in Assumption~\ref{Ass1}, we can find a $k$ such that $\tau>k$. Based on this condition, we can get the following
\begin{equation}
\begin{aligned}
       \tau > k &= 1+ k-1 \\& =1 + \frac{\|\tilde{\*A}_{\rm OI}^{(u)}\|_F^2 \|\tilde{\*A}^{(u)}\|_F^2 }{\|\tilde{\*A}_{\rm OI}^{(u)}\|_F^2 \|\tilde{\*A}^{(u)}\|_F^2 -1} (k-1) - O( \frac{k-1 }{\|\tilde{\*A}_{\rm OI}^{(u)}\|_F^2 \|\tilde{\*A}^{(u)}\|_F^2 -1} ).
\end{aligned}
\end{equation}

When we have a sufficient number of ID or OOD data, $\|\tilde{\*A}_{\rm OI}^{(u)}\|_F^2 \|\tilde{\*A}^{(u)}\|_F^2$ will be large such that we can omit the term $O( \frac{k-1 }{\|\tilde{\*A}_{\rm OI}^{(u)}\|_F^2 \|\tilde{\*A}^{(u)}\|_F^2 -1} )$ because $k \leq N$. Therefore, we will have the inequality of
\begin{equation}
      \tau > 1 + \frac{\|\tilde{\*A}_{\rm OI}^{(u)}\|_F^2 \|\tilde{\*A}^{(u)}\|_F^2 }{\|\tilde{\*A}_{\rm OI}^{(u)}\|_F^2 \|\tilde{\*A}^{(u)}\|_F^2 -1} (k-1).
\end{equation}

Simply the above inequality, we can get that
\begin{equation}
    \|\tilde{\*A}_{\rm OI}^{(u)} \|_F^2 \|\tilde{\*A}^{(u)}\|_F^2 > \frac{\tau-1}{\tau-k }.
\end{equation}
Therefore we have
\begin{equation}
   -r^2 \left(2-\left\|\tilde{\mathbf{A}}_{\rm OI }^{(u)}\right\|_F^2\left\|\tilde{\mathbf{A}}^{(u)}\right\|_F^2 \frac{2\left(\tau-k\right)}{\tau-1}\right) \cdot \sum_{i=1}^c\|\mathfrak{q}_{i}\|_1 > 0,
\end{equation}
 Under this condition, the main error term $\epsilon(\mathfrak{p}, \mathfrak{q},\tilde{\*A}^{(u)},\tilde{\*A}_{\rm OI}^{(u)} )$ in Theorem~\ref{MainT-2-app} can be further bounded by:
 \begin{equation}
     \epsilon(\mathfrak{p}, \mathfrak{q},\tilde{\*A}^{(u)},\tilde{\*A}_{\rm OI}^{(u)} ) \geq 2 \sum_{i=1}^{c} \operatorname{Tr}\left(\mathfrak{p}_i\mathfrak{q}_i^\top \cdot   \tilde{\*A}_{\rm OI}^{(u)\top}\right)  +  \left(1- \|\tilde{\*A}_{\rm OI}^{(u)} \|_F^2 \|\tilde{\*A}^{(u)}\|_F^2  \right)  \sum_{i=1}^c \|\mathfrak{q}_i\|_F^2,
 \end{equation}
 Following the same analysis approach in literature (See Section 4.1 of \cite{haochen2021provable}), we decompose the vectors $\mathfrak{p}$ and $\mathfrak{q}$ into numbers in each dimension instead of using the vectorized form for calculation. Specifically, we let $\mathcal{T}(\*x|\bar{\*x})=\kappa, \*x, \bar{\*x}\sim\mathcal{X}_{\rm ood}$ and $\mathcal{T}(\*x|\bar{\*x})=\gamma, \*x, \bar{\*x}\sim\mathcal{X}_{\rm id}$
Based on the adjacency matrix and its relationship with the augmentation graph (Equation~\ref{eq:def_wxx}), we can prove that the following result holds:

 \begin{equation}
      \tilde{\*A}_{\rm OI}^{(u)} =\sum_{i=1}^c\phi_u  \frac{\gamma N}{M+N} \cdot [\mathfrak{p}_i, ..., \mathfrak{p}_i] + O(\kappa) = \sum_{i=1}^c\frac{1}{N^3\gamma}  \cdot [\mathfrak{p}_i, ..., \mathfrak{p}_i]+ O(\kappa),
 \end{equation}
 where  $\phi_u = \frac{M+N}{N^4\gamma^2}$.
Then, we obtain that:
\begin{equation}
 \operatorname{Tr}\left(\mathfrak{p}_i\mathfrak{q}_i^\top \cdot   \tilde{\*A}_{\rm OI}^{(u)\top}\right) = \operatorname{Tr}\left(\mathfrak{q}_i^\top \cdot   \tilde{\*A}_{\rm OI}^{(u)\top}\cdot \mathfrak{p}_i\right) =\sum_{i=1}^c \gamma N^2\|\tilde{\*A}_{\rm OI}^{(u)} \|_F^2 \|\mathfrak{q}_i\|_1 +O(\kappa).
\end{equation}

After that, we can get:
\begin{equation}
\begin{aligned}
    &  \epsilon(\mathfrak{p}, \mathfrak{q},\tilde{\*A}^{(u)},\tilde{\*A}_{\rm OI}^{(u)} ) \geq \left(1- \|\tilde{\*A}_{\rm OI}^{(u)} \|_F^2 \|\tilde{\*A}^{(u)}\|_F^2  \right)  \sum_{i=1}^c\|\mathfrak{q}_i\|_F^2  + 2\gamma N^2\|\tilde{\*A}_{\rm OI}^{(u)} \|_F^2 \sum_{i=1}^c\|\mathfrak{q}_i\|_1 +O(\kappa),
\end{aligned}
\end{equation}

Further simplify the right-hand side of the above inequality, we can get:
\begin{equation}
    \begin{aligned}
       & \left(1- \|\tilde{\*A}_{\rm OI}^{(u)} \|_F^2 \|\tilde{\*A}^{(u)}\|_F^2  \right)  \sum_{i=1}^c\|\mathfrak{q}_i\|_F^2  + 2\gamma N^2\|\tilde{\*A}_{\rm OI}^{(u)} \|_F^2 \sum_{i=1}^c\|\mathfrak{q}_i\|_1 +O(\kappa)\\& = \left(1- \|\tilde{\*A}_{\rm OI}^{(u)} \|_F^2 \|\tilde{\*A}^{(u)}\|_F^2  \right)  \sum_{i=1}^c\|\mathfrak{q}_i\|_F^2 + 2c\gamma N^2\|\tilde{\*A}_{\rm OI}^{(u)} \|_F^2 N \gamma+O(\kappa)\\& =\left(1- \|\tilde{\*A}_{\rm OI}^{(u)} \|_F^2 \|\tilde{\*A}^{(u)}\|_F^2  \right)  \sum_{i=1}^c\|\mathfrak{q}_i\|_F^2  + 2N^2 \|\tilde{\*A}_{\rm OI}^{(u)} \|_F^2 \cdot\sum_{i=1}^c \|\mathfrak{q}_i\|_F^2 +O(\kappa)\\ & = \left(1- \|\tilde{\*A}_{\rm OI}^{(u)} \|_F^2 \|\tilde{\*A}^{(u)}\|_F^2 +2N^2\|\tilde{\*A}_{\rm OI}^{(u)} \|_F^2    \right)  \sum_{i=1}^c\|\mathfrak{q}_i\|_F^2 +O(\kappa)
    \end{aligned}
\end{equation}
We have completed the proof of Theorem~\ref{MainT-3-app}. 
\newpage

%% file: lemmas_app.tex
\subsection{Details for Spectral Contrastive Learning}
\label{sec:loss_deri_app}
\begin{lemma*}
\label{lemma1_app} (Recap of Lemma~\ref{lemma1_main} in the main paper) 
We define $\*f_\*x = \sqrt{\zeta_\*x}\*h_\*w(\*x)$ for some function $\*h_\*w$. Recall $\phi_{u},\phi_{l}$ are two weight coefficients given in Definition~\ref{Def5}, then minimizing the loss function $\mathcal{L}(\*F^{(l)}, \*A^{(l)})$ in Equation~\ref{eq:lmf} is equivalent to minimizing the surrogate loss in Equation~\ref{eq:def_SORL_app}.

\begin{align}
\begin{split}
    \mathcal{L}_\text{\rm labeled}(\*h_\*w) &\triangleq - 2\phi_{l} \mathcal{L}_1(\*h_\*w) 
- 2\phi_{u}  \mathcal{L}_2(\*h_\*w)  +\\& \phi_{l}^2 \mathcal{L}_3(\*h_\*w) + 2\phi_{l} \phi_{u} \mathcal{L}_4(\*h_\*w) +  
\phi_{u}^2 \mathcal{L}_5(\*h_\*w),
\label{eq:def_SORL_app}
\end{split}
\end{align} where
\begin{equation}
\begin{aligned}
    & \mathcal{L}_1(\*h_\*w) = \sum_{i \in \mathcal{Y}}\underset{\substack{\Bar{\*x}_{l} \sim \mathbb{P}_{{l_i}}, \Bar{\*x}'_{l} \sim \mathbb{P}_{{l_i}},\\\*x \sim \mathcal{T}(\cdot|\Bar{\*x}_{l}), \*x^{+} \sim \mathcal{T}(\cdot|\Bar{\*x}'_l)}}{\mathbb{E}}\left[\*h_\*w(\*x)^{\top} {\*h_\*w}\left(\*x^{+}\right)\right] , \\&
    \mathcal{L}_2(\*h_\*w) = \underset{\substack{\Bar{\*x}_{u} \sim \mathbb{P}_{\rm id},\\\*x \sim \mathcal{T}(\cdot|\Bar{\*x}_{u}), \*x^{+} \sim \mathcal{T}(\cdot|\Bar{\*x}_u)}}{\mathbb{E}}
\left[\*h_\*w(\*x)^{\top} {\*h_\*w}\left(\*x^{+}\right)\right], \\&
    \mathcal{L}_3(\*h_\*w) = \sum_{i, j\in \mathcal{Y}}
    \underset{\substack{\Bar{\*x}_l \sim \mathbb{P}_{{l_i}}, \Bar{\*x}'_l \sim \mathbb{P}_{{l_{j}}},\\\*x \sim \mathcal{T}(\cdot|\Bar{\*x}_l), \*x^{-} \sim \mathcal{T}(\cdot|\Bar{\*x}'_l)}}{\mathbb{E}}
\left[\left(\*h_\*w(\*x)^{\top} {\*h_\*w}\left(\*x^{-}\right)\right)^2\right], \\&
    \mathcal{L}_4(\*h_\*w) = \sum_{i \in \mathcal{Y}}\underset{\substack{\Bar{\*x}_l \sim \mathbb{P}_{{l_i}}, \Bar{\*x}_u \sim \mathbb{P}_{\rm id},\\\*x \sim \mathcal{T}(\cdot|\Bar{\*x}_l), \*x^{-} \sim \mathcal{T}(\cdot|\Bar{\*x}_u)}}{\mathbb{E}}
\left[\left(\*h_\*w(\*x)^{\top} {\*h_\*w}\left(\*x^{-}\right)\right)^2\right], \\&
    \mathcal{L}_5(\*h_\*w) = \underset{\substack{\Bar{\*x}_u \sim \mathbb{P}_{\rm id}, \Bar{\*x}'_u \sim \mathbb{P}_{\rm id},\\\*x \sim \mathcal{T}(\cdot|\Bar{\*x}_u), \*x^{-} \sim \mathcal{T}(\cdot|\Bar{\*x}'_u)}}{\mathbb{E}}
\left[\left(\*h_\*w(\*x)^{\top} {\*h_\*w}\left(\*x^{-}\right)\right)^2\right].
\label{eq:def_SORL_detail_app}
\end{aligned}
\end{equation}
\end{lemma*}
\begin{proof} We can expand $\mathcal{L}_{\rm labeled}(\*h_\*w)$ and obtain
\begin{equation}
\small
\begin{aligned}
   & \mathcal{L}_{\rm labeled}(\*h_\*w) =\sum_{\*x, \*x^{\prime} \in \mathcal{X}_{\rm id}}\left(\frac{\zeta_{\*x \*x^{\prime}}}{\sqrt{\zeta_\*x \zeta_{\*x^{\prime}}}}-\*f_\*x^{\top} \*f_{\*x^{\prime}}\right)^2 =\\& {\rm const} + 
\sum_{\*x, \*x^{\prime} \in \mathcal{X}_{\rm id}}\left(-2 \zeta_{\*x \*x^{\prime}} \*h_\*w(\*x)^{\top} \*h_\*w\left(\*x^{\prime}\right)+\zeta_\*x \zeta_{\*x^{\prime}}\left(\*h_\*w(\*x)^{\top}\*h_\*w\left(\*x^{\prime}\right)\right)^2\right),
\end{aligned}
\end{equation}

where $\*f_\*x = \sqrt{\zeta_\*x}\*h_\*w(\*x)$ is a re-scaled version of $\*h_\*w(\*x)$.
At a high level, we follow the proof in ~\cite{haochen2021provable}, while the specific form of loss varies with the different definitions of positive/negative pairs. The form of $\mathcal{L}_{\rm labeled}(\*h_\*w)$ is derived from plugging $\zeta_{\*x\*x'}$  and $\zeta_\*x$. 

Recall that $\zeta_{\*x\*x'}$ is defined by
\begin{align*}
\zeta_{\*x\*x'} &= \phi_{l} \sum_{i \in \mathcal{Y}}\mathbb{E}_{\bar{\*x}_{l} \sim {\mathbb{P}_{l_i}}} \mathbb{E}_{\bar{\*x}'_{l} \sim {\mathbb{P}_{l_i}}} \mathcal{T}(\*x | \bar{\*x}_{l}) \mathcal{T}\left(\*x' | \bar{\*x}'_{l}\right)+ \phi_{u} \mathbb{E}_{\bar{\*x}_{u} \sim {\mathbb{P}_{\rm id}}} \mathcal{T}(\*x| \bar{\*x}_{u}) \mathcal{T}\left(\*x'| \bar{\*x}_{u}\right) ,
\end{align*}
and $\zeta_{\*x}$ is given by 
\begin{align*}
\zeta_{\*x } &= \sum_{\*x^{\prime}}\zeta_{\*x\*x'} \\ &=\phi_{l} \sum_{i \in \mathcal{Y}}\mathbb{E}_{\bar{\*x}_{l} \sim {\mathbb{P}_{l_i}}} \mathbb{E}_{\bar{\*x}'_{l} \sim {\mathbb{P}_{l_i}}} \mathcal{T}(\*x | \bar{\*x}_{l}) \sum_{\*x^{\prime}} \mathcal{T}\left(\*x' | \bar{\*x}'_{l}\right)+ \phi_{u} \mathbb{E}_{\bar{\*x}_{u} \sim {\mathbb{P}_{\rm id}}} \mathcal{T}(\*x| \bar{\*x}_{u}) \sum_{x^{\prime}} \mathcal{T}\left(\*x' | \bar{\*x}_{u}\right) \\
&= \phi_{l} \sum_{i \in \mathcal{Y}}\mathbb{E}_{\bar{\*x}_{l} \sim {\mathbb{P}_{l_i}}} \mathcal{T}(\*x | \bar{\*x}_{l}) + \phi_{u} \mathbb{E}_{\bar{\*x}_{u} \sim {\mathbb{P}_{\rm id}}} \mathcal{T}(\*x| \bar{\*x}_{u}). 
\end{align*}

Plugging in $\zeta_{\*x \*x^{\prime}}$ we have, 
\begin{align*}
    &-2 \sum_{\*x, \*x^{\prime} \in \mathcal{X}_{\rm id}}\zeta_{\*x \*x^{\prime}} \*h_\*w(\*x)^{\top} \*h_\*w\left(\*x^{\prime}\right) \\
  = & -2 \sum_{\*x, \*x^{+} \in \mathcal{X}_{\rm id}} \zeta_{\*x \*x^{+}} \*h_\*w(\*x)^{\top} \*h_\*w\left(\*x^{+}\right) 
    \\ = & -2\phi_{l}  \sum_{i \in  \mathcal{Y}}\mathbb{E}_{\bar{\*x}_{l} \sim {\mathbb{P}_{l_i}}} \mathbb{E}_{\bar{\*x}'_{l} \sim {\mathbb{P}_{l_i}}} \sum_{\*x, \*x^{\prime} \in \mathcal{X}_{\rm id}} \mathcal{T}(\*x | \bar{\*x}_{l}) \mathcal{T}\left(\*x' | \bar{\*x}'_{l}\right) \*h_\*w(\*x)^{\top} \*h_\*w\left(\*x^{\prime}\right) \\
    & -2 \phi_{u} \mathbb{E}_{\bar{\*x}_{u} \sim {\mathbb{P}_{\rm id}}} \sum_{\*x, \*x^{\prime}} \mathcal{T}(\*x| \bar{\*x}_{u}) \mathcal{T}\left(\*x'| \bar{\*x}_{u}\right) \*h_\*w(\*x)^{\top} \*h_\*w\left(\*x^{\prime}\right) 
    \\ = & -2\phi_{l}  \sum_{i \in  \mathcal{Y}}\underset{\substack{\bar{\*x}_{l} \sim \mathbb{P}_{{l_i}}, \bar{\*x}'_{l} \sim \mathbb{P}_{{l_i}},\\\*x \sim \mathcal{T}(\cdot|\bar{\*x}_{l}), \*x^{+} \sim \mathcal{T}(\cdot|\bar{\*x}'_l)}}{\mathbb{E}}  \left[\*h_\*w(\*x)^{\top} \*h_\*w\left(\*x^{+}\right)\right] \\
    & - 2\phi_{u} 
    \underset{\substack{\bar{\*x}_{u} \sim \mathbb{P}_{\rm id},\\\*x \sim \mathcal{T}(\cdot|\bar{\*x}_{u}), \*x^{+} \sim \mathcal{T}(\cdot|\bar{\*x}_u)}}{\mathbb{E}}
\left[\*h_\*w(\*x)^{\top} \*h_\*w\left(\*x^{+}\right)\right] \\ =& - 2\phi_{l}  \mathcal{L}_1(\*h_\*w) 
- 2\phi_{u}  \mathcal{L}_2(\*h_\*w).
\end{align*}

Plugging $\zeta_{\*x}$ and $\zeta_{\*x'}$ we have, 
\begin{align*}
    &\sum_{\*x, \*x^{\prime} \in \mathcal{X}_{\rm id}}\zeta_\*x \zeta_{\*x^{\prime}}\left(\*h_\*w(\*x)^{\top}\*h_\*w\left(\*x^{\prime}\right)\right)^2 \\
    = &  \sum_{\*x, \*x^{-} \in \mathcal{X}_{\rm id}}\zeta_\*x \zeta_{\*x^{-}}\left(\*h_\*w(\*x)^{\top}\*h_\*w\left(\*x^{-}\right)\right)^2 \\
    = & \sum_{\*x, \*x^{\prime} \in \mathcal{X}_{\rm id}} \left( \phi_{l} \sum_{i \in  \mathcal{Y}}\mathbb{E}_{\bar{\*x}_{l} \sim {\mathbb{P}_{l_i}}} \mathcal{T}(\*x | \bar{\*x}_{l}) + \phi_{u} \mathbb{E}_{\bar{\*x}_{u} \sim {\mathbb{P}_{\rm id}}} \mathcal{T}(\*x| \bar{\*x}_{u}) \right)  \\&~~~~~~~~~~~\cdot \left(\phi_{l}  \sum_{j \in  \mathcal{Y}}\mathbb{E}_{\bar{\*x}'_{l} \sim {\mathbb{P}_{l_j}}} \mathcal{T}(\*x^{-} | \bar{\*x}'_{l}) + \phi_{u} \mathbb{E}_{\bar{\*x}'_{u} \sim {\mathbb{P}_{\rm id}}} \mathcal{T}(\*x^{-}| \bar{\*x}'_{u}) \right) \left(\*h_\*w(\*x)^{\top}\*h_\*w\left(\*x^{-}\right)\right)^2 \\
    = & \phi_{l} ^ 2 \sum_{\*x, \*x^{-} \in \mathcal{X}_{\rm id}}  \sum_{i \in  \mathcal{Y}}\mathbb{E}_{\bar{\*x}_{l} \sim {\mathbb{P}_{l_i}}} \mathcal{T}(\*x | \bar{\*x}_{l}) \sum_{j \in  \mathcal{Y}}\mathbb{E}_{\bar{\*x}'_{l} \sim {\mathbb{P}_{l_j}}} \mathcal{T}(\*x^{-} | \bar{\*x}'_{l})\left(\*h_\*w(\*x)^{\top}\*h_\*w\left(\*x^{-}\right)\right)^2 \\
    &+ 2\phi_{u} \phi_{l} \sum_{\*x, \*x^{-} \in \mathcal{X}_{\rm id}} \sum_{i \in  \mathcal{Y}}\mathbb{E}_{\bar{\*x}_{l} \sim {\mathbb{P}_{l_i}}} \mathcal{T}(\*x | \bar{\*x}_{l})  \mathbb{E}_{\bar{\*x}_{u} \sim {\mathbb{P}_{\rm id}}} \mathcal{T}(\*x^{-}| \bar{\*x}_{u}) \left(\*h_\*w(\*x)^{\top}\*h_\*w\left(\*x^{-}\right)\right)^2  \\
    &+ \phi_{u}^2 \sum_{\*x, \*x^{-} \in \mathcal{X}_{\rm id}}  \mathbb{E}_{\bar{\*x}_{u} \sim {\mathbb{P}_{\rm id}}} \mathcal{T}(\*x| \bar{\*x}_{u}) \mathbb{E}_{\bar{\*x}'_{u} \sim {\mathbb{P}_{\rm id}}} \mathcal{T}(\*x^{-}| \bar{\*x}'_{u}) \left(\*h_\*w(\*x)^{\top}\*h_\*w\left(\*x^{-}\right)\right)^2 \\
    = & \phi_{l}^2 \sum_{i \in  \mathcal{Y}}\sum_{j \in  \mathcal{Y}}\underset{\substack{\bar{\*x}_l \sim \mathbb{P}_{{l_i}}, \bar{\*x}'_l \sim \mathbb{P}_{{l_j}},\\\*x \sim \mathcal{T}(\cdot|\bar{\*x}_l), \*x^{-} \sim \mathcal{T}(\cdot|\bar{\*x}'_l)}}{\mathbb{E}}
\left[\left(\*h_\*w(\*x)^{\top} \*h_\*w\left(\*x^{-}\right)\right)^2\right] \\
&+ 2\phi_{u}\phi_{l}
    \sum_{i \in  \mathcal{Y}}\underset{\substack{\bar{\*x}_l \sim \mathbb{P}_{{l_i}}, \bar{\*x}_u \sim \mathbb{P}_{\rm id},\\\*x \sim \mathcal{T}(\cdot|\bar{\*x}_l), \*x^{-} \sim \mathcal{T}(\cdot|\bar{\*x}_u)}}{\mathbb{E}}
\left[\left(\*h_\*w(\*x)^{\top} {f}\left(\*x^{-}\right)\right)^2\right] \\ &+ \phi_{u}^2
     \underset{\substack{\bar{\*x}_u \sim \mathbb{P}_{\rm id}, \bar{\*x}'_u \sim \mathbb{P}_{\rm id},\\\*x \sim \mathcal{T}(\cdot|\bar{\*x}_u), \*x^{-} \sim \mathcal{T}(\cdot|\bar{\*x}'_u)}}{\mathbb{E}}
\left[\left(\*h_\*w(\*x)^{\top} {f}\left(\*x^{-}\right)\right)^2\right] 
\\ = & \phi_{l}^2 \mathcal{L}_3(\*h_\*w) + 2\phi_{u}\phi_{l} \mathcal{L}_4(\*h_\*w) + \phi_{u}^2\mathcal{L}_5(\*h_\*w).
\end{align*}
We complete this proof.
\end{proof}

\subsection{Derivation of the Data Representations}
\label{derivation_represe}
For the ID representations, we have the following equation.
\begin{equation}
     \*Z^{(l)} = [\*D^{(l)}]^{-\frac{1}{2}} \*V_k^{(l)} [\*\Sigma_k^{(l)}]^{\frac{1}{2}}.
\end{equation}
It is easy to see that the following equation holds because of the  Eckart–Young–Mirsky theorem~\cite{eckart1936approximation} as we use the top-$k$ SVD components of the matrix $\*A^{(l)}$ to approximate it (low-rank approximation).

For OOD representations, we have the following equation.
\begin{equation}
   \*Z_{\rm ood}^{(l)} = [\*D_{\rm ood}^{(l)}]^{-\frac{1}{2}} \tilde{\*A}_{\rm OI}^{(l)}  \mathbf{V}_k^{(l)}  [\mathbf{\Sigma}_k^{(l)}]^{-\frac{1}{2}}.
\end{equation}
The above result can be obtained by the condition that the optimization problem in Equation~\ref{aoi_optimization} is solved perfectly, i.e., $\tilde{\*A}_{\rm OI}^{(l)} = \*F_{\rm ood}^{(l)} \*F^{(l)\top}$. Note that it is easy to check that $ \*F^{(l)} = \*V_k^{(l)} [\*\Sigma_k^{(l)}]^{\frac{1}{2}}$, if we multiply $ \mathbf{V}_k^{(l)}  [\mathbf{\Sigma}_k^{(l)}]^{-\frac{1}{2}}$ to both sides of the equation $\tilde{\*A}_{\rm OI}^{(l)} = \*F_{\rm ood}^{(l)} \*F^{(l)\top}$, we can get:
\begin{equation}
    \tilde{\*A}_{\rm OI}^{(l)} \mathbf{V}_k^{(l)}  [\mathbf{\Sigma}_k^{(l)}]^{-\frac{1}{2}} = \*F_{\rm ood}^{(l)} [\*\Sigma_k^{(l)}]^{\frac{1}{2}} \*V_k^{(l)\top}  \mathbf{V}_k^{(l)}  [\mathbf{\Sigma}_k^{(l)}]^{-\frac{1}{2}} = \*F_{\rm ood}^{(l)}.
\end{equation}
Note that $\*F_{\rm ood}^{(l)}  =[\*D_{\rm ood}^{(l)}]^{\frac{1}{2}} \*Z_{\rm ood}^{(l)}$ and thus we can get $\*Z_{\rm ood}^{(l)} = [\*D_{\rm ood}^{(l)}]^{-\frac{1}{2}} \tilde{\*A}_{\rm OI}^{(l)}  \mathbf{V}_k^{(l)}  [\mathbf{\Sigma}_k^{(l)}]^{-\frac{1}{2}}$. Therefore we have completed the proof.

\textbf{Design Rationale.} Here we explain the design rationale of the optimization problem in Equation~\ref{aoi_optimization} to get the OOD representation $\*F_{\rm ood}^{(l)}$ and $\*Z_{\rm ood}^{(l)}$, which is inspired by the literature in out-of-sample extension. 

The {out-of-sample extension} is a statistical approach that computes the embeddings of new vertices in a graph with the existing in-sample embeddings and the similarity measurements~\cite{bengio2003out,levin2018out}. The goal is to avoid the repeated computational cost of embedding calculations on a large graph when a new vertex emerges. Most of the current works focused on graph Laplacian embedding~\cite{belkin2003laplacian,trosset2008out,belkin06a,quispe2016extreme,jansen2015scalable}, while a few works relied on the adjacency spectral embedding~\cite{sussman2012consistent,dcc78e4e-176f-32e9-8912-dd9d49a3f7a4} for embedding extension. Our framework is similar to the least-squares optimization approach in out-of-sample extension, which derives the out-of-sample embeddings using the adjacency matrix between the in-sample and out-of-sample data (Please refer to Section 3.1 of~\cite{levin2018out} for detailed derivation).

\subsection{Upper Bound of the Linear Probing Loss}
\label{proof_upper_bound_lp_loss}
\begin{lemma*}[Recap of Lemma~\ref{lemma2_main} in the main paper]\label{lemma2_app}
    Denote the $\*y \in \mathbb{R}^{(N+M)\times 2}$ as a matrix contains the one-hot labels for ID and OOD features in $\*Z_{\rm all}$. We have:
    \begin{equation}
        R(\*Z_{\rm all}) \leq  \frac{2}{N+M}\operatorname{Tr}\left(\left(\mathbf{I}-\*Z_{\rm all}\*Z_{\rm all}^{\dagger}\right) \mathbf{y} \mathbf{y}^\top\right).
    \end{equation}
\end{lemma*}
\begin{proof}
    Recall the definition of the linear probing loss is defined as follows:
    \begin{equation}
   R(\*Z_{\rm all}) \triangleq \min _{\boldsymbol{\theta} \in \mathbb{R}^{k \times 2 }} \mathbb{E}_{\*z \in \*Z_{\rm all}} \mathds{1}[y(\*z) \neq \*g_{\boldsymbol{\theta}}(\*z)].
\end{equation}
According to Lemma 5.1 in~\cite{sun2023when}, we can get the upper bound of $R(\*Z_{\rm all})$ as follows:
\begin{equation}
   R(\*Z_{\rm all}) \leq  \frac{2}{N+M} \min _{\boldsymbol{\theta} \in \mathbb{R}^{k \times 2 }} \|\*y - \*Z_{\rm all}\boldsymbol{\theta} \|_F^2. 
\end{equation}
Given the fact that the closed-form solution of the above minimization problem is $\*Z_{\rm all}^{\dagger}\*y$ where $\*Z_{\rm all}^{\dagger}$ denotes the Moore-Penrose inverse of the feature matrix $\*Z_{\rm all}$, we can rewrite the right side of the above inequality as follows:
\begin{equation}
    \min _{\boldsymbol{\theta} \in \mathbb{R}^{k \times 2 }} \|\*y - \*Z_{\rm all}\boldsymbol{\theta} \|_F^2 = \|\*y - \*Z_{\rm all}\*Z_{\rm all}^{\dagger}\*y\|_F^2.
\end{equation}
Simplify it further, we have
\begin{equation}
     \|\*y - \*Z_{\rm all}\*Z_{\rm all}^{\dagger}\*y\|_F^2 = \|(\mathbf{I}-\*Z_{\rm all}\*Z_{\rm all}^{\dagger}) \mathbf{y} \|_F^2 = \operatorname{Tr}\left(\left(\mathbf{I}-\*Z_{\rm all}\*Z_{\rm all}^{\dagger}\right) \mathbf{y} \mathbf{y}^\top\right).
\end{equation}
Therefore we finished the proof.
\end{proof}

\subsection{Necessary Lemmas for Theorem~\ref{MainT-2-app}}

\begin{lemma*}\label{lemma3_app} Following the definitions in Equations~\ref{eq:lp_loss} and~\ref{eq28_app}, the error difference for linear probing between the unlabeled and labeled case (with a perturbation magnitude $ \phi_l$) is upper bounded by the following term:
\begin{equation}
    \mathcal{G}(\phi_l) \geq \frac{2}{3(N+M)} \phi_l \cdot \max \lambda\left(\*Z_{\rm all}(\phi_l)^{\top} \*Z_{\rm all}(\phi_l)\right) \cdot \operatorname{Tr}\left(\left(\*Z_{\rm all}(\phi_l) \*Z_{\rm all}(\phi_l)^{\top}\right)^{\prime}\*y\*y^{\top}\right)\bigg|_{\phi_l=0},
\end{equation}
 where $\max \lambda\left(\*Z_{\rm all}(\phi_l)^{\top} \*Z_{\rm all}(\phi_l)\right)$ means the maximum eigenvalue of the matrix $\*Z_{\rm all}(\phi_l)^{\top} \*Z_{\rm all}(\phi_l)$.
\end{lemma*}
\begin{proof}[Proof of Lemma \ref{lemma3_app}] From the definition of the derivative, we know that $ \mathcal{G}(\phi_l)=-\left.\frac{\mathrm{d} \overline{R}(\*Z_{\rm all}(\phi_l))}{\mathrm{d} \phi_l}\right|_{\phi_l=0} \cdot  \phi_l$, we then investigate the key component $-\left.\frac{\mathrm{d} \overline{R}(\*Z_{\rm all}(\phi_l))}{\mathrm{d} \phi_l}\right|_{\phi_l=0}$ as follows. Recall that we have 
\begin{equation}
\small
\begin{aligned}
    &-\left.\frac{\mathrm{d} \overline{R}(\*Z_{\rm all}(\phi_l))}{\mathrm{d} \phi_l}\right|_{\phi_l=0}=   {\rm Tr} \left.\left(  \frac{\rm d\*Z_{\rm all}(\phi_l)\*Z_{\rm all}(\phi_l)^{\dagger}  }{\rm d \phi_l} \*y\*y^{\top}\right)  \right|_{\phi_l=0} \cdot \frac{2\phi_l}{N+M}   \\& = \frac{2\phi_l}{N+M} \cdot    \rm Tr\bigg(\bigg[ \*Z_{\rm all}'(\phi_l) \*Z_{\rm all}(\phi_l)^{\dagger}     +  \*Z_{\rm all}(\phi_l)\bigg(-\left( \*Z_{\rm all}(\phi_l)^\top  \*Z_{\rm all}(\phi_l)\right)^{-1}\bigg( \*Z_{\rm all}'(\phi_l)^\top \*Z_{\rm all}(\phi_l)\\&  \quad \quad \quad \quad\quad \quad +\*Z_{\rm all}(\phi_l)^\top \*Z_{\rm all}'(\phi_l)\bigg)\left(\*Z_{\rm all}(\phi_l)^\top \*Z_{\rm all}(\phi_l)\right)^{-1} \*Z_{\rm all}(\phi_l)^\top +\left(\*Z_{\rm all}(\phi_l)^\top \*Z_{\rm all}(\phi_l)\right)^{-1} \*Z_{\rm all}'(\phi_l)^\top\bigg) \bigg] \*y\*y^{\top}\bigg)\bigg|_{\phi_l=0} \\& 
    =   \frac{2\phi_l}{N+M}\rm Tr\bigg( \bigg[ \*Z_{\rm all}'(\phi_l) \*Z_{\rm all}(\phi_l)^{\dagger}-\left(\*Z_{\rm all}(\phi_l)^{\dagger}\right)^\top \*Z_{\rm all}'(\phi_l)^\top \*Z_{\rm all}(\phi_l) \*Z_{\rm all}(\phi_l)^{\dagger}\\& \quad \quad \quad \quad\quad \quad-\left(\*Z_{\rm all}(\phi_l)^{\dagger}\right)^\top \*Z_{\rm all}(\phi_l)^\top \*Z_{\rm all}'(\phi_l) \*Z_{\rm all}(\phi_l)^{\dagger}+\left(\*Z_{\rm all}(\phi_l)^{\dagger}\right)^\top \*Z_{\rm all}'(\phi_l)^\top \bigg]\*y\*y^{\top} \bigg)\bigg|_{\phi_l=0}\\& 
    = \frac{2\phi_l}{N+M}\rm Tr\bigg(
\bigg[\*Z_{\rm all}'(\phi_l) \*Z_{\rm all}(\phi_l)^{\dagger}-\left(\*Z_{\rm all}(\phi_l)^{\dagger}\right)^\top\*Z_{\rm all}'(\phi_l)^\top \*Z_{\rm all}(\phi_l)\*Z_{\rm all}(\phi_l)^{\dagger}\\& \quad \quad \quad \quad\quad \quad-(\*Z_{\rm all}(\phi_l)\*Z_{\rm all}(\phi_l)^{\dagger})^\top \*Z_{\rm all}'(\phi_l) \*Z_{\rm all}(\phi_l)^{\dagger}+\left(\*Z_{\rm all}(\phi_l)^{\dagger}\right)^\top \*Z_{\rm all}'(\phi_l)^\top \bigg] \*y\*y^{\top}
 \bigg) \bigg|_{\phi_l=0} \\& = \frac{2\phi_l}{N+M}
 \rm Tr\left( \bigg[(\*I-\*Z_{\rm all}(\phi_l)\*Z_{\rm all}(\phi_l)^{\dagger})^\top \*Z_{\rm all}'(\phi_l) \*Z_{\rm all}(\phi_l)^{\dagger}+\left(\*Z_{\rm all}(\phi_l)^{\dagger}\right)^\top \*Z_{\rm all}'(\phi_l)^\top(\*I-\*Z_{\rm all}(\phi_l)\*Z_{\rm all}(\phi_l)^{\dagger})\bigg] \*y\*y^{\top} \right) \bigg|_{\phi_l=0}   \\& 
 = \frac{2\phi_l}{N+M} \rm Tr\left( \bigg[(\*I-\*Z_{\rm all}(\phi_l)\*Z_{\rm all}(\phi_l)^{\dagger}) \*Z_{\rm all}'(\phi_l) \*Z_{\rm all}(\phi_l)^{\dagger}+\left(\*Z_{\rm all}(\phi_l)^{\dagger}\right)^\top \*Z_{\rm all}'(\phi_l)^\top(\*I-\*Z_{\rm all}(\phi_l)\*Z_{\rm all}(\phi_l)^{\dagger})^\top \bigg]   \*y\*y^{\top}\right)  \bigg|_{\phi_l=0}  \\& 
=  \frac{2\phi_l}{N+M} \rm Tr\left( \bigg[(\*I-\*Z_{\rm all}(\phi_l)\*Z_{\rm all}(\phi_l)^{\dagger}) \*Z_{\rm all}'(\phi_l) \*Z_{\rm all}(\phi_l)^{\dagger}+\left[(\*I-\*Z_{\rm all}(\phi_l)\*Z_{\rm all}(\phi_l)^{\dagger})  \*Z_{\rm all}'(\phi_l) \*Z_{\rm all}(\phi_l)^{\dagger}\right]^\top\bigg] \*y\*y^{\top} \right)  \bigg|_{\phi_l=0}   
 \\& 
  = \frac{2\phi_l}{N+M} \rm Tr\left(  (\*I-\*Z_{\rm all}(\phi_l)\*Z_{\rm all}(\phi_l)^{\dagger}) \cdot \left( \*Z_{\rm all}'(\phi_l) \*Z_{\rm all}(\phi_l)^{\dagger} + (\*Z_{\rm all}'(\phi_l) \*Z_{\rm all}(\phi_l)^{\dagger})^\top \right)  \*y\*y^{\top} \right)  \bigg|_{\phi_l=0}   \\& 
  \geq\frac{2\phi_l}{N+M} \rm Tr\left(  (\frac{2}{3}\*I-\*Z_{\rm all}(\phi_l)\*Z_{\rm all}(\phi_l)^{\dagger}) \cdot \left( \*Z_{\rm all}'(\phi_l) \*Z_{\rm all}(\phi_l)^{\dagger} + (\*Z_{\rm all}'(\phi_l) \*Z_{\rm all}(\phi_l)^{\dagger})^\top \right)   \*y\*y^{\top}\right)  \bigg|_{\phi_l=0}  
\\ & 
  \geq \frac{2\phi_l}{N+M}  \left[\nu(\frac{2}{3}\*I-\*Z_{\rm all}(\phi_l)\*Z_{\rm all}(\phi_l)^{\dagger})+\pi(\frac{2}{3}\*I-\*Z_{\rm all}(\phi_l)\*Z_{\rm all}(\phi_l)^{\dagger})\right] \cdot \\&  \quad \quad \quad \quad\quad \quad \rm Tr\bigg(\bigg( \*Z_{\rm all}'(\phi_l) (\*Z_{\rm all}(\phi_l)^\top\*Z_{\rm all}(\phi_l))^{-1}\*Z_{\rm all}(\phi_l)^\top  + \*Z_{\rm all}(\phi_l) (\*Z_{\rm all}(\phi_l)^\top\*Z_{\rm all}(\phi_l))^{-1}\*Z_{\rm all}'(\phi_l)^\top \bigg)\*y\*y^{\top}\bigg) \bigg|_{\phi_l=0}\\
 & =   \frac{2\phi_l}{N+M}  \frac{1}{3} \cdot \rm Tr\bigg( \bigg((\*Z_{\rm all}(\phi_l)^\top\*Z_{\rm all}(\phi_l))^{-1}\*Z_{\rm all}(\phi_l)^\top  \*Z_{\rm all}'(\phi_l) +  (\*Z_{\rm all}(\phi_l)^\top\*Z_{\rm all}(\phi_l))^{-1} \*Z_{\rm all}'(\phi_l)^\top\*Z_{\rm all}(\phi_l)\bigg) \*y\*y^{\top}\bigg)\bigg|_{\phi_l=0}
 \\& 
 \geq   \frac{2\phi_l}{N+M}\frac{1}{3}  \min \lambda((\*Z_{\rm all}(\phi_l)^\top\*Z_{\rm all}(\phi_l))^{-1}) \cdot   \rm Tr\bigg(\bigg(  \*Z_{\rm all}(\phi_l)^\top \*Z_{\rm all}'(\phi_l) + \*Z_{\rm all}'(\phi_l)^\top\*Z_{\rm all}(\phi_l)\bigg)\*y\*y^{\top}\bigg)\bigg|_{\phi_l=0}
 \\
 &   =  \frac{2\phi_l}{3(N+M)} \max \lambda(
  \*Z_{\rm all}(\phi_l)^\top\*Z_{\rm all}(\phi_l)) \cdot   \rm Tr\bigg(  \bigg((\*Z_{\rm all}(\phi_l)\*Z_{\rm all}(\phi_l)^\top)'  \bigg) \*y\*y^{\top}\bigg)\bigg|_{\phi_l=0}\,
\end{aligned}
\end{equation}
where $\nu(\frac{2}{3}\*I-\*Z_{\rm all}(\phi_l)\*Z_{\rm all}(\phi_l)^{\dagger})$ and $\pi(\frac{2}{3}\*I-\*Z_{\rm all}(\phi_l)\*Z_{\rm all}(\phi_l)^{\dagger})$ denote the
smallest negative eigenvalues and the smallest positive eigenvalues\footnote{The value of the two terms are $\frac{2}{3}$ and $-\frac{1}{3}$ according to \url{https://math.stackexchange.com/questions/188129}} of the matrix $\frac{2}{3}\*I-\*Z_{\rm all}(\phi_l)\*Z_{\rm all}(\phi_l)^{\dagger}$. 

The first seven equations calculate the derivative w.r.t.  the perturbation magnitude $\phi_l$ and use the cyclic property of the trace operator. The first inequality is true because we have 
\begin{equation}
\scriptsize
    \frac{1}{3}\rm Tr((\*Z'_{\rm all}(\phi_l) \*Z_{\rm all}(\phi_l)^{\dagger} + (\*Z'_{\rm all}(\phi_l) \*Z_{\rm all}(\phi_l)^{\dagger})^\top )\*y\*y^{\top})\big|_{\phi_l=0} \geq \frac{1}{3}\max \lambda(
  \*Z_{\rm all}(\phi_l)^\top\*Z_{\rm all}(\phi_l)) \max \lambda(\*y\*y^{\top} ) \cdot   \rm Tr(  (\*Z_{\rm all}(\phi_l)\*Z_{\rm all}(\phi_l)^\top)'  )\big|_{\phi_l=0}.
\end{equation}
 Since we know $\max \lambda(
  \*Z_{\rm all}(\phi_l)^\top\*Z_{\rm all}(\phi_l))\big|_{\phi_l=0} > 0$ (Lemma~\ref{lemma4_app}) and $\max \lambda(\*y\*y^{\top})>0$, and $\rm Tr(  (\*Z_{\rm all}(\phi_l)\*Z_{\rm all}(\phi_l)^\top)'  )\big|_{\phi_l=0}$ is lower bounded and can be large than 0 (Lemmas~\ref{lemma5_app} and~\ref{lemma6_app}).  The second inequality holds because of the main theorem in~\cite{baksalary1992inequality}. The third inequality holds because the inequality for the trace of matrix product of two square matrices~\cite{362841}, i.e., $\rm Tr(\*H \*M) \geq \min \lambda(\*H) \cdot \rm Tr( \*M)$.
\end{proof}
$~~$

\begin{lemma*} Suppose the perturbation magnitude $\phi_l$ is 0, denote the input feature of linear probing as $\*Z_{\rm all}(0)$, then we have that:
 \begin{equation}
        \max \lambda\left(\*Z_{\rm all}(\phi_l)^{\top} \*Z_{\rm all}(\phi_l)\right)\bigg|_{\phi_l=0} >  0,
    \end{equation}
here $\max \lambda\left(\*Z_{\rm all}(\phi_l)^{\top} \*Z_{\rm all}(\phi_l)\right)$ means the maximum eigenvalue of the matrix $\*Z_{\rm all}(\phi_l)^{\top} \*Z_{\rm all}(\phi_l)$. 
\label{lemma4_app}
\end{lemma*}
\begin{proof}[Proof of Lemma \ref{lemma4_app}] Since we know that the column rank of the ID feature matrix $\*Z(0) \in \mathbb{R}^{N\times k}$ is $k$, then when adding $M$ rows of OOD features  $\*Z_{\rm ood}(0)$ to get $\*Z_{\rm all}(0)$ will not increase the column rank. Therefore, the rank of matrix $\*Z_{\rm all}(0)$ is equal to $k$, which means $\*Z_{\rm all}(0)$ has a full column rank. Thus, the matrix $\*Z_{\rm all}(0)^{\top} \*Z_{\rm all}(0)$  is positive definite~\footnote{Proof idea is shown here~\url{https://math.stackexchange.com/questions/2202242}}, which means the eigenvalues of $\*Z_{\rm all}(0)^{\top} \*Z_{\rm all}(0)$  are all greater than 0 and the lemma is proved. 
\end{proof}
$~~$

\begin{lemma*}
\label{lemma5_app} If Assumption~\ref{Ass2} holds, the lower bound for  $  \operatorname{Tr}\left(\left(\mathbf{V}_{k}(\phi_l) \boldsymbol{\Sigma}_{k}(\phi_l) \mathbf{V}_{k}(\phi_l)^{\top}\right)^{\prime}\right)  \bigg|_{\phi_l=0}$ is given as follows:
\begin{equation}
      \operatorname{Tr}\left(\left(\mathbf{V}_k(\phi_l) \boldsymbol{\Sigma}_{k}(\phi_l) \mathbf{V}_{k}(\phi_l)^{\top}\right)^{\prime}\right)  \bigg|_{\phi_l=0}\geq \sum_{i=1}^c \|\mathfrak{q}_i\|_F^2-
2r^2  \sum_{i=1}^c\|\mathfrak{q}_{i} \|_1,
\end{equation}
where $\mathfrak{q} \in \mathbb{R}^{N\times c}$ with each column defined as $(\mathfrak{q}_i)_\*x = \mathbb{E}_{\Bar{\*x}_{l} \sim {\mathbb{P}_{l_i}}} \mathcal{T}(\*x | \Bar{\*x}_{l}), \*x \in \mathcal{X}_{\rm id}$, and $r$ is the maximum $l_2$ norm of the ID representations, i.e.,  $r = \max_{\*z \in \*Z^{(u)}} \|\*z\|_2$.
\end{lemma*}
\begin{proof}[Proof of Lemma \ref{lemma5_app}] 
As $\mathbf{V}_k(\phi_l) \boldsymbol{\Sigma}_{k}(\phi_l) \mathbf{V}_{k}(\phi_l)^{\top}$ denotes the top-$k$ SVD components of the matrix $\tilde{\*A}(\phi_l)$, according to Equation~\ref{eq26_app}, we rewrite it using the eigenvalues and the eigen projectors as follows. 
    \begin{equation}
    \tiny
    \begin{aligned}
\operatorname{Tr}\left(\left(\mathbf{V}_k(\phi_l) \boldsymbol{\Sigma}_{k}(\phi_l) \mathbf{V}_{k}(\phi_l)^{\top}\right)^{\prime}\right)  \bigg|_{\phi_l=0} &=  \sum_{j=1}^{k}  { \rm Tr}\bigg(  [  \*D(\phi_l)^{-\frac{1}{2}} \boldsymbol{\lambda}_j(\phi_l) \*\Phi_j(\phi_l) \*D(\phi_l)^{-\frac{1}{2}}]'  \bigg)\bigg|_{\phi_l=0} \\ &=
          \sum_{j=1}^k \operatorname{Tr}\left(\left[\*D(\phi_l)^{-\frac{1}{2}}\right]^{\prime} \boldsymbol{\lambda}_j^{ (u)} \*\Phi_j^{ (u)}+\boldsymbol{\lambda}_j^{ (u)} \*\Phi_j^{ (u)}\left[\*D(\phi_l)^{-\frac{1}{2}}\right]^{\prime}+\left[\boldsymbol{\lambda}_j(\phi_l)\right]^{\prime} \*\Phi_j^{ (u)}+\boldsymbol{\lambda}_j^{ (u)}\left[\*\Phi_j(\phi_l)\right]^{\prime}\right)\bigg|_{\phi_l=0},
    \end{aligned}
    \label{eq68_app}
    \end{equation}
    where $\boldsymbol{\lambda}_j^{(u)}$ and $\*\Phi_j^{(u)}$ are the $j$-th eigenvalue and eigen projector of matrix $\tilde{\*A}(\phi_l)$ when $\phi_l=0$. Moreover, assume $\tilde{\*A}^{(u)}$ to be the normalized adjacency matrix when $\phi_l=0$, recall $\*D(\phi_l) = \*I_{N \times N} + \phi_l \*D_l$ where $\*D_l = \sum_{i=1}^c   {\rm diag}(\mathfrak{q}_i)$ (Equation~\ref{eq24_app}), we have the following calculation for the derivatives:
\begin{equation}
    \left[\*D(\phi_l)^{-\frac{1}{2}}\right]^{\prime} = -\frac{1}{2} \*D_l.
\end{equation}
\begin{equation}
    \begin{aligned} {\left.[\tilde{\*A}(\phi_l)]^{\prime}\right|_{\phi_l=0} } & =\left.\left[\*D(\phi_l)^{-\frac{1}{2}} \*A(\phi_l) \*D(\phi_l)^{-\frac{1}{2}}\right]^{\prime}\right|_{\phi_l=0} \\ & =\left[\*D(\phi_l)^{-\frac{1}{2}}\right]^{\prime} \tilde{\*A}^{(u)}+[\*A(\phi_l)]^{\prime}+\tilde{\*A}^{(u)}\left[\*D(\phi_l)^{-\frac{1}{2}}\right]^{\prime} \\ & =-\frac{1}{2} \*D_l \tilde{\*A}^{(u)}+\sum_{i=1}^c \mathfrak{q_i q_i}^{\top}-\frac{1}{2} \tilde{\*A}^{(u)} \*D_l\end{aligned}
    \label{eq70_pp}
\end{equation}
To get the derivative of the eigenvalues, according to Equation 3 in~\cite{greenbaum2020first},
\begin{equation}
    \begin{aligned} {\left.\left[\boldsymbol{\lambda}_j(\phi_l)\right]^{\prime}\right|_{\phi_l=0} } & =\operatorname{Tr}\left(\*\Phi_j^{ (u)}[\tilde{\*A}(\phi_l)]^{\prime}\right) \bigg|_{\phi_l=0}\\ & =\operatorname{Tr}\left(\*\Phi_j^{ (u)}\left(-\frac{1}{2} \*D_l \tilde{\*A}^{(u)}+\sum_{i=1}^c \mathfrak{q_i q_i}^{\top}-\frac{1}{2} \tilde{\*A}^{(u)} \*D_l\right)\right) \\ & =\operatorname{Tr}\left(-\frac{\boldsymbol{\lambda}_j^{ (u)}}{2} \*D_l \*\Phi_j^{ (u)}+\*\Phi_j^{ (u)} \sum_{i=1}^c \mathfrak{q_i q_i}^{\top}-\frac{\boldsymbol{\lambda}_j^{ (u)}}{2} \*\Phi_j^{ (u)} \*D_l\right) \\ & =\operatorname{Tr}\left(\*\Phi_j^{ (u)}\left(\sum_{i=1}^c \mathfrak{q_i q_i}^{\top}-\boldsymbol{\lambda}_j^{ (u)} \*D_l\right)\right) .\end{aligned}
    \label{eq71_pp}
\end{equation}
To get the derivative of the eigenvectors, according to Equation 10 in~\cite{greenbaum2020first},
\begin{equation}
\small
    \begin{aligned} {\left.\left[\*\Phi_j(\phi_l)\right]^{\prime}\right|_{\phi_l=0} } & =\left(\boldsymbol{\lambda}_j^{ (u)} \*I_N-\tilde{\*A}^{(u)}\right)^{\dagger}[\tilde{\*A}(\phi_l)]^{\prime} \*\Phi_j^{ (u)}+\*\Phi_j^{ (u)}[\tilde{\*A}(\phi_l)]^{\prime}\left(\boldsymbol{\lambda}_j^{ (u)} \*I_N-\tilde{\*A}^{(u)}\right)^{\dagger} \\
    & =\sum_{i \neq j}^N \frac{1}{\boldsymbol{\lambda}_j^{ (u)}-\boldsymbol{\lambda}_i^{ (u)}}\left(\*\Phi_i^{ (u)}[\tilde{\*A}(\phi_l)]^{\prime} \*\Phi_j^{ (u)}+\*\Phi_j^{ (u)}[\tilde{\*A}(\phi_l)]^{\prime} \*\Phi_i^{ (u)}\right) \\ 
    & =\sum_{i \neq j}^N \frac{1}{\boldsymbol{\lambda}_j^{ (u)}-\boldsymbol{\lambda}_i^{ (u)}}\left(\*\Phi_i^{ (u)}\left(-\frac{1}{2} \*D_l \tilde{\*A}^{(u)}+\sum_{i=1}^c \mathfrak{q_i q_i}^{\top}-\frac{1}{2} \tilde{\*A}^{(u)} \*D_l\right) \*\Phi_j^{ (u)}+\*\Phi_j^{ (u)}(\ldots) \*\Phi_i^{ (u)}\right) \\ 
    & =\sum_{i \neq j}^N \frac{1}{\boldsymbol{\lambda}_j^{ (u)}-\boldsymbol{\lambda}_i^{ (u)}}\left(\*\Phi_i^{ (u)}\left(\sum_{i=1}^c \mathfrak{q_i q_i}^{\top}-\frac{\boldsymbol{\lambda}_j^{ (u)}+\boldsymbol{\lambda}_i^{ (u)}}{2} \*D_l\right) \*\Phi_j^{ (u)}+\*\Phi_j^{ (u)}\left(\sum_{i=1}^c \mathfrak{q_i q_i}^{\top}-\frac{\boldsymbol{\lambda}_j^{ (u)}+\boldsymbol{\boldsymbol{\lambda}}_i^{ (u)}}{2} \*D_l\right) \*\Phi_i^{ (u)}\right) .\end{aligned}
\end{equation}
Put them together in Equation~\ref{eq68_app}, we can get 
\begin{equation}
\begin{aligned}
       &    \operatorname{Tr}\left(\left(\mathbf{V}_k(\phi_l) \boldsymbol{\Sigma}_{k}(\phi_l) \mathbf{V}_{k}(\phi_l)^{\top}\right)^{\prime}\right)  \bigg|_{\phi_l=0}  \\& =   \sum_{j=1}^k \operatorname{Tr}\left( -\frac{1}{2}\*D_l  \boldsymbol{\lambda}_j^{ (u)} \*\Phi_j^{ (u)}  -\frac{1}{2}\boldsymbol{\lambda}_j^{ (u)} \*\Phi_j^{ (u)}\*D_l + \operatorname{Tr}\left(\*\Phi_j^{ (u)}\left(\sum_{i=1}^c \mathfrak{q_i q_i}^{\top}-\boldsymbol{\lambda}_j^{ (u)} \*D_l\right)\right) \*\Phi_j^{ (u)}  + \boldsymbol{\lambda}_j^{ (u)}\left[\*\Phi_j(\phi_l)\right]^{\prime}\right)\bigg|_{\phi_l=0}
               \\& =
              \left[ \sum_{j=1}^k -\frac{\boldsymbol{\lambda}_j^{ (u)}}{2}\operatorname{Tr}\left( \*D_l  \*\Phi_j^{ (u)}  +  \*\Phi_j^{ (u)}\*D_l \right) + \operatorname{Tr}\left(\*\Phi_j^{ (u)}\left(\sum_{i=1}^c \mathfrak{q_i q_i}^{\top}-\boldsymbol{\lambda}_j^{ (u)} \*D_l\right)\right) \operatorname{Tr}\left(\*\Phi_j^{ (u)} \right) +\operatorname{Tr}\left( \boldsymbol{\lambda}_j^{ (u)}\left[\*\Phi_j(\phi_l)\right]^{\prime}\right)\right]\bigg|_{\phi_l=0}\\
\end{aligned}
\end{equation}
For $\operatorname{Tr}\left(\*\Phi_j^{ (u)}\left(\sum_{i=1}^c \mathfrak{q_i q_i}^{\top}-\boldsymbol{\lambda}_j^{ (u)} \*D_l\right)\right) \operatorname{Tr}\left(\*\Phi_j^{ (u)} \right) $, we have the following:
\begin{equation}
   \sum_{j=1}^k \operatorname{Tr}\left(\*\Phi_j^{ (u)}\left(\sum_{i=1}^c \mathfrak{q_i q_i}^{\top}-\boldsymbol{\lambda}_j^{ (u)} \*D_l\right)\right) \operatorname{Tr}\left(\*\Phi_j^{ (u)} \right) = \sum_{j=1}^k\operatorname{Tr}\left(\left(\sum_{i=1}^c \mathfrak{q_i q_i}^{\top}-\boldsymbol{\lambda}_j^{ (u)} \*D_l\right)\*\Phi_j^{ (u)}  \*\Phi_j^{ (u)}\right)
    \label{eq74_app}
\end{equation}

For $\operatorname{Tr}\left( \boldsymbol{\lambda}_j^{ (u)}\left[\*\Phi_j(\phi_l)\right]^{\prime}\right)$, we have the following:
\begin{equation}
\small
 \begin{aligned}
 & \sum_{j=1}^k\operatorname{Tr}\left( \boldsymbol{\boldsymbol{\lambda}}_j^{ (u)}\left[\*\Phi_j(\phi_l)\right]^{\prime}\right)  \\
     & =\sum_{j=1}^k \operatorname{Tr}\left(\sum_{i \neq j}^N \frac{\boldsymbol{\lambda}_j^{ (u)}}{\boldsymbol{\lambda}_j^{ (u)}-\boldsymbol{\lambda}_i^{ (u)}}\left( \*\Phi_i^{ (u)}\left(\sum_{i=1}^c \mathfrak{q_i q_i}^{\top}-\frac{\boldsymbol{\lambda}_j^{ (u)}+\boldsymbol{\lambda}_i^{ (u)}}{2} \*D_l\right) \*\Phi_j^{ (u)}+ \*\Phi_j^{ (u)}\left(\sum_{i=1}^c \mathfrak{q_i q_i}^{\top}-\frac{\boldsymbol{\lambda}_j^{ (u)}+\boldsymbol{\lambda}_i^{ (u)}}{2} \*D_l\right) \*\Phi_i^{ (u)}\right)\right) \\ 
     &= \sum_{j=1}^k \operatorname{Tr}\left(\sum_{i \neq j}^N \frac{\boldsymbol{\lambda}_j^{ (u)}}{\boldsymbol{\lambda}_j^{ (u)}-\boldsymbol{\lambda}_i^{ (u)}}\left(\left(\*\Phi_j^{ (u)}  \*\Phi_i^{ (u)}+\*\Phi_i^{ (u)}  \*\Phi_j^{ (u)}\right)\left(\sum_{i=1}^c \mathfrak{q_i q_i}^{\top}-\frac{\boldsymbol{\lambda}_j^{ (u)}+\boldsymbol{\lambda}_i^{ (u)}}{2} \*D_l\right)\right)\right) \\  
     & =\sum_{j=1}^k \operatorname{Tr}\left(\sum_{i \neq j, i \leq k} \frac{\boldsymbol{\lambda}_j^{ (u)}}{\boldsymbol{\lambda}_j^{ (u)}-\boldsymbol{\lambda}_i^{ (u)}}\left(\left(\*\Phi_j^{ (u)}  \*\Phi_i^{ (u)}+\*\Phi_i^{ (u)}  \*\Phi_j^{ (u)}\right)\left(\sum_{i=1}^c \mathfrak{q_i q_i}^{\top}-\frac{\boldsymbol{\lambda}_j^{ (u)}+\boldsymbol{\lambda}_i^{ (u)}}{2} \*D_l\right)\right)\right) \\ & +\sum_{j=1}^k \operatorname{Tr}\left(\sum_{i=k+1}^N \frac{\boldsymbol{\lambda}_j^{ (u)}}{\boldsymbol{\lambda}_j^{ (u)}-\boldsymbol{\lambda}_i^{ (u)}}\left(\left(\*\Phi_j^{ (u)}  \Phi_i^{(u)}+\*\Phi_i^{ (u)}  \*\Phi_j^{ (u)}\right)\left(\sum_{i=1}^c \mathfrak{q_i q_i}^{\top}-\frac{\boldsymbol{\lambda}_j^{ (u)}+\boldsymbol{\lambda}_i^{ (u)}}{2} \*D_l\right)\right)\right) \\ 
     = & \sum_{j=1}^k \operatorname{Tr}\left(\sum_{i<j}\left(\frac{\boldsymbol{\lambda}_j^{ (u)}}{\boldsymbol{\lambda}_j^{ (u)}-\boldsymbol{\lambda}_i^{ (u)}}+\frac{\boldsymbol{\lambda}_i^{ (u)}}{\boldsymbol{\lambda}_i^{ (u)}-\boldsymbol{\lambda}_j^{ (u)}}\right)\left(\left(\*\Phi_j^{ (u)}  \*\Phi_i^{ (u)}+\*\Phi_i^{ (u)}  \*\Phi_j^{ (u)}\right)\left(\sum_{i=1}^c \mathfrak{q_i q_i}^{\top}-\frac{\boldsymbol{\lambda}_j^{ (u)}+\boldsymbol{\lambda}_i^{ (u)}}{2} \*D_l\right)\right)\right) \\ & +\sum_{j=1}^k \operatorname{Tr}\left(\sum_{i=k+1}^N \frac{\boldsymbol{\lambda}_j^{ (u)}}{\boldsymbol{\lambda}_j^{ (u)}-\boldsymbol{\lambda}_i^{ (u)}}\left(\left(\*\Phi_j^{ (u)}  \*\Phi_i^{ (u)}+\*\Phi_i^{ (u)}  \*\Phi_j^{ (u)}\right)\left(\sum_{i=1}^c \mathfrak{q_i q_i}^{\top}-\frac{\boldsymbol{\lambda}_j^{ (u)}+\boldsymbol{\lambda}_i^{ (u)}}{2} \*D_l\right)\right)\right)
      \\
    = & \sum_{j=1}^k \operatorname{Tr}\left(\sum_{i<j}\left(\left(\*\Phi_j^{ (u)}  \*\Phi_i^{ (u)}+\*\Phi_i^{ (u)}  \*\Phi_j^{ (u)}\right)\left(\sum_{i=1}^c \mathfrak{q_i q_i}^{\top}-\frac{\boldsymbol{\lambda}_j^{ (u)}+\boldsymbol{\lambda}_i^{ (u)}}{2} \*D_l\right)\right)\right) \\
    & +\sum_{j=1}^k \operatorname{Tr}\left(\sum_{i=k+1}^N \frac{\boldsymbol{\lambda}_j^{ (u)}}{\boldsymbol{\lambda}_j^{ (u)}-\boldsymbol{\lambda}_i^{ (u)}}\left(\left(\*\Phi_j^{ (u)}  \*\Phi_i^{ (u)}+\*\Phi_i^{ (u)}  \*\Phi_j^{ (u)}\right)\left(\sum_{i=1}^c \mathfrak{q_i q_i}^{\top}-\frac{\boldsymbol{\lambda}_j^{ (u)}+\boldsymbol{\lambda}_i^{ (u)}}{2} \*D_l\right)\right)\right) \\ 
    = & \sum_{j=1}^k \operatorname{Tr}\left(\sum_{i \neq j, i \leq k} \frac{1}{2}\left(\left(\*\Phi_j^{ (u)}  \*\Phi_i^{ (u)}+\*\Phi_i^{ (u)}  \*\Phi_j^{ (u)}\right)\left(\sum_{i=1}^c \mathfrak{q_i q_i}^{\top}-\frac{\boldsymbol{\lambda}_j^{ (u)}+\boldsymbol{\lambda}_i^{ (u)}}{2} \*D_l\right)\right)\right) \\ & +\sum_{j=1}^k \operatorname{Tr}\left(\sum_{i=k+1}^N \frac{\boldsymbol{\lambda}_j^{ (u)}}{\boldsymbol{\lambda}_j^{ (u)}-\boldsymbol{\lambda}_i^{ (u)}}\left(\left(\*\Phi_j^{ (u)}  \*\Phi_i^{ (u)}+\*\Phi_i^{ (u)}  \*\Phi_j^{ (u)}\right)\left(\sum_{i=1}^c \mathfrak{q_i q_i}^{\top}-\frac{\boldsymbol{\lambda}_j^{ (u)}+\boldsymbol{\lambda}_i^{ (u)}}{2} \*D_l\right)\right)\right) 
     \end{aligned}
     \label{eq75_app}
\end{equation}

If we put the result in Equations~\ref{eq74_app} and~\ref{eq75_app} together, we can get:
\begin{equation}
\begin{aligned}
       & \sum_{j=1}^k\operatorname{Tr}\left(\left(\sum_{i=1}^c \mathfrak{q_i q_i}^{\top}-\boldsymbol{\lambda}_j^{ (u)} \*D_l\right) \*\Phi_j^{ (u)}\*\Phi_j^{ (u)}\right) + \sum_{j=1}^k\operatorname{Tr}\left( \boldsymbol{\lambda}_j^{ (u)}\left[\*\Phi_j(\phi_l)\right]^{\prime}\right)  \\
       & = \sum_{j=1}^k\operatorname{Tr}\left( \frac{\*\Phi_j^{ (u)}\*\Phi_j^{ (u)}+\*\Phi_j^{ (u)}\*\Phi_j^{ (u)}}{2}\left(\sum_{i=1}^c \mathfrak{q_i q_i}^{\top}-\frac{\boldsymbol{\lambda}_j^{ (u)}+\boldsymbol{\lambda}_j^{ (u)}}{2} \*D_l\right) \right)  + \sum_{j=1}^k\operatorname{Tr}\left( \boldsymbol{\lambda}_j^{ (u)}\left[\*\Phi_j(\phi_l)\right]^{\prime}\right) \\
       = & \sum_{j=1}^k \operatorname{Tr}\left(\sum_{ i \leq k} \frac{1}{2}\left(\left(\*\Phi_j^{ (u)}  \*\Phi_i^{ (u)}+\*\Phi_i^{ (u)}  \*\Phi_j^{ (u)}\right)\left(\sum_{i=1}^c \mathfrak{q_i q_i}^{\top}-\frac{\boldsymbol{\lambda}_j^{ (u)}+\boldsymbol{\lambda}_i^{ (u)}}{2} \*D_l\right)\right)\right) \\ & +\sum_{j=1}^k \operatorname{Tr}\left( \sum_{i=k+1}^N \frac{\boldsymbol{\lambda}_j^{ (u)}}{\boldsymbol{\lambda}_j^{ (u)}-\boldsymbol{\lambda}_i^{ (u)}}\left(\left(\*\Phi_j^{ (u)}  \*\Phi_i^{ (u)}+\*\Phi_i^{ (u)}  \*\Phi_j^{ (u)}\right)\left(\sum_{i=1}^c \mathfrak{q_i q_i}^{\top}-\frac{\boldsymbol{\lambda}_j^{ (u)}+\boldsymbol{\lambda}_i^{ (u)}}{2} \*D_l\right)\right)\right) 
\end{aligned}
\end{equation}
For the term $-\frac{\boldsymbol{\lambda}_j^{ (u)}}{2} \sum_{j=1}^k \operatorname{Tr}\left( \*D_l  \*\Phi_j^{ (u)}  +  \*\Phi_j^{ (u)}\*D_l \right) $, we can have the following equation:
\begin{equation}
    \begin{aligned}
        & -\frac{\boldsymbol{\lambda}_j^{ (u)}}{2} \sum_{j=1}^k \operatorname{Tr}\left( \*D_l  \*\Phi_j^{ (u)}  +  \*\Phi_j^{ (u)}\*D_l \right)  =- \sum_{j=1}^k \frac{\boldsymbol{\lambda}_j^{ (u)}}{2} \operatorname{Tr}\left( \*D_l  \*\Phi_j^{ (u)}  +  \*\Phi_j^{ (u)}\*D_l \right)\\
        &  =- \sum_{j=1}^k \frac{\boldsymbol{\lambda}_j^{ (u)}}{2} \operatorname{Tr}\left( ( \*\Phi_j^{ (u)} +  \*\Phi_j^{ (u)}) \*D_l  \right)\\
         &  =- \sum_{j=1}^k \frac{\boldsymbol{\lambda}_j^{ (u)}}{2} \operatorname{Tr}\left( ( \*\Phi_j^{ (u)} \*I+   \*I\*\Phi_j^{ (u)}) \*D_l  \right)\\
          &  =- \sum_{j=1}^k \frac{\boldsymbol{\lambda}_j^{ (u)}}{2} \operatorname{Tr}\left( ( \*\Phi_j^{ (u)} \sum_{i=1}^{N}  \*\Phi_i^{ (u)} 
 +\sum_{i=1}^{N}  \*\Phi_i^{ (u)} \*\Phi_j^{ (u)}  ) \*D_l  \right)\\
 &=- \sum_{j=1}^k \operatorname{Tr}\left(  \sum_{i=1}^{N} \frac{\boldsymbol{\lambda}_j^{ (u)}}{2}(   \*\Phi_j^{ (u)} \*\Phi_i^{ (u)} 
 +   \*\Phi_i^{ (u)}\*\Phi_j^{ (u)} ) \*D_l  \right) \\
    \end{aligned}
\end{equation}
Therefore, we can get the following result for $  \operatorname{Tr}\left(\left(\mathbf{V}_{k}(\phi_l) \boldsymbol{\Sigma}_{k}(\phi_l) \mathbf{V}_{k}(\phi_l)^{\top}\right)^{\prime}\right) \bigg|_{\phi_l=0}   $:
\begin{equation}
\scriptsize
    \begin{aligned}
    \operatorname{Tr}\left(\left(\mathbf{V}_k(\phi_l) \boldsymbol{\Sigma}_{k}(\phi_l) \mathbf{V}_{k}(\phi_l)^{\top}\right)^{\prime}\right)  \bigg|_{\phi_l=0}   & =
  \bigg[ \sum_{j=1}^k \operatorname{Tr}\left(\sum_{i=1}^k \frac{1}{2}\left(\left(\*\Phi_j^{ (u)}  \*\Phi_i^{ (u)}+\*\Phi_i^{ (u)} \*\Phi_j^{ (u)}\right)\left(\sum_{i=1}^c \mathfrak{q_i q_i}^{\top}-\frac{3\boldsymbol{\lambda}_j^{ (u)}+\boldsymbol{\lambda}_i^{ (u)}}{2} \*D_l\right)\right)\right) \\ 
 &+ \sum_{j=1}^k \operatorname{Tr}\left( \sum_{i=k+1}^N \frac{\boldsymbol{\lambda}_j^{ (u)}}{\boldsymbol{\lambda}_j^{ (u)}-\boldsymbol{\lambda}_i^{ (u)}}\left(\left(\*\Phi_j^{ (u)}  \*\Phi_i^{ (u)}+\*\Phi_i^{ (u)}  \*\Phi_j^{ (u)}\right)\left(\sum_{i=1}^c \mathfrak{q_i q_i}^{\top}-{\boldsymbol{\lambda}_j^{ (u)}} \*D_l\right)\right)\right) \bigg]  \\
 &  = \bigg[\sum_{j=1}^k \operatorname{Tr}\left(\sum_{i=1}^k \frac{1}{2}\left(\left(\*\Phi_j^{ (u)}  \*\Phi_i^{ (u)}+\*\Phi_i^{ (u)}  \*\Phi_j^{ (u)}\right)\left(\sum_{i=1}^c \mathfrak{q_i q_i}^{\top}-2 \boldsymbol{\lambda}_j^{ (u)} \*D_l\right)\right)\right) \\ 
& + \sum_{j=1}^k \operatorname{Tr}\left(\sum_{i=k+1}^N \frac{\boldsymbol{\lambda}_j^{ (u)}}{\boldsymbol{\lambda}_j^{ (u)}-\boldsymbol{\lambda}_i^{ (u)}}\left(\left(\*\Phi_j^{ (u)}  \*\Phi_i^{ (u)}+\*\Phi_i^{ (u)}  \*\Phi_j^{ (u)}\right)\left(\sum_{i=1}^c \mathfrak{q_i q_i}^{\top}-\boldsymbol{\lambda}_j^{ (u)} \*D_l\right)\right)\right) \bigg]\\ 
 & =\bigg[ \sum_{j=1}^k \sum_{i=1}^k \operatorname{Tr}\left(\*v_{j}^{(u) \top}  \*v_i^{ (u)} \cdot \*v_{i}^{(u) \top}\left(\sum_{i=1}^c \mathfrak{q_i q_i}^{\top}-2 \boldsymbol{\lambda}_j^{ (u)} \*D_l\right) \*v_j^{ (u)} \right) \\
& +\sum_{j=1}^k \sum_{i=k+1}^N \operatorname{Tr}\left(\frac{2 \boldsymbol{\lambda}_j^{ (u)}}{\boldsymbol{\lambda}_j^{ (u)}-\boldsymbol{\lambda}_i^{ (u)}} \*v_{j}^{(u) \top}  \*v_i^{ (u)} \cdot \*v_{i}^{(u) \top}\left(\sum_{i=1}^c \mathfrak{q_i q_i}^{\top}-\boldsymbol{\lambda}_j^{ (u)} \*D_l\right) \*v_j^{ (u)}\right)\bigg] .
    \end{aligned}
    \label{eq:mile1_eq9}
\end{equation}
Since $\tilde{\*A}^{(u)}$ has all positive eigenvalues, then we can get $\frac{ \boldsymbol{\lambda}_j^{ (u)}}{\boldsymbol{\lambda}_j^{ (u)}-\boldsymbol{\lambda}_i^{ (u)}} \geq 1$. Denote $\*V_k^{(u)}$ and $\*\Sigma_k^{(u)}$as the matrix of the first $k$ eigenvectors and eigenvalues of $\tilde{\*A}^{(u)}$ when $\phi_l=0$, if we rewrite the above formula in the matrix form,  we can get the following:
\begin{equation}
\small
\begin{aligned}
        &   \operatorname{Tr}\left(\left(\mathbf{V}_k(\phi_l) \boldsymbol{\Sigma}_{k}(\phi_l) \mathbf{V}_{k}(\phi_l)^{\top}\right)^{\prime}\right)  \bigg|_{\phi_l=0}    \geq 
          \bigg[  \operatorname{Tr}( \*V_{k}^{(u)\top} \*V_k^{(u)} \cdot\*V_{k}^{(u)\top} \sum_{i=1}^c \mathfrak{q_i q_i}^{\top} \*V_k^{(u)} ) - 2  \operatorname{Tr}( \*V_{k}^{(u)\top} \*V_k^{(u)} \cdot \*\Sigma_k^{(u)} 
 \*V_{k}^{(u)\top}\*D_l \*V_k^{(u)}) \\
        & + 2 \operatorname{Tr}( \*V_{\varnothing}^{(u)\top} \*V_k^{(u)} \cdot\*V_{k}^{(u)\top} \sum_{i=1}^c \mathfrak{q_i q_i}^{\top} \*V_{\varnothing}^{(u)} )  - 2 \operatorname{Tr}( \*V_{\varnothing}^{(u)\top} \*V_k^{(u)} \cdot \*\Sigma_k^{(u)} \*V_{k}^{(u)\top}  \*D_l\*V_{\varnothing}^{(u)} )   \\
        &=  \bigg[  \operatorname{Tr}( \*V_k^{(u)} \cdot\*V_{k}^{(u)\top} \sum_{i=1}^c \mathfrak{q_i q_i}^{\top} \*V_k^{(u)} \*V_{k}^{(u)\top} ) - 2  \operatorname{Tr}( \tilde{\*A}_{k}^{(u)} \*D_l \*V_k^{(u)} \cdot\*V_{k}^{(u)\top}  ) \\
     & + 2 \operatorname{Tr}\left( \*V_k^{(u)} \*V_{k}^{(u) \top}  \sum_{i=1}^c \mathfrak{q_i q_i}^{\top}\left(\*I_N- \*V_k^{(u)} \*V_{k}^{(u) \top}\right)\right)-2 \operatorname{Tr}\left( \tilde{\*A}_{k}^{(u)} \*D_l\left(\*I_N- \*V_k^{(u)} \*V_{k}^{(u) \top}\right)\right)  \bigg] \\
     & = \bigg[  2 \operatorname{Tr}\left( \*V_k^{(u)} \*V_{k}^{(u) \top}  \sum_{i=1}^c \mathfrak{q_i q_i}^{\top}\right)-2 \operatorname{Tr}\left( \tilde{\*A}_{k}^{(u)} \*D_l\right)-   \operatorname{Tr}\left( \*V_k^{(u)} \*V_{k}^{(u) \top}   \sum_{i=1}^c \mathfrak{q_i q_i}^{\top} \*V_k^{(u)} \*V_{k}^{(u) \top}  \right) \bigg] \\  
    & =  \bigg[ 2 \operatorname{Tr}\left(\*V_k^{(u)} \*V_{k}^{(u) \top}   \sum_{i=1}^c \mathfrak{q_i q_i}^{\top}-\tilde{\*A}_{k}^{(u)} \*D_l-\frac{1}{2}\*V_k^{(u)} \*V_{k}^{(u) \top}   \sum_{i=1}^c \mathfrak{q_i q_i}^{\top} \*V_k^{(u)} \*V_{k}^{(u) \top}  \right)  \bigg] \\
      & =  \bigg[ \operatorname{Tr}\left(\*V_k^{(u)} \*V_{k}^{(u) \top}   \sum_{i=1}^c \mathfrak{q_i q_i}^{\top}-2 \tilde{\*A}_{k}^{(u)} \*D_l+\*V_k^{(u)} \*V_{k}^{(u) \top}   \sum_{i=1}^c \mathfrak{q_i q_i}^{\top} \*V_{\varnothing}^{(u)} \*V_{\varnothing}^{(u)\top} \right) \bigg]. \\
      \label{penulitmate_bound_claim1}
\end{aligned}
\end{equation}
For $\operatorname{Tr}\left(\*V_k^{(u)} \*V_{k}^{(u) \top}   \sum_{i=1}^c \mathfrak{q_i q_i}^{\top}-2 \tilde{\*A}_{k}^{(u)} \*D_l+\*V_k^{(u)} \*V_{k}^{(u) \top}   \sum_{i=1}^c \mathfrak{q_i q_i}^{\top} \*V_{\varnothing}^{(u)} \*V_{\varnothing}^{(u)\top} \right) $, given that we let $\*D_l \triangleq \sum_{i=1}^c   {\rm diag}(\mathfrak{q}_i)$, and Assumption~\ref{Ass2}, i.e., $\mathfrak{q}_i^\top\*V_{\varnothing}^{(u)}$=0, we have that:
\begin{equation}
\begin{aligned}
    \operatorname{Tr}   \left(\left(\mathbf{V}_{k}(\phi_l) \boldsymbol{\Sigma}_{k}(\phi_l) \mathbf{V}_k(\phi_l)^{\top}\right)^{\prime}\right) \bigg|_{\phi_l=0}  & \geq\operatorname{Tr}\left(   \sum_{i=1}^c \mathfrak{q_i q_i}^{\top} -2 \tilde{\*A}_{k}^{(u)} \*D_l\right) = \operatorname{Tr}\left(   \sum_{i=1}^c \mathfrak{q_i q_i}^{\top} -2 \*F^{(u)}\*F^{(u)\top} \*D_l\right) \\& = \sum_{i=1}^c \|\mathfrak{q}_{i}\|_F^2 - 2 \*F_i^{ (u)\top}\*F_i^{ (u)}  \sum_{i=1}^c \|\mathfrak{q}_{i}\|_1 ,
\end{aligned}
\end{equation}

According to the definition of $r$, we have the upper bound of the ID features as $\*F_{ i}^{(u) \top}\*F_i^{ (u)}\leq r^2$. Therefore, we have the following lower bound:
\begin{equation}
     \operatorname{Tr}   \left(\left(\mathbf{V}_{k}(\phi_l) \boldsymbol{\Sigma}_{k}(\phi_l) \mathbf{V}_k(\phi_l)^{\top}\right)^{\prime}\right) \bigg|_{\phi_l=0} \geq \sum_{i=1}^c \|\mathfrak{q}_{i}\|_F^2 - 2 r^2 \sum_{i=1}^c \|\mathfrak{q}_{i}\|_1.
     \label{eq81_app}
\end{equation}

\end{proof}
$~~$

\begin{lemma*}
\label{lemma6_app} Under same conditions in Lemma~\ref{lemma5_app}, the lower bound for  $ \scriptsize \operatorname{Tr}\left(\left(\tilde{\*A}_{\rm OI}(\phi_l) \mathbf{V}_k(\phi_l) \boldsymbol{\Sigma}_k(\phi_l)^{-1} \mathbf{V}_k(\phi_l)^{\top} \tilde{\*A}_{\rm OI}(\phi_l)^{\top}\right)^{\prime}\right)\bigg|_{\phi_l=0}$ is given as follows:
\begin{equation}
\begin{aligned}
         &\operatorname{Tr}\left(\left(\tilde{\*A}_{{\rm OI}}(\phi_l) \mathbf{V}_k(\phi_l) \boldsymbol{\Sigma}_k(\phi_l)^{-1} \mathbf{V}_k(\phi_l)^{\top} \tilde{\*A}_{{\rm OI}}(\phi_l)^{\top}\right)^{\prime}\right)\bigg|_{\phi_l=0}  \\ &  \geq 2\sum_{i=1}^{c} \operatorname{Tr}\left(\mathfrak{p}_i\mathfrak{q}_i^\top \cdot   \tilde{\*A}_{\rm OI}^{(u)\top}\right)  -\|\tilde{\*A}_{\rm OI}^{(u)} \|_F^2  \|\tilde{\*A}^{(u)}\|_F^2  \cdot(  {\sum_{i=1}^c \|\mathfrak{q}_i\|_F^2 - \frac{ 2r^2(\tau- k)}{\tau-1}}\sum_{i=1}^c\| \mathfrak{q}_{i}\|_1),
\end{aligned}
\end{equation}
where $\mathfrak{p}\in \mathbb{R}^{M\times c}$ is defined as $(\mathfrak{p}_i)_\*x = \mathbb{E}_{\Bar{\*x}_{l} \sim {\mathbb{P}_{l_i}}} \mathcal{T}(\*x | \Bar{\*x}_{l}), \*x \in \mathcal{X}_{\rm ood}$.
\end{lemma*}

\begin{proof}[Proof of Lemma \ref{lemma6_app}] 
    Given the fact that $\tilde{\*A}_{\rm OI}(\phi_l) = \phi_u \tilde{\*A}_{\rm OI}^{(u)}  + \phi_l \cdot \sum_{i=1}^c\mathfrak{p}_i\mathfrak{q}_i^\top$, we have the following:
    \begin{equation}
    \begin{aligned}
             & \operatorname{Tr}\left(\left(\tilde{\*A}_{\rm OI}(\phi_l) \mathbf{V}_k(\phi_l) \boldsymbol{\Sigma}_k(\phi_l)^{-1} \mathbf{V}_k(\phi_l)^{\top} \tilde{\*A}_{\rm OI}(\phi_l)^{\top}\right)^{\prime}\right)\bigg|_{\phi_l=0}  \\&=  2 \operatorname{Tr}\left( \sum_{i=1}^c\mathfrak{p}_i\mathfrak{q}_i^\top \cdot  \mathbf{V}_k^{(u)} {\boldsymbol{\Sigma}_k^{(u)}}^{-1}
 \mathbf{V}_k^{(u)\top} \tilde{\*A}_{\rm OI}^{(u)\top}\right) +  \operatorname{Tr}\left(\tilde{\*A}_{\rm OI}^{(u)} \left(\mathbf{V}_k(\phi_l) \boldsymbol{\Sigma}_k(\phi_l)^{-1} \mathbf{V}_k(\phi_l)^{\top}\right)^\prime \tilde{\*A}_{\rm OI}^{(u)\top} \right)\bigg|_{\phi_l=0} 
  \\& = 2 \operatorname{Tr}\left( \sum_{i=1}^c\mathfrak{p}_i\mathfrak{q}_i^\top \cdot ( \tilde{\*A}^{(u)})^{-1} \tilde{\*A}_{\rm OI}^{(u)\top}\right) + \operatorname{Tr}\left( (\tilde{\*A}_{k}(\phi_l)^{-1})^\prime \cdot  \tilde{\*A}_{\rm OI}^{(u)\top} \tilde{\*A}_{\rm OI}^{(u)}\right) \bigg|_{\phi_l=0}
 \\& \geq   2 \operatorname{Tr}\left( \sum_{i=1}^c\mathfrak{p}_i\mathfrak{q}_i^\top \cdot  ( \tilde{\*A}^{(u)})^{-1} \tilde{\*A}_{\rm OI}^{(u)\top}\right) + \min \lambda(\tilde{\*A}_{\rm OI}^{(u)\top} \tilde{\*A}_{\rm OI}^{(u)}) \cdot \operatorname{Tr}\left((\tilde{\*A}_{k}(\phi_l)^{-1})^\prime\right) \bigg|_{\phi_l=0}
 \\& =  2 \operatorname{Tr}\left( \sum_{i=1}^c\mathfrak{p}_i\mathfrak{q}_i^\top \cdot  ( \tilde{\*A}^{(u)})^{-1} \tilde{\*A}_{\rm OI}^{(u)\top}\right) - \min \lambda( \tilde{\*A}_{\rm OI}^{(u)\top} \tilde{\*A}_{\rm OI}^{(u)}) \cdot  \operatorname{Tr}\left( (\tilde{\*A}^{(u)})^{-2} \cdot  (\tilde{\*A}_{k}(\phi_l))^\prime\right) \bigg|_{\phi_l=0}
 \\& \geq   2 \operatorname{Tr}\left( \sum_{i=1}^c\mathfrak{p}_i\mathfrak{q}_i^\top \cdot  ( \tilde{\*A}^{(u)})^{-1} \tilde{\*A}_{\rm OI}^{(u)\top}\right) - \min \lambda( \tilde{\*A}_{\rm OI}^{(u)\top} \tilde{\*A}_{\rm OI}^{(u)}) \cdot \max \lambda((\tilde{\*A}^{(u)})^{-2}) \cdot  \operatorname{Tr}\left(   (\tilde{\*A}_{k}(\phi_l))^\prime\right) \bigg|_{\phi_l=0}
 \\& = 2 \operatorname{Tr}\left( \sum_{i=1}^c\mathfrak{p}_i\mathfrak{q}_i^\top \cdot  ( \tilde{\*A}^{(u)})^{-1} \tilde{\*A}_{\rm OI}^{(u)\top}\right) - \min \lambda( \tilde{\*A}_{\rm OI}^{(u)\top} \tilde{\*A}_{\rm OI}^{(u)}) \cdot (\min \lambda(\tilde{\*A}^{(u)}))^2 \cdot  \operatorname{Tr}\left(   (\tilde{\*A}_{k}(\phi_l))^\prime\right) \bigg|_{\phi_l=0}
 \\& \geq 2 \operatorname{Tr}\left( \sum_{i=1}^c\mathfrak{p}_i\mathfrak{q}_i^\top \cdot  ( \tilde{\*A}^{(u)})^{-1} \tilde{\*A}_{\rm OI}^{(u)\top}\right) - \|\tilde{\*A}_{\rm OI}^{(u)} \|_F^2  \cdot (\min \lambda(\tilde{\*A}^{(u)}))^2 \cdot  \operatorname{Tr}\left(   (\tilde{\*A}_{k}(\phi_l))^\prime\right) \bigg|_{\phi_l=0}
  \\& \geq 2 \operatorname{Tr}\left( \sum_{i=1}^c\mathfrak{p}_i\mathfrak{q}_i^\top \cdot  ( \tilde{\*A}^{(u)})^{-1} \tilde{\*A}_{\rm OI}^{(u)\top}\right) - \|\tilde{\*A}_{\rm OI}^{(u)} \|_F^2  \cdot \|\tilde{\*A}^{(u)}\|_F^2 \cdot  \operatorname{Tr}\left(   (\tilde{\*A}_{k}(\phi_l))^\prime\right) \bigg|_{\phi_l=0}
  \\& = 2  \operatorname{Tr}\left( \sum_{i=1}^c\mathfrak{p}_i\mathfrak{q}_i^\top \tilde{\*A}_{\rm OI}^{(u)\top}\right) - \|\tilde{\*A}_{\rm OI}^{(u)} \|_F^2  \cdot \|\tilde{\*A}^{(u)}\|_F^2 \cdot  \operatorname{Tr}\left(   (\tilde{\*A}_{k}(\phi_l))^\prime\right) \bigg|_{\phi_l=0}
    \end{aligned}
    \end{equation} 
The first two inequalities are again the inequality for the trace of matrix product of two square matrices~\cite{362841}. The third inequality holds because:
\begin{equation}
    \min \lambda( \tilde{\*A}_{\rm OI}^{(u)\top} \tilde{\*A}_{\rm OI}^{(u)}) < \operatorname{Tr}(\tilde{\*A}_{\rm OI}^{(u)\top} \tilde{\*A}_{\rm OI}^{(u)}) = \|\tilde{\*A}_{\rm OI}^{(u)} \|_F^2.
\end{equation}
The last inequality holds because:
\begin{equation}
    (\min \lambda(\tilde{\*A}^{(u)}))^2  = \min \lambda(\tilde{\*A}^{(u)2}) = \min \lambda(\tilde{\*A}^{(u)} \cdot \tilde{\*A}^{(u)\top}) \leq  \|\tilde{\*A}^{(u)}\|_F^2
\end{equation}
Finally, the last equation holds true because of Assumption~\ref{Ass2} where the perturbation vector $\mathfrak{q}_i$ lies in the linear span of the matrix $\*V_k^{(u)}[{\*\Sigma_k^{(u)}}]^{-\frac{1}{2}}$: 
\begin{equation}
    \mathfrak{q}_i^\top \cdot  ( \tilde{\*A}^{(u)})^{-1} = \mathfrak{q}_i^\top \cdot   \mathbf{V}_k^{(u)} {\boldsymbol{\Sigma}_k^{(u)}}^{-1}
 \mathbf{V}_k^{(u)\top}  =  \mathfrak{q}_i^\top .
\end{equation}

For the key term $\operatorname{Tr}\left(\tilde{\*A}_{k}(\phi_l)^\prime \right) \bigg|_{\phi_l=0}$, it is easy to check that $\frac{ \boldsymbol{\lambda}_j^{ (u)}}{\boldsymbol{\lambda}_j^{ (u)}-\boldsymbol{\lambda}_i^{ (u)}} \leq  \frac{\tau}{\tau-1},j\leq k<k+1\leq i$ because $\tau$ is a constant that measures the $k$-th spectral gap of matrix $\tilde{\*A}^{(u)}$, i.e., $\boldsymbol{\lambda}_{k}^{(u)} \geq \tau \boldsymbol{\lambda}_{k+1}^{(u)}$, we can similarly get its upper bound to Equation~\ref{penulitmate_bound_claim1} as follows:
\begin{equation}
\small
\begin{aligned}
        \operatorname{Tr}\left(\tilde{\*A}_{k}(\phi_l)^\prime \right) \bigg|_{\phi_l=0} &= \operatorname{Tr}   \left(\left(\mathbf{V}_{k}(\phi_l) \boldsymbol{\Sigma}_{k}(\phi_l) \mathbf{V}_{k}(\phi_l)^{\top}\right)^{\prime}\right)\bigg|_{\phi_l=0}  \\& \leq   \bigg[  \operatorname{Tr}( \*V_k^{(u)} \cdot\*V_{k}^{(u)\top} \sum_{i=1}^c \mathfrak{q_i q_i}^{\top} \*V_k^{(u)} \*V_{k}^{(u)\top} ) - 2  \operatorname{Tr}( \tilde{\*A}_{k}^{(u)} \*D_l \*V_k^{(u)} \cdot\*V_{k}^{(u)\top}  ) \\
     & + 2\frac{\tau}{\tau-1} \operatorname{Tr}\left( \*V_k^{(u)} \*V_{k}^{(u) \top}  \sum_{i=1}^c \mathfrak{q_i q_i}^{\top}\left(\*I_N- \*V_k^{(u)} \*V_{k}^{(u) \top}\right)\right)-2 \frac{\tau}{\tau-1}\operatorname{Tr}\left( \tilde{\*A}_{k}^{(u)} \*D_l\left(\*I_N- \*V_k^{(u)} \*V_{k}^{(u) \top}\right)\right)  \bigg] \\ & = 2 \frac{\tau}{\tau-1}\operatorname{Tr}\left( \*V_k^{(u)} \*V_{k}^{(u) \top}  \sum_{i=1}^c \mathfrak{q_i q_i}^{\top}\right) - \frac{\tau+1}{\tau-1}    \operatorname{Tr}\left( \*V_k^{(u)} \*V_{k}^{(u) \top}   \sum_{i=1}^c \mathfrak{q_i q_i}^{\top} \*V_k^{(u)} \*V_{k}^{(u) \top}  \right) \\& + \frac{2}{\tau-1} \operatorname{Tr}\left( \tilde{\*A}_{k}^{(u)} \*D_l\*V_k^{(u)} \*V_{k}^{(u) \top}\right) - 2 \frac{\tau}{\tau-1}\operatorname{Tr}\left( \tilde{\*A}_{k}^{(u)} \*D_l\right) \\& = \operatorname{Tr}\left( \*V_k^{(u)} \*V_{k}^{(u) \top}  \sum_{i=1}^c \mathfrak{q_i q_i}^{\top}\right)+ \frac{2}{\tau-1} \operatorname{Tr}\left( \tilde{\*A}_{k}^{(u)} \*D_l\*V_k^{(u)} \*V_{k}^{(u) \top}\right) - 2 \frac{\tau}{\tau-1}\operatorname{Tr}\left( \tilde{\*A}_{k}^{(u)} \*D_l\right) \\&  \leq \operatorname{Tr}\left( \*V_k^{(u)} \*V_{k}^{(u) \top}  \sum_{i=1}^c \mathfrak{q_i q_i}^{\top}\right) + \frac{2}{\tau-1}\cdot \|\*V_k^{(u)}\|_F^2 \operatorname{Tr}\left( \tilde{\*A}_{k}^{(u)} \*D_l\right) - 2 \frac{\tau}{\tau-1}\operatorname{Tr}\left( \tilde{\*A}_{k}^{(u)} \*D_l\right)\\& \leq 
      \operatorname{Tr}\left( \*V_k^{(u)} \*V_{k}^{(u) \top}  \sum_{i=1}^c \mathfrak{q_i q_i}^{\top}\right) + \frac{2(k-\tau)}{\tau-1} \operatorname{Tr}\left( \tilde{\*A}_{k}^{(u)} \*D_l\right)
      \\& =   {\sum_{i=1}^c \|\mathfrak{q}_i\|_F^2 - \frac{ 2\*F_{ i}^{ (u)\top}\*F_i^{ (u)}(\tau- k)}{\tau-1}}\sum_{i=1}^c\| \mathfrak{q}_{i}\|_1  
       =  {\sum_{i=1}^c \|\mathfrak{q}_i\|_F^2 - \frac{ 2r^2(\tau- k)}{\tau-1}}\sum_{i=1}^c\| \mathfrak{q}_{i}\|_1
\end{aligned}
\end{equation}

Therefore, it is natural to obtain that:
\begin{equation}
\begin{aligned}
      &   \operatorname{Tr}\left(\left( \tilde{\mathbf{A}}_{\rm OI}(\phi_l) \mathbf{V}_k(\phi_l) \boldsymbol{\Sigma}_k(\phi_l)^{-1} \mathbf{V}_k(\phi_l)^{\top} \tilde{\mathbf{A}}_{\rm OI}(\phi_l)^{\top}\right)^{\prime}\right)\bigg|_{\phi_l=0} \geq\\& 2\sum_{i=1}^{c} \operatorname{Tr}\left(\mathfrak{p}_i\mathfrak{q}_i^\top \cdot   \tilde{\*A}_{\rm OI}^{(u)\top}\right)  -\|\tilde{\*A}_{\rm OI}^{(u)} \|_F^2  \|\tilde{\*A}^{(u)}\|_F^2  \cdot (  {\sum_{i=1}^c \|\mathfrak{q}_i\|_F^2 - \frac{ 2r^2(\tau- k)}{\tau-1}}\sum_{i=1}^c\| \mathfrak{q}_{i}\|_1).
\end{aligned}
\label{final_bound_claim2}
\end{equation}

\end{proof}
\newpage

%% file: empi_verifi.tex
\section{Empirical Verification of the Main Theorems}
\label{sec:verification_discrepancy}

\textbf{Verification of Theorems.} We provide more verification results on \textsc{Cifar10}. Firstly, we verify how the value of $\mathcal{G}$ will change given a larger Frobenius norm of $\tilde{\*A}_{\rm OI}^{(u)}$.

\begin{table}[!h]
    \centering
    \begin{tabular}{c|cc}
    \hline 
    OOD dataset & \multicolumn{1}{c}{\textsc{Svhn}}  & \multicolumn{1}{c}{\textsc{C100}}\\
   \cline{2-3}
   & {\small \textsc{Far OOD}}& {\small\textsc{Near OOD}}\\
   \hline
 $\| \tilde{\*A}_{\rm OI}^{(u)} \|_F$  $\uparrow$  & $2583$& $2876$ \\
$\mathcal{G}$ $\uparrow$ & 0.00 & \textbf{0.14} \\
       \hline
    \end{tabular}
    \caption{ Verification with different $\| \tilde{\*A}_{\rm OI}^{(u)} \|_F$ (\textsc{Cifar10} as ID).}
    \label{tab:veri_a_ood_id_app}
\end{table}

Next we verify the relationship of $\|\tilde{\*A}^{(u)}\|_F$ and the error difference (\textsc{Cifar100} as OOD) in Table~\ref{tab:veri_a_id_id_app}.
\begin{table}[!h]
    \centering
    \begin{tabular}{c|ccccccc}
    \hline 
    Epochs  & 40 & 80 & 120  & 160& 200 & 240 \\
    \hline
      $\| \tilde{\*A}^{(u)} \|_F$   $\downarrow$  &19873 & 19762& 17539 & 16640& 15982 & 15361\\
     $ \mathcal{G} $ $\uparrow$& 0.03 & 0.04 & 0.07 & 0.11& 0.14 & 0.15\\
       \hline
    \end{tabular}
    \caption{Verification with different $\| \tilde{\*A}^{(u)} \|_F$ (\textsc{Cifar10} dataset as ID).}
    \label{tab:veri_a_id_id_app}
\end{table}

These two tables show a similar result as when the ID dataset is \textsc{Cifar100}, where the error difference on \textsc{Cifar100} (near OOD, with larger $\|\tilde{\*A}_{\rm OI}^{(u)}\|_F$) is consistently larger than that on \textsc{Svhn} (far OOD). Moreover, the difference in linear probing error $\mathcal{G}$ tends to increase with decreasing $\|\tilde{\*A}^{(u)}\|_F$, which aligns with Theorem~\ref{MainT-3}. 

\textbf{Verification of the assumptions.} Here are the empirical verifications of the assumptions on real-world datasets, i.e., \textsc{Cifar10} and \textsc{Cifar100} datasets and simulated datasets as shown in Figure~\ref{fig:toy_example}. 

For \textsc{Cifar10} and \textsc{Cifar100}, we have checked the largest eigengap of matrix $ \tilde{\*A}^{(u)}$ and observed that $\tau=7740.92$ for \textsc{Cifar10} and $\tau=8834.78$ for \textsc{Cifar100}, which are much larger than the feature dimension $512$. Therefore, we can always find a proper $k$ such that $\tau > k$ is satisfied because $k\leq 512$.

For the simulated datasets as shown in Figure~\ref{fig:toy_example}, the largest eigengap of matrix $ \tilde{\*A}^{(u)}$ is 872.00. Therefore the condition $\tau>k$ always holds because $k=2$ in this case. We have provided more details about the matrix itself in Section~\ref{detail_toy_example_app}.

\section{Details of the Illustrative Example}
\label{detail_toy_example_app}
For Figure~\ref{fig:toy_example} in the main paper, we generate the augmentation graph $\mathcal{T}$ as follows:
\begin{equation}
   \mathcal{T} =  \begin{pmatrix}
        \*B_1 &   \*B_2 & \*B_2  \\
          \*B_2 &   \*B_3 & \*B_2  \\
            \*B_2 &   \*B_2 & \*B_4 
    \end{pmatrix},
\end{equation}
where the block-wise matrices $\*B_1, \*B_2, \*B_3, \*B_4 \in \mathbb{R}^{40 \times 40}$ are square matrices. Specifically, $\*B_1$ has the following definition:
\begin{equation}
\begin{aligned}
    &  \*B_1 = \*1_{40 \times 40} \cdot B_1 + \frac{1}{2} \boldsymbol{\varepsilon}_1 + \frac{1}{2}\boldsymbol{\varepsilon}_1^\top, \\&
    \*B_2 = \*1_{40 \times 40} \cdot B_2 + \frac{1}{2} \boldsymbol{\varepsilon}_2 + \frac{1}{2}\boldsymbol{\varepsilon}_2^\top,\\&
     \*B_3 = \*1_{40 \times 40} \cdot B_3 + \frac{1}{2} \boldsymbol{\varepsilon}_3 + \frac{1}{2}\boldsymbol{\varepsilon}_3^\top,\\&
      \*B_4 = \*1_{40 \times 40} \cdot B_4 + \frac{1}{2} \boldsymbol{\varepsilon}_4 + \frac{1}{2}\boldsymbol{\varepsilon}_4^\top,
\end{aligned}
\end{equation}
here $ \boldsymbol{\varepsilon}_1,  \boldsymbol{\varepsilon}_2,  \boldsymbol{\varepsilon}_3,  \boldsymbol{\varepsilon}_4\in \mathbb{R}^{40 \times 40}$ are matrices where each element of them is sampled from a truncated normal distribution (lower bound is -0.1, upper bound is 0.1, mean is 0 and variance is 1). $\*1_{40 \times 40}$ is a matrix where each element in it is 1. Essentially, the matrix $\*B_1$ measures the connectivity pattern between the data that belongs to the first ID class. We set $B_1=0.8$. 

Similarly, $\*B_3$ and $\*B_4$ measure the connectivity pattern between the data that belongs to the second/third ID class, and we set $B_3=0.75, B_4=0.7$ respectively for $\*B_3$ and $\*B_4$. Moreover, the matrix  $\*B_2$ measures the connectivity pattern between the data that belongs to different ID classes and we set $B_2=0.1$ in this case. For the ID adjacency matrix $\*A^{(u)}$ and $\*A^{(l)}$, we can follow the definition in Definitions~\ref{Def4} and~\ref{Def5} of the main paper for calculation. 

In addition, we generate the OOD-ID adjacency matrix $\tilde{\*A}_{\rm OI}^{(u)}$ by sampling from a truncated normal distribution where the lower bound is 0, the upper bound is 0.5, the mean is 0.5 and variance is 0.05 in the near OOD scenario. In the far OOD scenario, we set the lower bound to be 0, the upper bound to be 0.2, the mean to be 0.2, and the variance to be 0.05 for the truncated normal distribution. In the labeled case, the OOD-ID adjacency matrix $\tilde{\*A}_{\rm OI}^{(l)}$ can be calculated according to Equation~\ref{eq23_app}. Here $\mathfrak{q}$ can be calculated based on the augmentation graph $\mathcal{T}$ and each element of $\mathfrak{p}$ (that represents the connectivity probability between the OOD and ID data) is set to 0.1 because we let the connectivity probability be 0.1 when the data belongs to different classes in the augmentation graph $\mathcal{T}$. Finally, we set $\phi_u=1$ and $\phi_l=0.5$ for the calculation of all the adjacency matrices and the data representations.

\section{Additional Results on the Illustrative Example}
\label{result_toy_example_app}
In this section, we provide additional visualization results on changing the Frobenius norm of the ID adjacency matrix $\tilde{\*A}^{(u)}$ and the semantic connection $\mathfrak{q}$ in Figure~\ref{fig:toy_example_app}. For all the examples, we visualize under the near OOD scenario (meaning that we use the same $\tilde{\*A}_{\rm OI}$ as in Section~\ref{detail_toy_example_app}). Please check the caption for a detailed explanation.

\begin{figure*}[!h]
  \begin{center}
   {\includegraphics[width=1\linewidth]{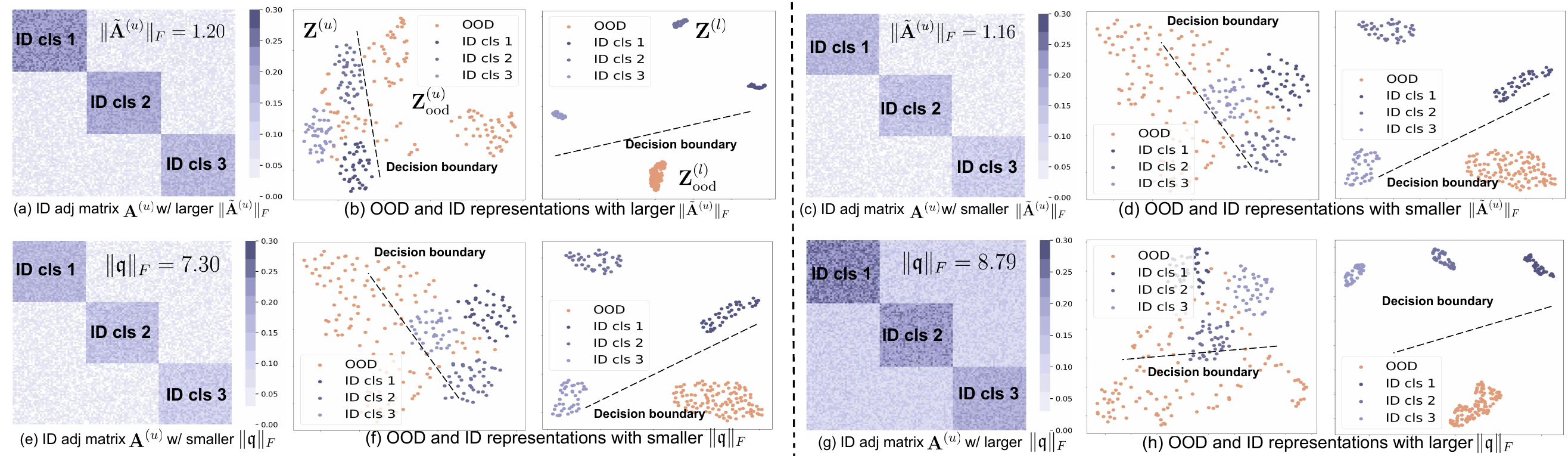}}
  \end{center}
  \vspace{-1em}
  \caption{\small  Additional example showcasing the contrast between adjacency matrices and representations w/ (l) and w/o (u) ID labels. (a) The ID adjacency matrix in the unlabeled case $\*A^{(u)}$ with a larger $\|\tilde{\*A}^{(u)}\|_F$ ($B_1= 0.8, B_2=0.1, B_3=0.75, B_4=0.7$). (b) The contrast of the learned representations in both labeled and unlabeled cases when $\|\tilde{\*A}^{(u)}\|_F=1.20$. (c) The ID adjacency matrix in the unlabeled case $\*A^{(u)}$ with a smaller $\|\tilde{\*A}^{(u)}\|_F$ ($B_1= 0.7, B_2=0.1, B_3=0.65, B_4=0.6$). (d) The contrast of the learned representations in both labeled and unlabeled cases when $\|\tilde{\*A}^{(u)}\|_F=1.16$. Compared with (b) where the difference in the linear probing loss $\mathcal{G}$ is 0.09, the linear probing loss reduces from 0.16 to 0.02. (e) The ID adjacency matrix in the unlabeled case $\*A^{(u)}$ with a smaller $\|\mathfrak{q}\|_F$ ($B_1= 0.7, B_2=0.1, B_3=0.65, B_4=0.6$). (f) The contrast of the learned representations in both labeled and unlabeled cases when $\|\mathfrak{q}\|_F=7.30$. (g) The ID adjacency matrix in the unlabeled case $\*A^{(u)}$ with a larger $\|\mathfrak{q}\|_F$ ($B_1= 0.8, B_2=0.2, B_3=0.75, B_4=0.7$). (h) The contrast of the learned representations in both labeled and unlabeled cases when $\|\mathfrak{q}\|_F=8.79$. Compared with (f) where the difference in the linear probing loss $\mathcal{G}$ is 0.14, the linear probing loss reduces from 0.16 to 0.00. The visualization aligns with our theoretical reasoning as shown in Section~\ref{sec:4.3}.}
  \label{fig:toy_example_app}
  \end{figure*}

\section{Additional Experimental Details}
\label{add_experiment_detail}
We provide more training details for both contrastive learning and linear probing. For contrastive training, we use the same data augmentation strategies as SimSiam~\cite{chen2021exploring}. For \textsc{Cifar10}, we set $\phi_u=0.5,\phi_l=0.25$ with training epoch 200, and we evaluate using features extracted from the layer preceding the projection. For \textsc{Cifar100}, we
set $\phi_u=3,\phi_l=0.0225$ with 200 training epochs and assess based on the projection layer’s features. We use SGD with momentum 0.9 as an optimizer with cosine annealing (lr=0.03), weight decay 5e-4,
and batch size 512. 

For linear probing, we train a linear layer on the extracted features from the pretrained model by contrastive learning. We use SGD for 50 epochs with momentum 0.9 as an optimizer with a decayed learning rate in epoch 30 by 0.2 (The initial learning rate is 5),
and batch size 512.

\section{Additional Experimental Results}
\label{add_result_c10_app}

\textbf{Results on \textsc{Cifar10}.} We present the experimental results on \textsc{Cifar10} in Table~\ref{exp_result_app}, where the effect of the ID labels is similar to the results on \textsc{Cifar100}.

\begin{table*}[!h]
    \centering
    \small
    \begin{tabular}{cc|ccccccc}
    \hline
   OOD category&     OOD dataset & ID labels & FPR95 & AUROC & LP error &   FPR95 & AUROC & LP error\\
        \hline
       &   &  &\multicolumn{3}{c}{ $\mathbb{P}_{\rm ood}^{\rm test} =\mathbb{P}_{\rm ood}^{\rm LP} $}      &  \multicolumn{3}{c}{  $\mathbb{P}_{\rm ood}^{\rm test} \neq\mathbb{P}_{\rm ood}^{\rm LP} $}     \\
        \cline{4-9}
\multirow{10}{*}{\textsc{Far OOD}}      &  \multirow{2}{*}{\textsc{Svhn} }     & - & 0.05$^{\pm 0.01}$&\textbf{99.98}$^{\pm 0.00}$ &{0.01}$^{\pm 0.00}$   & 36.60$^{\pm 2.84}$ & 91.71$^{\pm 1.90}$ & 0.33$^{\pm 0.11}$ \\
  &   &+& \cellcolor[HTML]{EFEFEF}\textbf{0.03}$^{\pm 0.00}$&\cellcolor[HTML]{EFEFEF}\textbf{99.98}$^{\pm 0.01}$ & \cellcolor[HTML]{EFEFEF}\textbf{0.00}$^{\pm 0.00}$&   \cellcolor[HTML]{EFEFEF}\textbf{8.62}$^{\pm 1.12}$ & \cellcolor[HTML]{EFEFEF}\textbf{98.43}$^{\pm 0.47}$  & \cellcolor[HTML]{EFEFEF}\textbf{0.25}$^{\pm 0.19}$
    \\
   &    \multirow{2}{*}{\textsc{Textures} }     &-&  \textbf{0.21}$^{\pm 0.03}$&\textbf{99.96}$^{\pm 0.01}$& \textbf{0.01}$^{\pm 0.00}$ &\textbf{10.16}$^{\pm 0.72}$ &  \textbf{98.13}$^{\pm 2.03}$ & \textbf{0.52}$^{\pm 0.10}$ \\
&     & +& \cellcolor[HTML]{EFEFEF}{0.57}$^{\pm 0.06}$&\cellcolor[HTML]{EFEFEF}{99.80}$^{\pm 0.10}$ & \cellcolor[HTML]{EFEFEF}\textbf{0.01}$^{\pm 0.00}$ & \cellcolor[HTML]{EFEFEF}13.87$^{\pm 0.82}$ & \cellcolor[HTML]{EFEFEF}{97.60}$^{\pm 1.39}$ &  \cellcolor[HTML]{EFEFEF}0.58$^{\pm 0.18}$\\
  &     \multirow{2}{*}{\textsc{Places365} }     &- & \textbf{0.36}$^{\pm 0.01}$&\textbf{99.91}$^{\pm 0.00}$ & \textbf{0.02}$^{\pm 0.00}$& 20.19$^{\pm 1.23}$   & 96.18$^{\pm 0.28}$ & 0.43$^{\pm 0.16}$\\
  &   &+&\cellcolor[HTML]{EFEFEF}{0.68}$^{\pm 0.02}$&\cellcolor[HTML]{EFEFEF}{99.68}$^{\pm 0.06}$ & \cellcolor[HTML]{EFEFEF}\textbf{0.02}$^{\pm 0.01}$&  \cellcolor[HTML]{EFEFEF}\textbf{15.83}$^{\pm 1.83}$ & \cellcolor[HTML]{EFEFEF}\textbf{97.01}$^{\pm 0.96}$ & \cellcolor[HTML]{EFEFEF}\textbf{0.26}$^{\pm 0.06}$\\ 
 &    \multirow{2}{*}{\textsc{Lsun-Resize} }     & - & {0.68}$^{\pm 0.01}$&\textbf{99.78}$^{\pm 0.18}$ &  {0.03}$^{\pm 0.00}$   &18.21$^{\pm 1.71}$ & 96.87$^{\pm 3.28}$& 0.42$^{\pm 0.21}$ \\
  &   &+& \cellcolor[HTML]{EFEFEF}\textbf{0.56}$^{\pm 0.23}$&\cellcolor[HTML]{EFEFEF}\textbf{99.78}$^{\pm 0.10}$&  \cellcolor[HTML]{EFEFEF}\textbf{0.01}$^{\pm 0.01}$&\cellcolor[HTML]{EFEFEF}\textbf{10.83}$^{\pm 1.18}$ & \cellcolor[HTML]{EFEFEF}\textbf{97.94}$^{\pm 0.23}$ & \cellcolor[HTML]{EFEFEF}\textbf{0.39}$^{\pm 0.03}$\\ 
   &    \multirow{2}{*}{\textsc{Lsun-C} }     & - & 0.68$^{\pm 0.12}$&\textbf{99.80}$^{\pm 0.05}$ &0.04$^{\pm 0.02}$ & 11.10$^{\pm 1.78}$ & 97.93$^{\pm 1.89}$ & 0.42$^{\pm 0.07}$ \\
  &   &+&  \cellcolor[HTML]{EFEFEF}\textbf{0.32}$^{\pm 0.16}$&\cellcolor[HTML]{EFEFEF}{99.65}$^{\pm 0.05}$ &  \cellcolor[HTML]{EFEFEF}\textbf{0.02}$^{\pm 0.02}$&\cellcolor[HTML]{EFEFEF}\textbf{7.95}$^{\pm 1.92}$ & \cellcolor[HTML]{EFEFEF}\textbf{98.44}$^{\pm 0.21}$  & \cellcolor[HTML]{EFEFEF}\textbf{0.40}$^{\pm 0.19}$\\
  \hline
\multirow{2}{*}{\textsc{Near OOD}}   &   \multirow{2}{*}{\textsc{Cifar100} }     & -& 49.88$^{\pm 1.81}$&89.29$^{\pm 0.99}$ &0.24$^{\pm 0.07}$ &54.76$^{\pm 2.21}$ & 86.65$^{\pm 3.04}$ & 0.53$^{\pm 0.08}$ \\
 &    &+& \cellcolor[HTML]{EFEFEF}\textbf{41.56}$^{\pm 1.49}$&\cellcolor[HTML]{EFEFEF}\textbf{92.83}$^{\pm 0.73}$ & \cellcolor[HTML]{EFEFEF}\textbf{0.10}$^{\pm 0.02}$&\cellcolor[HTML]{EFEFEF}\textbf{40.04}$^{\pm 2.92}$ & \cellcolor[HTML]{EFEFEF}\textbf{92.42}$^{\pm 0.18}$ & \cellcolor[HTML]{EFEFEF}\textbf{0.41}$^{\pm 0.19}$\\
    \hline
    \end{tabular}
    \caption{\small OOD detection results w/ and w/o ID labels (\textsc{Cifar10} as ID). Mean and std are estimated on three different runs. Better results are highlighted in bold. ``+,-" denotes the labeled and unlabeled case. ``LP error" denotes the error of linear probing. }
    \label{exp_result_app}
\end{table*}

\textbf{Results on using post-hoc OOD detection scores.} Instead of using linear probing to evaluate the OOD detection performance, we investigate another approach, which directly calculates the $k$-NN score~\cite{sun2022out} on top of the extracted representations for both the ID and OOD data and then computes the OOD detection metrics for comparison. The result is shown in Table~\ref{exp_result_app_post_hoc}, where the OOD detection performance is usually better in the labeled case.

\begin{table*}[!h]
    \centering
    \small
    \begin{tabular}{cc|ccccc}
    \hline
   OOD category&     OOD dataset & ID labels & FPR95 & AUROC &    FPR95 & AUROC \\
        \hline
       &   &  &\multicolumn{2}{c}{ \textsc{Cifar10} as ID}      &  \multicolumn{2}{c}{  \textsc{Cifar100} as ID}     \\
        \cline{4-7}
\multirow{10}{*}{\textsc{Far OOD}}      &  \multirow{2}{*}{\textsc{Svhn} }     & - & 37.35$^{\pm 3.12}$ & 88.08$^{\pm 2.91}$&  \textbf{27.27}$^{\pm 0.88}$ & \textbf{93.85}$^{\pm 2.16}$\\
  &   &+& \cellcolor[HTML]{EFEFEF}\textbf{14.97}$^{\pm 1.81}$ & \cellcolor[HTML]{EFEFEF}\textbf{97.82}$^{\pm 0.96}$ &   \cellcolor[HTML]{EFEFEF}55.39$^{\pm 4.78}$ &  \cellcolor[HTML]{EFEFEF}86.31$^{\pm 1.12}$
     \\
   &    \multirow{2}{*}{\textsc{Textures} }     &-& 63.91$^{\pm 2.16}$ & 78.29$^{\pm 3.06}$ & 77.87$^{\pm 1.97}$ & 61.68$^{\pm 1.68}$ \\
&     & +& \cellcolor[HTML]{EFEFEF}\textbf{44.91}$^{\pm 2.90}$ & \cellcolor[HTML]{EFEFEF}\textbf{92.89}$^{\pm 0.95}$  & \cellcolor[HTML]{EFEFEF}\textbf{67.19}$^{\pm 2.08}$ & \cellcolor[HTML]{EFEFEF}\textbf{81.11}$^{\pm 2.89}$ \\
  &     \multirow{2}{*}{\textsc{Places365} }     &- & 65.90$^{\pm 4.02}$ & 86.91$^{\pm 3.38}$ & 76.52$^{\pm 4.29}$ & 70.68$^{\pm 2.90}$\\
  &   &+&\cellcolor[HTML]{EFEFEF}\textbf{61.02}$^{\pm 2.61}$ & \cellcolor[HTML]{EFEFEF}\textbf{92.65}$^{\pm 1.91}$  &  \cellcolor[HTML]{EFEFEF}\textbf{64.16}$^{\pm 0.72}$ &  \cellcolor[HTML]{EFEFEF}\textbf{85.28}$^{\pm 1.29}$ \\ 
 &    \multirow{2}{*}{\textsc{Lsun-Resize} }     & - & 68.91$^{\pm 2.17}$ & 85.89$^{\pm 0.90}$ &  90.00$^{\pm 1.38}$ & 57.91$^{\pm 3.96}$  \\
  &   &+& \cellcolor[HTML]{EFEFEF}\textbf{56.81}$^{\pm 0.61}$ & \cellcolor[HTML]{EFEFEF}\textbf{92.81}$^{\pm 1.74}$ & \cellcolor[HTML]{EFEFEF}\textbf{67.40}$^{\pm 3.19}$ & \cellcolor[HTML]{EFEFEF}\textbf{85.89}$^{\pm 3.77}$ \\ 
   &    \multirow{2}{*}{\textsc{Lsun-C} }     & - & \textbf{26.08}$^{\pm 1.48}$ & 93.21$^{\pm 0.27}$ &62.57$^{\pm 0.97}$ & 75.50$^{\pm 1.98}$ \\
  &   &+&  \cellcolor[HTML]{EFEFEF}36.32$^{\pm 2.20}$ & \cellcolor[HTML]{EFEFEF}\textbf{94.97}$^{\pm 1.88}$ &   \cellcolor[HTML]{EFEFEF}\textbf{57.10}$^{\pm 2.93}$ &  \cellcolor[HTML]{EFEFEF}\textbf{83.23}$^{\pm 2.26}$  \\
  \hline
\multirow{2}{*}{\textsc{Near OOD}}   &   \multirow{2}{*}{\textsc{Cifar100}/\textsc{Cifar10} }     & -&71.76$^{\pm 2.49}$ & 81.97$^{\pm 0.71}$ & 92.10$^{\pm 2.90}$ & 58.57$^{\pm 4.86}$\\
&   &+& \cellcolor[HTML]{EFEFEF}\textbf{53.22}$^{\pm 2.10}$ & \cellcolor[HTML]{EFEFEF}\textbf{90.94}$^{\pm 1.17}$  & \cellcolor[HTML]{EFEFEF}\textbf{84.12}$^{\pm 2.26}$ & \cellcolor[HTML]{EFEFEF}\textbf{74.31}$^{\pm 0.63}$\\
    \hline
    \end{tabular}
    \caption{\small OOD detection results measured by post-hoc $k$-NN score w/ and w/o ID labels (\textsc{Cifar10} and \textsc{Cifar100} as ID). Mean and std are estimated on three different runs. Better results are highlighted in bold. ``+,-" denotes the labeled and unlabeled case and $k$ is set to 25 for all the experiments. }
    \vspace{-1em}
    \label{exp_result_app_post_hoc}
\end{table*}

\begin{wraptable}{r}{0.4\linewidth}
    \centering
    \vspace{-2em}
    \begin{tabular}{c|cc}
    \hline
       $\|\tilde{\mathbf{A}}_{\mathrm{OI}}^{(u)} \|_F$   & $\mathcal{G}$ &  Our bound \\
       \hline
       60& 0.09& 0.07\\ 
 72&0.16 &0.12 \\
 84& 0.21 &0.16\\
 96& 0.39&0.37 \\
 108&0.40 & 0.34\\
 120& 0.61& 0.56\\
         \hline 
    \end{tabular}
    \caption{\small Numerical results on the bound tightness.}
    \label{tab:bound_tightness}
\end{wraptable}

\section{Tightness of the Bound}
\label{sec:tightness_app}
We provide the evidence to verify the tightness of our bound. Specifically, we present numerical results on the illustrative example (Figure~\ref{fig:toy_example}) to show the proximity between the value of the error difference $\mathcal{G}$ and the bound in our Theorem in Table~\ref{tab:bound_tightness}. Specifically, we set a different value of $\|\tilde{\mathbf{A}}_{\mathrm{OI}}^{(u)} \|_F$ and observed our lower bound is sufficiently close to the error difference $\mathcal{G}$ (The details of the dataset used are the same as those described in Appendix Section~\ref{detail_toy_example_app} except for the Frobenius norm of the OOD-ID adjacency matrix).